\documentclass[]{style}
\usepackage{float}
\usepackage{graphicx}
\usepackage{titletoc}

\usepackage[english]{babel}
\usepackage{xcolor}
\usepackage{color}
\usepackage{wrapfig}

\usepackage{amsmath}
\usepackage{amssymb}
\usepackage{amsthm}
\usepackage{graphicx}
\usepackage{booktabs}
\usepackage{array}
\usepackage{multirow}
\usepackage{caption}  
\usepackage{amsmath}
\usepackage[table]{xcolor} 
\usepackage{xcolor}
\usepackage{microtype}
\usepackage{parskip}
\usepackage{tabularx}
\usepackage{tcolorbox}

\definecolor{bg}{HTML}{E7E7FE}
\definecolor{accent}{HTML}{6B67EE}
\definecolor{accent-bright}{HTML}{6B67EE}
\definecolor{mist}{HTML}{F0F0F8}
\definecolor{panelborder}{HTML}{9998B3}
\definecolor{muted}{HTML}{4a4a6a}
\definecolor{committed}{HTML}{0085B8}
\definecolor{refining}{HTML}{6B67EE}
\newtcolorbox{figurepanel}{
    enhanced,
    colback=mist,
    colframe=panelborder,
    boxrule=0.8pt,
    arc=6pt,
    left=10pt,right=10pt,top=10pt,bottom=10pt
}
\newtcolorbox{samplebox}{
    enhanced,
    colback=accent!4!bg!60!white,
    colframe=accent!55!black,
    boxrule=0.6pt,
    arc=4pt,
    left=8pt,right=8pt,top=7pt,bottom=7pt
}
\definecolor{mybg}{HTML}{EAE9FF} 
\definecolor{outline}{HTML}{9793F8} 
\newtcolorbox{borderbox}[1][]{
    colback=bg!50!white,
    colframe=outline,
    top=2pt, left=2pt, right=2pt, bottom=2pt,
    #1
}

\newcommand{\nfeheader}[1]{%
    {\sffamily\bfseries\large #1}
}
\newcommand{\methodtitle}[1]{%
    {\bfseries #1}
}

\usepackage{algorithm}

\usepackage[noEnd=false]{algpseudocodex}

\usepackage{thmtools}
\usepackage{thm-restate}
\tcbset{thmbox/.style={enhanced, breakable,
    colback=accent!4!bg!60!white, colframe=accent,
    boxrule=0pt, leftrule=0pt, arc=0pt,
    left=8pt, right=8pt, top=8pt, bottom=8pt,
    fonttitle=\bfseries\sffamily, coltitle=accent-bright}}
\newtcolorbox{thmbox}{enhanced, breakable,
    colback=accent!4!bg!60!white, colframe=accent,
    boxrule=0pt, leftrule=0pt, arc=0pt,
    left=8pt, right=8pt, top=8pt, bottom=8pt}

\tcolorboxenvironment{theorem}{thmbox}
\tcolorboxenvironment{proposition}{thmbox}
\tcolorboxenvironment{lemma}{thmbox}
\tcolorboxenvironment{corollary}{thmbox}
\tcolorboxenvironment{definition}{thmbox}
\tcolorboxenvironment{remark}{thmbox}
\tcolorboxenvironment{assumption}{thmbox}
\tcolorboxenvironment{box*}{thmbox}
\tcolorboxenvironment{restatable}{thmbox}
\newcommand{\restate}[1]{\begin{thmbox}#1*\end{thmbox}}

\DeclareMathOperator{\softmax}{softmax}
\DeclareMathOperator{\diag}{diag}
\DeclareMathOperator{\kl}{KL}
\DeclareMathOperator{\var}{Var}

\newcommand{\HL}[1]{\colorbox{mybg}{$\displaystyle #1$}}
\newcommand{\Mode}[1]{\textcolor{outline}{\texttt{#1}}}

\declaretheorem[
  name=Proposition,
  numberwithin=section,
  refname={Proposition,Propositions},
  Refname={Proposition,Propositions}
]{proposition}

\declaretheorem[
  name=Lemma,
  numberwithin=section,
  refname={Lemma,Lemmas},
  Refname={Lemma,Lemmas}
]{lemma}

\declaretheorem[
  name=Theorem,
  numberwithin=section,
  refname={Theorem,Theorems},
  Refname={Theorem,Theorems}
]{theorem}

\declaretheorem[
  name=Corollary,
  numberwithin=section,
]{corollary}

\declaretheorem[]{box*}

\declaretheorem[
  name=Remark,
  numberwithin=section,
  refname={Remark,Remarks},
  Refname={Remark,Remarks}
]{remark}
\declaretheorem[
  name=Assumption,
  numberwithin=section,
  refname={Assumption,Assumptions},
  Refname={Assumption,Assumptions}
]{assumption}

\title{\fontsize{0.6cm}{0.72cm}\selectfont Discrete Beckmann Transport Models\\[0.3cm] for One-Step Language Modeling and Reasoning}

\author[1]{Sophia Tang$^\dagger$}
\author[2,3,4]{Shiyi Wang}

\affiliation[1]{University of Pennsylvania}
\affiliation[2]{Harvard University}
\affiliation[3]{Kempner Institute}
\affiliation[4]{IAIFI}
\abstract{Discrete diffusion and flow models are a promising alternative to autoregressive language models, but compressing many-step sampling into fewer steps typically requires distilling a pretrained teacher model. This caps the student at the teacher's quality and requires a costly two-stage training pipeline. We introduce Discrete Beckmann Transport Models (DBTM), built on a time-independent flow whose autonomous transport map provably carries any point in the ambient space to a fixed point on the vertices of the simplex in a single step. We show that this fixed-point property is characterized by a conservation equation whose residual can be minimized directly from data, removing the requirement for a teacher flow and time conditioning. Under this construction, a partially trained map corresponds to the flow truncated at finite time, so generation reduces to iterating one map until it reaches a fixed point. We further extend the map to a partial-context interpolant where additional function evaluations act as refinement steps rather than ODE integration steps. On language modeling and reasoning tasks, DBTM enables one- and few-step generation that improves quality and accuracy over discrete diffusion and continuous flow baselines.

\vspace{0.5em}%
\textbf{\sffamily\bfseries Correspondence:} \href{sophtang@engineering.upenn.edu}{\texttt{sophtang@engineering.upenn.edu}}, \href{fwang@math.harvard.edu}{\texttt{fwang@math.harvard.edu}}
}
\begin{document}
\maketitle
\begingroup
\renewcommand{\thefootnote}{}
\footnotetext{$^\dagger$Work done while a visiting researcher at Harvard University.}
\endgroup

\renewcommand{\footnoterule}{%
  \kern -3pt
  \hrule width \linewidth
  \kern 2.6pt
}
\vspace{-0.2cm}


\begin{figure}[h!]
    \centering
    \includegraphics[width=0.9\linewidth]{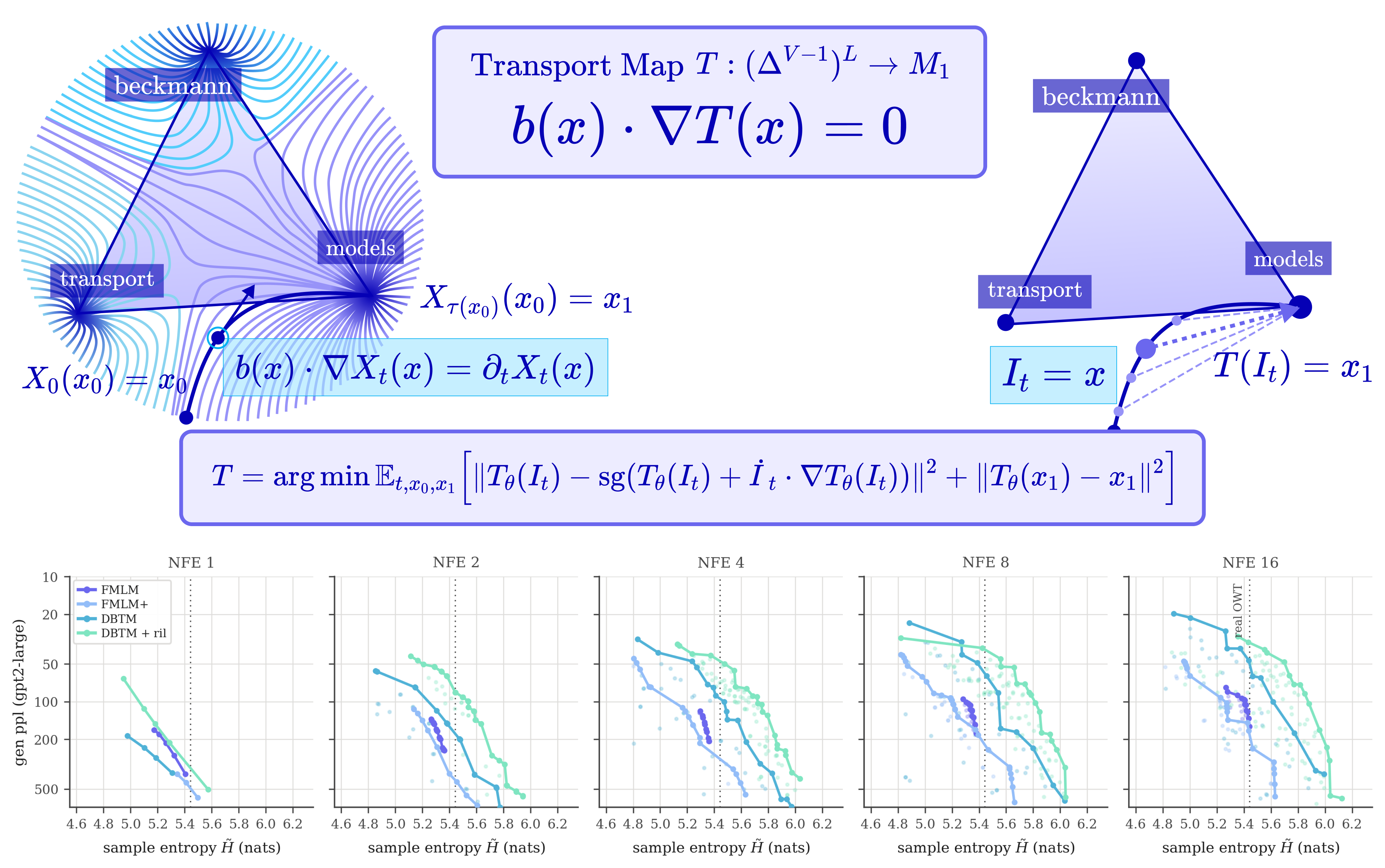}
    \caption{\textbf{Discrete Beckmann Transport Models.} \textbf{Top Left:} The autonomous field on a 2-dimensional simplex where the autonomous velocity field $b(x)$ solves the Eulerian equation. \textbf{Top Right:} The associated transport map $T(x)$ that solves the conservation equation $b(x)\cdot \nabla T(x)=0$ with boundary constraint $T(x)=x$ for $x\in M_1$ maps every point along the interpolant $I_t=x$ directly to the endpoint $x_1$. \textbf{Bottom:} Generative frontiers of DBTM variants vs. flow map baselines FMLM and FMLM+ on OpenWebText, with lower generative PPL and higher entropy being the optimal directions.}
    \label{fig:discrete-btm}
    \vspace{-10pt}
\end{figure}

\section{Introduction}

Discrete diffusion models \citep{shi2024simplified, sahoo2024simple, ou2024your, zheng2024masked} are competitive alternatives to autoregressive language modeling, offering bidirectional dependencies. However, for sequences of length $L$ over a vocabulary of size $V$, the state space grows as $V^L$, forcing practical samplers to factorize the per-step denoising distribution across positions. This factorization discards intra-step dependencies, so unmasking many positions at once degrades sample quality, and high-fidelity generation requires many function evaluations scaling linearly with $L$. The resulting limit on parallelism has motivated continuous-space flows that land on the simplex and can be distilled into few-step samplers. Among these techniques are \textit{flow maps} \citep{boffi2025build} which have been adapted to discrete sequences \citep{potaptchik2026discrete, lee2026flow} by defining a mean denoiser that transports a point at time $s$ directly to its further denoised state at time $t$, landing at a vertex of the simplex at $t=1$.

A key limitation of few-step language modeling with flow maps is the need to learn both the diagonal ($s=t$) and off-diagonal ($s\neq t$) terms, either by distilling a teacher flow approximating the marginal velocity $b_t$ or by training both jointly via bootstrapping, with the former commonly performing better \citep{potaptchik2026discrete, lee2026flow}. Distillation caps the student at the teacher's quality, propagating its approximation errors, and imposes a two-stage pipeline that must train the teacher to convergence before fitting the student. Bootstrapping removes the teacher but makes the off-diagonal targets depend on the model's own evolving predictions, which can be unstable and slow to converge.

We introduce \textbf{Discrete Beckmann Transport Models (DBTM)}, a framework that learns the \textit{autonomous transport map} $T$ associated with the time-independent flow, which provably has straighter trajectories than the time-dependent flow and carries any ambient point to the data manifold supported on the simplex vertices. This fixed-point property is formalized by the \textit{conservation equation} $b\cdot\nabla T=0$, whose residual we minimize to train directly from data, replacing the distillation stage entirely.

Our \textbf{main contributions} can be summarized as follows:
\begin{enumerate}
    \item \textbf{Discrete Beckmann Transport Models:} From the time-independent flow ODE $X_s(x_0)$, we define a \textit{one-step map} $T(x):=X_{\tau(x_0)}(x)$ that provably transports any ambient point to a fixed point on the discrete data manifold without conditioning on time.
    \item \textbf{Direct Training from Data:} $T$ can be trained without distillation by iteratively regressing onto a frozen copy of itself plus an Euler correction. A partially trained $T_\theta$ maps to an intermediate point along the autonomous flow and converges to $T_{\theta^\star}=X_{\tau(x_0)}$, so the manifold can be reached by iteratively applying $T_\theta$.
    \item \textbf{Self-Correction via Iterative Refinement:} We train with a partial-context interpolant, which enables confidence- or quality-based iterative refinement during training and inference. 
\end{enumerate}

We validate DBTM on unconditional language modeling and conditional reasoning tasks, demonstrating stronger performance than baselines at lower or matched NFEs. Further discussion of related works is provided in App. \ref{app:related-work}.

\section{Preliminaries}

\paragraph{Continuous Generative Modeling on the Simplex}
Let $\mathcal{V}=\{1,\dots,V\}$ be a vocabulary and $x=(x^1,\dots,x^L)\in\mathcal{V}^L$ a sequence of $L$ tokens. Embedding each token as a vertex of the probability simplex $\Delta^{V-1}:=\{x\in\mathbb{R}^V_{\geq0}:\sum_jx_j=1\}$ via the one-hot map $x^\ell\mapsto e_{x^\ell}$ represents a sequence with a point in $(\mathbb{R}^V)^L$. The data distribution $\mu_1$ is atomic, supported on the $V^L$ vertex configurations:
\begin{align}
    M_1:=\{e_1,\dots,e_V\}^L\subset(\Delta^{V-1})^L,\label{eq:data-manifold}
\end{align}
and we decode by taking the per-position argmax $x\mapsto(\arg\max_jx^\ell_j)_{\ell\in[L]}$, which is exact on $M_1$. The prior distribution can be defined as a product of Gaussians $\mu_0=\mathcal{N}(0,\sigma^2\boldsymbol{I}_V)^{\otimes L}$, and continuous generative modeling aims to transport $\mu_0$ to $\mu_1$ via a continuous flow in ambient space.

\paragraph{Generative Flows and Flow Maps}
The stochastic interpolants framework \citep{albergo2022building, albergo2025stochastic} connects $\mu_0$ and $\mu_1$ through the interpolant:
\begin{align}
    I_t:=\alpha_tx_0+\beta_tx_1, \quad t\in [0,1]\label{eq:interpolant}
\end{align}
with $x_0\sim\mu_0$, $x_1\sim\mu_1$ independent and schedules satisfying $\alpha_0=\beta_1=1$, $\alpha_1=\beta_0=0$, so that $I_0=x_0$ and $I_1=x_1$. We take $\beta_t=1-\alpha_t$ with $\alpha_t=(1-t)^a$ for $a\geq1$, recovering the linear interpolant at $a=1$. Writing $\mu_t:=\mathrm{Law}(I_t)$, the flow matching velocity is given by:
\begin{align}
    b_t(x):=\mathbb{E}_{x_0,x_1}\left[\dot I_t\big|I_t=x\right]=\underset{\hat b_t}{\arg\min}\ \mathbb{E}_{x_0,x_1}\left[\left|\hat b_t(I_t)-\dot I_t\right|^2\right]\label{eq:fm-velocity}
\end{align}
where the second equality is the regression characterization that makes $b_t$ learnable from samples without access to $\mu_t$. The pair $(b_t, \mu_t)$ satisfies the \textit{continuity equation}:
\begin{align}
    \nabla \cdot (b_t\mu_t)=-\partial_t\mu_t.\label{eq:continuity}
\end{align}
The probability flow ODE is given by:
\begin{align}
    \frac{d}{dt}X_{s,t}(x)=b_t(X_{s,t}(x)), \quad X_{s,s}(x)=x\label{eq:probability-flow}
\end{align}
transports $\mu_s$ to $\mu_t$. Sampling from a learned $b_t$ requires integrating \eqref{eq:probability-flow}, at a cost measured in network function evaluations (NFE). Flow maps instead learn the two-time solution operator $X_{s,t}$, advancing the state over a finite interval in one evaluation. By uniqueness of the ODE solution, it obeys the semigroup property:
\begin{align}
     X_{t,u}\circ X_{s,t}=X_{s,u}\ \ \text{for all}\ s\leq t\leq u,\qquad X_s\circ X_u=X_{s+u}\ \ \text{when}\ b_t\equiv b\label{eq:semigroup}
\end{align}
where the second identity is the case where the velocity field is time-independent. 

\section{Autonomous Flows on the Simplex}

\subsection{Properties of the Autonomous Flow}
While standard flow matching learns a time-dependent field $b_t$ and integrates it over a fixed horizon $t\in[0,1]$, an autonomous flow model \citep{lee2026beckmann} defines a time-independent velocity field $b$:
\begin{align}
    b(x):=\mathbb{E}_{t, x_0,x_1}[\dot I_t|I_t=x]=\underset{\hat b}{\arg\min}\mathbb{E}_{t,x_0,x_1}\left[|\hat b(I_t)-\dot I_t|^2\right], \qquad t\sim \mathcal{U}[0,1]
\end{align}
associated with the autonomous ODE:
\begin{align}
    \dot X_s(x_0)=b(X_s(x_0)),  \quad X_0(x_0)=x_0\label{eq:autonomous-flow}
\end{align}
where $s\in [0,\infty)$ is the integration time of the autonomous ODE, which is decoupled from the interpolation time $t$.

\begin{restatable}[Time-average of standard flow matching]{proposition}{timeaverage}\label{prop:time-average}
    Given the time-dependent flow matching velocity $b_t(x):=\mathbb{E}_{x_0,x_1}[\dot I_t|I_t=x]$ defined with the interpolant $I_t:=\alpha_tx_0+\beta_tx_1$, the autonomous velocity field is the time-weighted average:
    \begin{align}
    \label{eqn:defn auto velocity}
        b(x)
    &=\frac{\int_0^1\mu_t(x)b_t(x)dt}{\int_0^1\mu_t(x)dt}
    \end{align}
\end{restatable}

The proof is in App. \ref{app:time-average}. Intuitively, $b(x)$ averages the velocities of every interpolant passing through $x$ at any time $t\in [0,1]$, so the field depends on position alone, removing the time dependence that constrains the flow to terminate at $\mu_1$ at $t=1$ and yielding straighter dynamics. We make this precise in App. \ref{app:dynamics-proofs}. Just as the time-dependent flow satisfies a \textit{continuity equation} \eqref{eq:continuity}, the autonomous flow satisfies the \textit{divergence condition} (Proof in App. \ref{app:divergence-cond}) \citep{lee2026beckmann} obtained by integrating over time $t\in [0,1]$:
\begin{align}
    \int_0^1\nabla\cdot(b_t \mu_t)=-\int_0^1\partial_t p _t\implies \nabla \cdot j=\mu_0-\mu_1
\end{align}
where $j=\nu b=\int_0^1\mu_tb_tdt$ is the \textit{current} and $\nu(x):=\int_0^1\mu_t(x)dt$ the \textit{occupation measure}. This makes $\mu_0$ the source and $\mu_1$ the sink of the current, so $b$ transports $\mu_0$ to $\mu_1$ -- with the key distinction that $X_s$ need not reach $M_1$ at $s=1$. We call the time at which it converges to $M_1$ the \textit{hitting time} $\tau(x_0)\in[0,+\infty)$, and show it is finite on the simplex.

\begin{restatable}[Convergence of autonomous flow on the simplex]{theorem}{convergence}\label{thm:convergence-simplex}
    Let $\mu_0$ be absolutely continuous on $(\mathbb{R}^V)^L$ with continuous, strictly positive density, and let
    \begin{align}
        M_1=\mathrm{supp}(\mu_1)=(\{e_1,\dots,e_V\})^L\subset(\Delta^{V-1})^L
    \end{align}
    be the finite set of $V^L$ vertex configurations. Then for $\mu_0$-a.e.\ $x_0$, the autonomous flow $X_s(x_0)$ satisfies:
    \begin{enumerate}
        \item[(i)] \textbf{Absorption}: if $X_{s^\star}(x_0)\in M_1$ for some $s^\star<\infty$, then $X_s(x_0)=X_{s^\star}(x_0)$ for all $s>s^\star$.
        \item[(ii)] \textbf{Convergence}: there is $x^\star(x_0)\in M_1$ with $\lim_{s\to\infty}X_s(x_0)=x^\star(x_0)$.
        \item[(iii)] \textbf{Finite hitting time}: for the linear interpolant $I_t=(1-t)x_0+tx_1$, $\tau(x_0):=\inf\{t\geq0:X_s(x_0)\in M_1\}<+\infty$. Although the autonomous velocity field vanishes at the vertices, $b(e_j)=0$ for all $e_j\in M_1$, it does not vanish in the limit approaching them: $\lim_{r\downarrow0}b(e_j+r\omega)=-\kappa(\omega)\omega$. The approach speed is bounded below by $\kappa_->0$ and $\tau(x_0)\leq t_{\text{entry}}+2r_1/\kappa_-$ for a fixed radius $r_1$.
    \end{enumerate}
\end{restatable}

The proof is in App. \ref{app:vertex-vanish-proof}. The resulting discontinuity in $b$ at the vertices is an artifact of the linear interpolant: for a general interpolant \eqref{eq:interpolant} that is \textit{self-stopping}, i.e., $\dot\alpha_t\to0$ as $t\to1$, Proposition \ref{prop:continuity at the vertices} shows $b$ is continuous at the vertices with $b(x)\to0$ as $x\to e_i$.
 
\subsection{Autonomous Transport Map}
Given the hitting time $\tau(x_0)$ of the autonomous flow $X_s(x_0)$, the associated one-step transport map $T$ can be written as the limit:
\begin{align}
    T(x_0):=\lim_{s\to \tau(x_0)}X_s(x_0) \in M_1, \quad T_{\#}\mu_0=\mu_1\label{eq:transport-map}
\end{align}
which is also the pushforward of $\mu_0$ to $\mu_1$. Since $X_s(x_0)$ is a deterministic function of $x_0$, starting along any point $X_s(x_0)$ for $s\in [0,\tau(x_0))$ yields the same fixed point $T(X_s(x_0))=T(x_0)$. By the semigroup condition \eqref{eq:semigroup}, the map $T$ is idempotent such that $T^{\circ k}(\cdot )=T^{\circ k'}(\cdot)$ for any $k,k'\in \mathbb{N}$.

\begin{restatable}[Autonomous transport map solves conservation equation \citep{lee2026beckmann}]{proposition}{conservation}\label{prop:conservation-equation}
    Given the autonomous flow velocity field $b$, the autonomous flow $X_s(x_0)$ defined in \eqref{eq:autonomous-flow} solves the Eulerian equation:
    \begin{align}
    \begin{cases}
         b(x)\cdot \nabla X_s(x)=\partial_sX_s(x), \quad X_0(x) =x & x\notin M_1\\
         X_s(x)=x&x\in M_1
    \end{cases}\label{eq:eulerian-equation}
    \end{align}
    Since the autonomous transport map $T(x)$ defined in \eqref{eq:transport-map} is the limit of $X_s(x)$ as $s\to \tau(x)$ with vanishing velocity at the vertices ($\dot X_s(e_j)=0$ for all $j\in \{1, \dots, V\}$) by Theorem \ref{thm:convergence-simplex}, $T(x)$ is the unique solution of the conservation equation:
    \begin{align}
    \begin{cases}
         b(x)\cdot \nabla T(x)=0 & x\notin M_1\\
         T(x)=x&x\in M_1
    \end{cases}\label{eq:conservation-equation}
    \end{align}
\end{restatable}
Proposition \ref{prop:conservation-equation} follows from Theorem \ref{thm:convergence-simplex}, with proof in App. \ref{app:eulerian-conservation}. This is the key result that drives a principled approach for training the one-step map directly from data without distillation.

\section{Learning One-Step Maps Without Distillation}
\label{sec:training-map}
We present a method of directly learning the one-step transport map $T$ by matching the updates that converge to the solution of the conservation equation in \eqref{eq:conservation-equation}. The converged map $T$ transports any noisy state to its corresponding fixed point on the data manifold, yielding one-step generation and principled refinement techniques.

\subsection{Iteratively Solving the Eulerian Equation}
\label{sec:conservation-loss}
Unlike previous few-step model training, which requires distillation or bootstrapping from an approximated target, Proposition \ref{prop:conservation-equation} presents a principled approach to learning $T$ by iteratively regressing the autonomous flow $X_s(x)$ until the equilibrium point when $b(x)\cdot \nabla X_{\tau(x)}(x)=b(x)\cdot \nabla T(x)=0$ for all $x\in (\mathbb{R}^V)^L$. Starting from an untrained map $T^{(0)}(x)= X^{(0)}(x)=x$, we iteratively regress the Euler step of size $\eta$:
\begin{align}
    X^{(n+1)}(x)=X^{(n)}(x)+\eta b(x)\cdot \nabla X^{(n)}(x)
\end{align}
where $T_\theta$ is trained to match $X^{(n+1)}$ given fixed $X^{(n)}$ as $X^{(n)}\to T$ for large $n$. Since we define $b(x):=\mathbb{E}_{t, x_0,x_1}[\dot I_t|I_t=x]$, the optimal map is:
\begin{align}
    T=\underset{T_\theta}{\arg\min}\mathbb{E}_{t,x_0,x_1}\left[\|T_\theta(I_t )-\text{sg}(T_\theta(I_t)+\dot I_t\cdot \nabla T_\theta(I_t))\|^2+\|T_\theta(x_1)-x_1\|^2\right]\label{eq:argmin-T}
\end{align}
minimized exactly when $T$ satisfies the conservation equation $\mathbb{E}_{t,x_0,x_1}[\dot I_t \cdot\nabla T_\theta(I_t)|I_t=x ]=b(x)\cdot\nabla T_\theta(x)=0$ for $x\notin M_1$ and $T(x)=x$ for $x\in M_1$. Crucially, this objective ensures that even when $T_\theta$ has not converged, it approximates a map to some point along $X_s$, so by the semigroup property \eqref{eq:semigroup}, $k$ inference-time iterations $T_\theta^{\circ k}(x_0)$ approximates along the flow to $X_{kt}(x_0)$, which converges to $M_1$ as $kt\to \tau(x_0)$. 

\subsection{Training Objectives for the Transport Map}
\label{sec:training-objectives}
In practice, the autonomous transport map is parameterized as $T_\theta(\cdot)=\text{softmax}(f_\theta(\cdot))$ for a neural network $f_\theta$. Since the map is restricted to the simplex, we formulate each objective as a cross-entropy (CE) over categorical distributions rather than an unconstrained squared error, with the exception of the transport loss where the target does not lie on the simplex. 

\paragraph{Transport Loss}
From \eqref{eq:argmin-T}, the \textit{transport loss} is obtained by regressing $T_\theta$ onto the stop-gradient of itself, added to the Euler step correction along the autonomous flow:
\begin{align}
    \mathcal{L}_{\text{transport}}(\theta):=\mathbb{E}_{t,x_0,x_1}\bigg[\sum_{\ell\in [L]\setminus \mathcal{C}}\left\|T_\theta(I_t)_\ell-\text{sg}(T_\theta(I_t)+\dot I_t\cdot \nabla T_\theta(I_t))_\ell \right\|^2\bigg]\label{loss:transport}
\end{align}
where $\mathcal{C}$ is the set of committed tokens and the correction $\tfrac{d}{dt}T(I_t)=\dot I_t\cdot \nabla T(I_t)$ is tangent to the simplex and its elements sum to zero. 

\paragraph{Semigroup Loss}
From \eqref{eq:transport-map}, we show that by the semigroup condition $T_\theta^{\circ k}(\cdot)=T_\theta^{\circ k'}(\cdot)$, which we enforce using the \textit{semigroup loss}:
\begin{align}
    \mathcal{L}_{\text{semi}}(\theta):=\mathbb{E}_{t,x_0,x_1}\bigg[\sum_{k\in \mathcal{K}}\sum_{\ell\in [L]\setminus \mathcal{C}}\text{CE}\left(T_\theta(I_t)_\ell, \text{sg}(T^{\circ k}_\theta(I_t)_\ell) \right)\bigg]\label{loss:semigroup}
\end{align}
for $k\in \mathcal{K}\subset \mathbb{N}$ applications of the map. We apply this loss only toward the end of training, since the initial identity map $T=\text{id}$ satisfies it trivially. 

\paragraph{Boundary Loss}
To enforce the boundary condition $T(x) =x$ for $x\in M_1$ from Proposition \ref{prop:conservation-equation}, we add a \textit{boundary loss}:
\begin{align}
\mathcal{L}_{\text{bnd}}(\theta):=\mathbb{E}_{x_1}\bigg[\sum_{\ell\in [L]\setminus \mathcal{C}}\text{CE}\left(T_\theta(x_1)_\ell, x_1^\ell\right)\bigg]
\end{align}
Beyond the boundary itself, any intermediate $I_t$ should map to its target, $T_\theta(I_t)=x_1$. While the \textit{anchor condition} is implicitly enforced by the transport loss, the geometry of the simplex makes explicit enforcement useful over a specific time interval, which we describe next.

\subsection{Phase Transitions and the Anchor Loss}
The flow constructed from the time-dependent interpolant on the simplex commits to a vertex at a short time window called the \textit{phase transition}.  Before the transition, the output token remains uncertain; after it, the token is committed, and further integration does not change the output. Such transitions have been analyzed in diffusion models \citep{ambrogioni2023statistical, yu2025nonequilbrium}, and we can use tools from random energy models (REM) \citep{derrida1980random} to determine the timing of this transition. 

\begin{restatable}[Phase Transition and Commitment Time]{theorem}{optimalt}\label{thm:optimal-t}
    The phase transition where the interpolant \eqref{eq:interpolant} commits to a vertex occurs at time:
    \begin{align}
    \label{eqn:anchor time}
        t^\star=1-\left(1+\sigma\sqrt{2\log V}\right)^{-1/a}
    \end{align}
    where $\sigma$ is the noise scale of the $x_0$, and $a$ is the exponent on $\alpha_t=(1-t)^a$ and $V$ is the vocabulary size defining the simplex $\Delta^{V-1}$.  
\end{restatable}

The full derivation is in App. \ref{app:optimal-t}. The derivation relies on the idea that competing forces pull the flow trajectory towards each of the $V$ vertices, resembling the selection of states in a random energy model at a given temperature, which undergoes a phase transition at $t=t^\star$. With the target vertex fixed at $e_k$ during training, the trajectory starts from $x_0$, where each vertex carries an energy proportional to the noise scale $\sigma$, and the flow commits to the target when the energy of the target vertex outweighs the competing energies of all other vertices $j\neq k$. Intuitively, this means that the interval $[0,t^\star]$ is when the flow must resolve uncertainty in coordination with the other positions, and the interval $[t^\star, 1]$ is the period post-transition where the endpoint is decided, and we can safely anchor the trajectory to the dominant vertex. 

This gives a principled criterion for \textit{when} to supervise the map toward its target. Over $[t^\star,1]$, where the endpoint is already decided, we can anchor $T_\theta(I_t)$ directly without sacrificing diversity:
\begin{align}
    \mathcal{L}_{\text{anchor}}(\theta):=\mathbb{E}_{t,x_0,x_1}\bigg[\sum_{\ell}\boldsymbol{1}[t>t_{\text{anchor}}]\text{CE}\left(T_\theta(I_t)^\ell, x_1^\ell\right)\bigg]\label{eq:anchor-loss}
\end{align}
with $t_{\text{anchor}}=t^\star$. Empirically, disabling $\mathcal{L}_{\text{anchor}}$ or setting $t_{\text{anchor}}>t^\star$ slowed convergence, while $t_{\text{anchor}}<t^\star$ caused mode collapse (Table \ref{tab:anchor-ablation}). More broadly, this characterization of the commitment time on the simplex suggests that the transport loss \eqref{loss:transport} and inference-time steering are best concentrated on the pre-transition interval $[0,t^\star]$, where the endpoint is still being resolved. 

\section{Training and Sampling with Fixed-Point Refinement}
\label{sec:self-correction}

\subsection{Test-Time Scaling with Self-Refinement}
\label{sec:self-refinement}
\paragraph{Partial-Context Interpolant}
We consider the subclass of interpolants in which tokens at positions $\mathcal{C}\subset \{1, \dots, L\}$ in the sequence are held as \textit{context tokens} in their clean state $I^\ell_t=x^\ell_1$, while the remaining are noised to time $t$:
\begin{align}
    I_t^\ell(\mathcal{C}):=\begin{cases}
        x_1^\ell&\ell \in \mathcal{C}\\
        \alpha_t^\ell x_0^\ell+(1-\alpha_t^\ell)x_1^\ell&\ell\in [L]\setminus \mathcal{C}
    \end{cases}
\end{align}
The map $T_\theta$ remains time-independent but is trained on this larger family of per-position interpolants $I_t^\ell(\mathcal{C})$, enabling conditioning on partial clean contexts during inference.

\paragraph{Confidence-Based Refinement}
\label{sec:confidence-correct}
Using the predicted map $T_\theta$, we can obtain a measure of the model's per-token confidence $q_\phi^\ell$ and use it to determine which tokens to \textit{renoise} and \textit{refine} via the context-dependent map. Concretely, we define:
\begin{align}
    q^\ell(T_\theta(x) ):=\langle T_\theta(x), \hat x_1\rangle, \quad \hat x_1=\arg\max(T_\theta(x))
\end{align}
Given a threshold $\kappa$ (e.g. $\kappa=0.9$), the refinement step holds the clean one-hot tokens $\hat x^\ell$ for the set of high-confidence positions $\mathcal{C}:=\{\ell:q^\ell(T_\theta(x))\geq \kappa\}$ and samples fresh noise for all remaining positions $\ell \in [L]\setminus \mathcal{C}$, and reapplies the map $T_\theta$, which training with the partial-context interpolant makes well-defined.

\paragraph{Refinement with Learned Per-Token Quality}
\label{sec:quality-correct}
Another approach to performing self-correction on the output of the transport map is to measure the coherence of each token given the predicted context and resampling noise for incoherent predictions. The ideal coherence score is the probability of the proposed token $\hat x^\ell_1$ given the ground truth at all other positions, $q_\star^\ell(\hat x):=\text{Pr}[\hat x^\ell\mid x_1\oplus m^\ell]$, where $m^\ell$ masks the $\ell$th position. Since $x_1$ is unavailable at inference, computing it would cost $L$ extra forward passes. Following \citet{kim2025fine}, we replace the ground truth with the model's own prediction and define the \textit{per-token quality}:
\begin{align}
    q^\ell_\phi(T_\theta(x))=\text{Pr}\left[\hat x^\ell_1=x_1^\ell \mid T_\theta(x)\right], \quad \hat x^\ell= \arg\max\left(T_\theta(x)^\ell\right)
\end{align}
which can be estimated by a learnable quality head with parameters $\phi$ trained by minimizing: 
\begin{align}
    \mathcal{L}_{\phi}(\phi)=\mathbb{E}_{x_0, x_1}\bigg[\sum_\ell\text{BCE}\left(\boldsymbol{1}[\hat x_1^\ell=x_1^\ell], q_\phi^\ell(T_\theta(x))\right)\bigg], \quad \hat x_1= \arg\max\left(T_\theta(x)\right)
\end{align}
Since this expectation is over the same $(x_0,x_1)$ pairs as the map objectives, it trains jointly with $T_\theta$. At inference, obtaining $q_\phi$ costs no additional NFE and defines $\mathcal{C}:=\{\ell:q_\phi^\ell(T_\theta(x))\geq \kappa\}$ just like confidence-based refinement. Algorithm \ref{alg:btm-sampling-refinement} provides the full inference algorithm with refinement.

\subsection{Training with Refinement in the Loop}
\label{sec:refinement-in-loop}
Both refinement rules reuse a map trained on the partial-context interpolant of Section \ref{sec:self-refinement}, where the context set $\mathcal{C}$ is drawn at random. At inference, however, $\mathcal{C}$ is not random. It is built up over rounds by committing the highest-confidence or quality positions first. The training distribution therefore contains out-of-distribution contexts, since none of the losses enforce the map to complete the structured context generated from the inference sampler. The mismatch is significant in the few-NFE regime: under a budget of $k$ rounds, each round commits roughly $1/k$ of the remaining positions, so the context after round $r$ is a large, highly non-uniform set, and $q_\phi$ is also out of distribution when scoring it.

\paragraph{Commit Rule}
We close this gap by supervising $T_\theta$ on the commit sets the sampler is likely to visit at inference. Index refinement rounds by $r=1,\dots,k$ with $\mathcal{C}_0=\emptyset$, and let
$\mathcal{R}_r:=[L]\setminus\mathcal{C}_{r-1}$ be the positions still uncommitted
entering round $r$. Writing $q^\ell$ for the commit score of position $\ell$ (confidence or the quality $q_\phi^\ell$), each round commits the set $\Delta \mathcal{C}_r$ ranked highest by $q^\ell$ and renoises the rest:
\begin{align}
    \Delta\mathcal{C}_r
    :=\underbrace{\{\ell\in\mathcal{R}_r: q^\ell\geq\kappa\}}_{\text{threshold}}
    \;\cup\;
    \underbrace{\operatorname{top}_{n_r}\big(\mathcal{R}_r; q\big)}_{\text{floor}},
    \qquad
    n_r:=\Big\lceil \tfrac{|\mathcal{R}_r|}{k-r+1}\Big\rceil,\qquad
    \hat x_{r,k}^\ell:=\begin{cases}
        x_1^\ell & \ell\in\mathcal{C}_r\\
        x_0^\ell & \ell\in[L]\setminus\mathcal{C}_r
    \end{cases}
    \label{eq:commit-rule}
\end{align}
where $\operatorname{top}_{n_r}(\mathcal{R}_r;q)$ denotes the $n_r$ highest-scoring positions
of $\mathcal{R}_r$, and $\mathcal{C}_r:=\mathcal{C}_{r-1}\cup\Delta\mathcal{C}_r$. So every position above $\kappa$ is committed, and at minimum $n_r$ positions are
committed regardless. Since $|\Delta\mathcal{C}_r|\geq n_r$ and $n_k=|\mathcal{R}_k|$,
we have $\mathcal{C}_k=[L]$ and the sampler always terminates within $k$ NFE.

\paragraph{Refinement-in-Loop Objective}
From \eqref{eq:commit-rule}, $\hat x_{r,k}:= \textsc{Renoise}(T_\theta(\hat x_{r-1,k}))$ can be used like the partial-context interpolant of Section \ref{sec:self-refinement} with the committed positions determined from sampling with refinement. We train the map to recover $x_1$ from these states via the \textit{refinement-in-loop} (ril) objective:
\begin{align}
    \mathcal{L}_{\text{ril-c}}(\theta):=\mathbb{E}_{x_0,x_1}\bigg[\sum_{k\in\mathcal{K}}\sum_{r=1}^{k}\sum_{\ell\in[L]\setminus\mathcal{C}_r}\text{CE}\left(T_\theta(\hat x_{r-1,k})^\ell,\ x_1^\ell\right)\bigg]\label{loss:ril}
\end{align}
which sums only over uncommitted positions, since the committed positions are pinned one-hot and carry no gradient. The same states are used to train the quality head:
\begin{align}
    \mathcal{L}_{\text{ril-}\phi}(\phi):=\mathbb{E}_{x_0,x_1}\bigg[\sum_{k\in\mathcal{K}}\sum_{r=1}^{k}\sum_{\ell\in[L]\setminus\mathcal{C}_r}\text{BCE}\left(\boldsymbol{1}[\hat x_1^\ell=x_1^\ell],\ q^\ell_\phi(T_\theta(\hat x_{r-1,k}))\right)\bigg], \quad \hat x_1=\arg\max(T_\theta(\hat x_{r-1,k}))\label{loss:ril-phi}
\end{align}
with the argument of $q_\phi$ detached. The full objective is
\begin{align}
    \mathcal{L}(\theta,\phi)=\mathcal{L}_{\text{transport}}+\lambda_{\text{semi}}\mathcal{L}_{\text{semi}}+\lambda_{\text{bnd}}\mathcal{L}_{\text{bnd}}+\lambda_{\text{anchor}}\mathcal{L}_{\text{anchor}}+\lambda_{\phi}\mathcal{L}_{\phi}+\lambda_{\text{ril}}\left(\mathcal{L}_{\text{ril-c}}+\lambda_{\text{quality}}\mathcal{L}_{\text{ril-}\phi}\right)
\end{align}
The stop-gradient rollout costs $k$ extra forward passes and no extra backward pass during training. We implement $q_\phi$ as a head on the trunk of $T_\theta$ reading its hidden state, so a refinement round costs exactly one NFE and the $k$-round sampler requires only $k$ NFEs resulting in \textit{no inference overhead}. An alternative instantiation for ril training is given in App.~\ref{app:ril-training}.

\paragraph{Self-Distillation with Semigroup Loss}
Late in training, the output after $k$ refinement rounds approximates a clean sample from $\mu_1$, and its coupling to the initial noise $x_0$ is automatic, removing the need for a prescribed noise-data coupling. Writing $\hat x_{r,k}$ for the state after $r$ rounds under budget $k$ with $\hat x_{0,k}:=T_\theta(x_0)$, the \textit{semigroup objective} matches the map at every intermediate round to the final output:
\begin{align}
    \mathcal{L}_{\text{semi-r}}(\theta):=\mathbb{E}_{x_0,x_1}\bigg[\sum_{k\in \mathcal{K}}\sum_{r=1}^{k}\sum_{\ell}\text{CE}\left(T_\theta(\hat x_{r-1,k}), \text{sg}(T_\theta(\hat x_{k-1,k}))\right)\bigg]
\end{align}
where $\hat x_{k,k}=T_\theta(\hat x_{k-1,k})$ is the final clean sequence after $k$ rounds of refinement. This is analogous to the semigroup loss defined in \eqref{loss:semigroup} but with the map operation composed with the renoise operation.

\begin{figure}[t]
    \centering
    \includegraphics[width=0.85\linewidth]{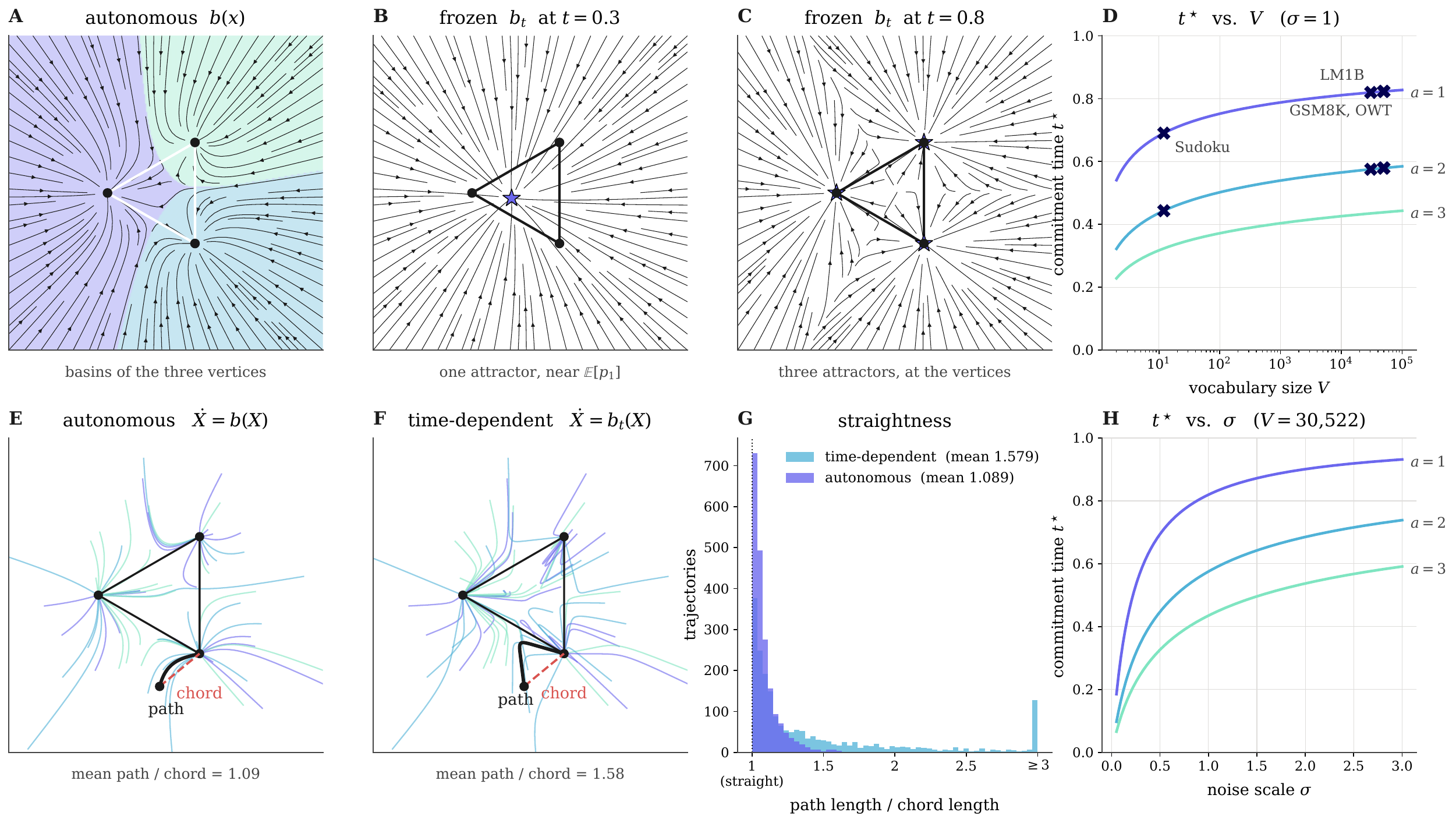}
    \vspace{-10pt}
    \caption{\textbf{Dynamics of Autonomous and Time-Dependent Flows.} \textbf{(A)} Visualization of the autonomous flow on $\Delta^2$ that is stationary across time and flows towards fixed basins of attraction at the vertices. \textbf{(B)} Time-dependent flow before the vertex commitment time ($t=0.3$) where the basin of attraction is at the mean, and \textbf{(C)} after the commitment time ($t=0.8$) when the basins of attraction are at the vertices. \textbf{(E) and (G)} Comparison of mean path and chord ratio of autonomous and time-dependent flows. \textbf{(D) and (H)} Plot of commitment time $t^\star$ against the simplex dimension $V$ and the prior scale $\sigma$ with curves for when the interpolant exponent $a\in \{1,2,3\}$.}
    \vspace{-10pt}
    \label{fig:flow-dynamics}
\end{figure}

\section{Experiments}
We validate our claims empirically: (i) the autonomous flow on the simplex yields stable attractors at the vertices of the simplex, resulting in smoother trajectories; (ii) the timing of the phase transition is the optimal anchor time that prevents mode collapse and yields faster convergence; (iii) direct training of the one-step map scales to unconditional language modeling and reasoning benchmarks; and (iv) refinement-in-loop training improves  coherence and accuracy. 

\subsection{The Autonomous Field Has Stable Attractors at the Vertices}
\label{sec:exp-dynamics}
In Figure \ref{fig:flow-dynamics}, we compare the time-dependent flow matching field against the autonomous field on the 2-simplex: the autonomous field partitions the ambient space into basins where the only equilibria are the vertices, a property we analyze further for general simplices in App. \ref{app:dynamics-proofs}. As illustrated in Figure \ref{fig:flow-dynamics}E-G, the autonomous flow has straighter trajectories due to the stationary basins of attraction than the time-dependent flow, whose basins of attraction appear at staggered times in the interior of the simplex before moving to the vertices (App. \ref{app:attractors}). 

\begin{figure}[t]
    \centering
    \includegraphics[width=\linewidth]{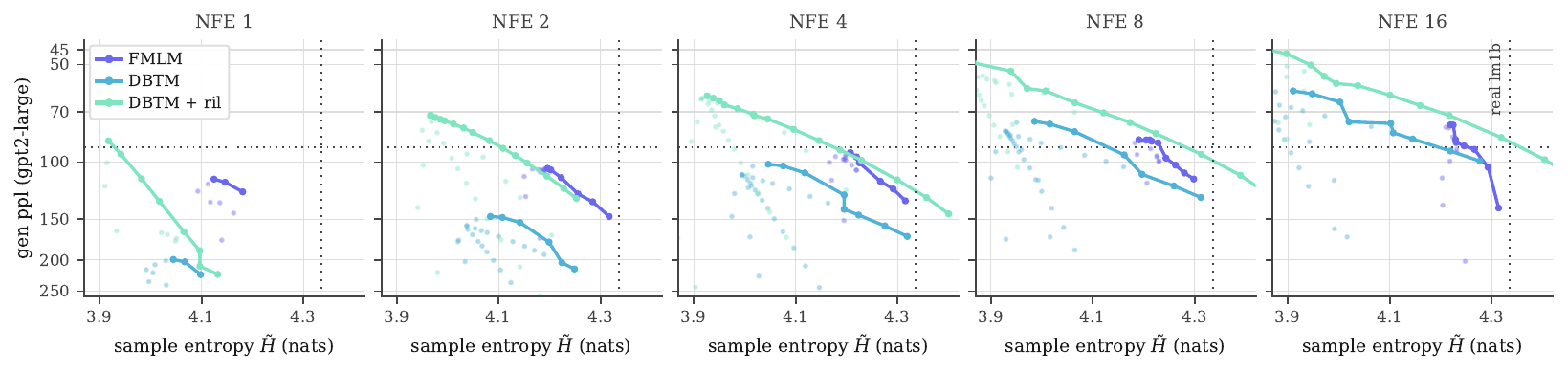}
    \caption{\textbf{Generative Frontiers of DBTM vs. FMLM.} Generative perplexity (plotted log scale, $\downarrow$) against sample entropy ($\uparrow$) on LM1B for NFE $\in \{1,2,4,8,16\}$. Each curve is a Pareto frontier: every inference-time knob swept for each method is pooled, and a point is kept only if no other reaches both higher entropy and lower Gen. PPL. DBTM is swept over commit temperature, prior scale $\sigma$, and commit threshold $\kappa$. FMLM is swept over its churn $\gamma$ and the same prior scale. The dotted line marks real LM1B data entropy (4.336). DBTM and DBTM + ril are trained for a total of 200K steps each, compared to FMLM initialized from a 1M-step teacher with 100K distillation steps.}
    \label{fig:generative-frontier-lm1b}
    \vspace{-8pt}
\end{figure}

\begin{table}[t]
\centering
\small
\resizebox{\textwidth}{!}{%
\begin{tabular}{l cc cc cc cc cc cc}
\toprule
& \multicolumn{6}{c}{\textbf{LM1B}} & \multicolumn{6}{c}{\textbf{OWT}} \\
\cmidrule(lr){2-7}\cmidrule(lr){8-13}
\textbf{NFE} & \multicolumn{2}{c}{1} & \multicolumn{2}{c}{2} & \multicolumn{2}{c}{4} & \multicolumn{2}{c}{1} & \multicolumn{2}{c}{2} & \multicolumn{2}{c}{4} \\
\cmidrule(lr){2-3}\cmidrule(lr){4-5}\cmidrule(lr){6-7}\cmidrule(lr){8-9}\cmidrule(lr){10-11}\cmidrule(lr){12-13}
\textbf{Method} & Gen-PPL~$\downarrow$ & Ent.~$\uparrow$ & Gen-PPL~$\downarrow$ & Ent.~$\uparrow$ & Gen-PPL~$\downarrow$ & Ent.~$\uparrow$ & Gen-PPL~$\downarrow$ & Ent.~$\uparrow$ & Gen-PPL~$\downarrow$ & Ent.~$\uparrow$ & Gen-PPL~$\downarrow$ & Ent.~$\uparrow$ \\
\midrule
Duo + DCD   & 1224.52 & 4.33 & 520.08 & 4.20 & 210.88 & 4.23 & 5743.29 & 6.02 & 891.16 & 5.41 & 250.86 & 5.37 \\
Duo + Di4C  & 292.94  & 3.79 & 247.69 & 3.87 & 150.67 & 4.00 & 370.51  & 3.92 & 210.22 & 4.63 & 154.67 & 4.85 \\
MDLM + SDTT & 1429.48 & 4.31 & 602.14 & 4.28 & 241.01 & 4.28 & 1260.86 & 5.26 & 877.22 & 5.34 & 339.73 & 5.38 \\
MDLM + Di4C & 1217.10 & 4.38 & 621.59 & 4.37 & 247.32 & 4.00 & 1298.80 & 5.29 & 758.23 & 5.35 & 239.27 & 5.40 \\
CFM         & 269.72  & 3.10 & 267.39 & 3.15 & 267.97 & 3.28 & --      & --   & --     & --   & --     & --   \\
FMLM        & 119.34  & 4.16 & 110.19 & 4.21 & 98.76  & 4.21 & 168.30  & 5.17 & 133.29 & 5.25 & 111.31 & 5.26 \\
FMLM+ & -- & -- & -- & -- & -- & -- & 378.53 & 5.34 & 114.46 & 5.14 & 41.16 & 4.77 \\
DFM (PSD)   & 94.08   & 4.06 & 87.42  & 4.08 & 78.89  & 4.10 & 180.29  & 4.91 & 152.83 & 5.03 & 122.32 & 5.10 \\
DFM (ESD)   & 68.11  & 3.79 & 77.60  & 4.11 & 71.53  & 4.13 & 5.33    & 0.26 & 108.91 & 5.15 & 77.08  & 5.27 \\
\midrule
\rowcolor{mybg} \textbf{DBTM}  & 201.20 & 4.03 & 167.05 & 4.19 & 119.99 & 4.20 & 188.2 & 4.97 & 76.2 & 5.14 & 52.2 & 5.26 \\
\rowcolor{mybg} \textbf{DBTM + ril}  & 85.50 & 3.92 & 81.00 & 4.05 & 73.54 & 4.04 & 65.2 & 4.95 & 62.2 & 5.38 & 57.5 & 5.52 \\
\bottomrule
\end{tabular}
}
\caption{\textbf{Discrete Beckmann Transport Model Unconditional Language Modeling Performance.} Generative perplexity ($\downarrow$) and entropy ($\uparrow$) across number of function evaluations (NFEs) for LM1B and OWT. Comparison at matched entropy at the generative frontier across inference-time knobs is given in Tables \ref{tab:decode-frontier-lm1b} and \ref{tab:decode-frontier-owt}. Reported OWT values for DBTM are with linear attention at matched parameter count. Standard DiT comparison is given in Table \ref{tab:language-results-nfe}.}
\label{tab:language-results}
\vspace{-10pt}
\end{table}

\subsection{Scaling the One-Step Map to Language}
\label{sec:exp-language}

\paragraph{Experimental setup.}
We evaluate unconditional generation on LM1B \citep{chelba2013one} and OpenWebText (OWT) \citep{Gokaslan2019OpenWeb}, reporting generative perplexity (Gen-PPL) under GPT-2-Large \citep{radford2019language} and within-sequence unigram entropy (Ent.). Since a model that collapses onto repeated tokens can reach arbitrarily low perplexity, we compare perplexities only among checkpoints whose entropy remains in the range of the data entropy. Baselines include distilled discrete diffusion and continuous flow map baselines evaluated at NFE budgets matched to ours. We report two variants: \textbf{DBTM}, trained with the transport, boundary, and anchor losses of Section \ref{sec:training-objectives}, and \textbf{DBTM + ril}, which adds the refinement-in-loop objectives of Section \ref{sec:refinement-in-loop}. Further experiment details in App.~\ref{app:language-exp-details}.

\paragraph{Discrete BTM is competitive with distilled flow maps.}
Table \ref{tab:language-results} compares against distilled few-step models and simplex flows at matched NFE. On LM1B, DBTM with ril outperforms every distilled discrete diffusion model at 1, 2, and 4 NFE, and is competitive with the best flow-map baseline while maintaining entropy close to the data. On OWT, with linear attention at matched parameter count, DBTM with ril achieves the lowest perplexity across all non-collapsed entropy methods at every NFE budget with the closest entropy to the data. At 4 NFE, DBTM without ril reaches the lowest perplexity overall (52.2) while DBTM + ril attains the highest entropy in the table (5.52), so the two variants trace the two ends of the quality-diversity frontier. Refinement-in-loop yields significant gains in the low NFE regime, reducing 1-NFE perplexity by 58\% on LM1B and by 65\% on OWT. Figures \ref{fig:discrete-btm} and \ref{fig:generative-frontier-lm1b} show the resulting perplexity-entropy frontier, where DBTM and DBTM + ril dominate the baselines across entropies within the range of the data. Notably, DBTM requires no teacher flow, fewer training steps, and no time conditioning. Example generations are in App.~\ref{app:example-generations}.

\paragraph{Anchoring at the phase transition prevents mode collapse.}
Theorem \ref{thm:optimal-t} gives $t^\star$ as a function of the vocabulary size, noise scale, and schedule exponent. Sweeping $t_{\text{anchor}}$ around this value on LM1B (Table \ref{tab:anchor-ablation}) demonstrates its importance. Anchoring before $t^\star$ supervises toward a fixed target while the endpoint is undecided, and the model collapses to repetitive tokens, whereas anchoring after $t^\star$ (e.g.\ $t_{\text{anchor}}=0.85$) leaves the post-transition interval unsupervised, resulting in higher Gen. PPL at equal training time (Table \ref{tab:anchor-ablation}).

\paragraph{Refinement enables more effective use of step budget.}
Table \ref{tab:language-results-nfe} evaluates a single DBTM checkpoint across NFE budgets: quality improves consistently with increasing budget until the sequence is fully committed, at which point it saturates. Distilled flow models instead move along a fixed time axis in discrete jumps, and masked diffusion unmasks on a fixed schedule with factorization error growing as NFE shrinks. Because every DBTM evaluation returns a clean-sequence distribution from any intermediate state, one model produces a one-shot proposal and then, via Section \ref{sec:self-correction}, renoises low-quality positions and remaps them against committed context. The commitment condition also acts as a self-stopping rule, so the NFE count adapts to the sample rather than being fixed in advance as in flow map and flow matching samplers.

\subsection{Reasoning with Discrete BTMs}
\label{sec:exp-reasoning}
\begin{wraptable}{r}{0.6\textwidth}
\vspace{-\baselineskip}
\centering
\small
\resizebox{\linewidth}{!}{%
\begin{tabular}{ll cccc cc}
\toprule
& & \multicolumn{4}{c}{\textbf{Sudoku}} & \multicolumn{2}{c}{\textbf{GSM8K}} \\
\cmidrule(lr){3-6}\cmidrule(lr){7-8}
Category & Method & NFE & Easy & Med. & Hard & NFE & Acc. (\%,$\uparrow$) \\
\midrule
\multirow{2}{*}{AR}
 & Sample & 128 & 13.9 & 5.1  & 0.6  & 512  & 53.9 \\
 & Greedy & 128 & 14.6 & 5.1  & 1.0  & 512  & 63.3 \\
\midrule
\multirow{2}{*}{Discrete}
 & MDLM & 128 & 92.0 & 77.1 & 30.2 & 1024 & 18.0 \\
 & Duo  & 128 & 96.3 & 84.7 & 58.4 & 1024 & 17.2 \\
\midrule
\multirow{3}{*}{Continuous}
 & CANDI & 128 & 79.3 & 45.9 & 16.7 & 1024 & 0.2 \\
 & FLM   & 128 & 94.2 & 82.7 & 44.5 & 1024 & 0.3 \\
 & S-FLM & 128 & 94.8 & 85.2 & 45.0 & 1024 & 18.0 \\
\midrule
 & FMLM+ & 4  & 97.9 & 92.0 & 71.2 & 1 & 0 \\
 & FMLM+ & 16 & 97.8 & 92.6 & 81.4 & 32 & 16.6 \\
\rowcolor{mybg} & \textbf{DBTM + ril} & 4  & \textbf{99.5} & \textbf{97.3} & \textbf{84.6} & 1 & \textbf{0.8} \\
\rowcolor{mybg} \multirow{-4}{*}{Few-Step} & \textbf{DBTM + ril} & 16 & \textbf{99.9} & \textbf{99.4} & \textbf{97.5} & 32 & \textbf{16.8} \\
\bottomrule
\end{tabular}
}
\caption{\textbf{Discrete Beckmann Transport Model Reasoning Performance.} Sudoku columns are exact-solve accuracy (\%) on the 2000-puzzle held-out set and GSM8K is accuracy (\%) on the 1319-problem test set. FMLM+ at NFE 16 is our own evaluation of the released checkpoints on easy and medium, and the hard and GSM8K cells are taken from Table 13 of \citet{agarwal2026posterior}.}
\label{tab:reasoning-results}
\vspace{-\baselineskip}
\end{wraptable}
\paragraph{Experimental setup.}
We evaluate on two conditional reasoning tasks, Sudoku at three difficulty levels, evaluated on one-shot exact-solve accuracy, and TinyGSM code \citep{liu2023tinygsm}, evaluated on zero-shot accuracy on the GSM8K test set \citep{cobbe2021training} following \citet{kim2025train}. For baselines, we compare against many-step autoregressive, discrete diffusion, and continuous diffusion models as well as FMLM+ \citep{agarwal2026posterior}, a distilled few-step flow map model trained for iterative refinement. We match our NFE budget and network parameters to FMLM+ for fair comparison. Details in App. \ref{app:exp-details}.

\paragraph{Refinement enables fixed-point reasoning.}
Reasoning tasks are where the refinement scheme matters most, because a joint one-shot decode of a puzzle or a chain of arithmetic almost always contains locally plausible but globally inconsistent positions. Table \ref{tab:reasoning-results} shows that DBTM matches or exceeds many-step autoregressive, discrete diffusion, and continuous flow baselines on all difficulty levels of Sudoku while using one to two orders of magnitude fewer network evaluations, and exceeds the strongest few-step baseline at the same budget, with the margin widening on the harder difficulty levels. On the 1319-problem GSM8K test set, DBTM achieves 0.8\% solve accuracy in 1 NFE, surpassing continuous baselines CANDI \citep{pynadath2025candi} and FLM \citep{lee2026flow} at 1024 NFEs and FMLM+ \citep{agarwal2026posterior}. At 32 NFEs, DBTM achieves 16.8\% solve accuracy, slightly higher than FMLM+ at matched NFE and closing the gap with S-FLM at 32$\times$ fewer NFE. Autoregressive (AR) methods still dominate on GSM8K, suggesting refinement schemes that exploit the left-to-right structure of reasoning traces as a promising direction for future work.

\section{Conclusion}
We introduced \textbf{Discrete Beckmann Transport Models} (DBTM), a class of discrete generative models that replace the time-dependent velocity of flow matching with a stationary, autonomous field on the simplex. The resulting trajectories are straighter, time-averaged paths whose only stable equilibria lie at the vertices of the simplex. Defining the transport map as the solution to the conservation equation along this field lets us train it end-to-end from data alone, with no time conditioning and no teacher to distill from. Since every application of the map proposes a clean sequence, scaling NFEs take the form of renoise-and-refine steps that self-correct inconsistencies given a set of committed context tokens rather than finer integration steps. Across language modeling and reasoning benchmarks, DBTMs improve on both discrete diffusion and flow baselines.

\section*{Acknowledgments}
The authors thank Michael Albergo and Brian Cheuk-Kit Lee for helpful discussions and guidance. SW is supported by a Kempner Graduate Fellowship. This work is made possible by a gift from the Chan Zuckerberg Initiative Foundation to establish the Kempner Institute for the Study of Natural and Artificial Intelligence. 

\paragraph{Code Availability} 
The code can be found at \url{https://github.com/sophtang/DBTM}.

\bibliographystyle{acl_natbib.bst}
\bibliography{citation.bib}

\clearpage
\beginappendix
\startcontents[app]
\printcontents[app]{l}{1}{\setcounter{tocdepth}{2}}

\newpage
\section{Related Work}
\label{app:related-work}

\subsection{Discrete and Continuous Generative Models for Sequence Data}
Generative modeling of discrete data in the form of sequences or graphs has taken many forms. Discrete diffusion \citep{austin2021structured, lou2023discrete}, masked discrete diffusion \citep{shi2024simplified, sahoo2024simple, ou2024your, zheng2024masked}, and discrete flow matching \citep{gat2024discrete, campbell2024generative} model discrete sequence generation as a continuous-time Markov chain (CTMC) where transitions are defined by a rate matrix. However, since the discrete state space grows exponentially with the vocabulary size and sequence length, discrete diffusion models rely on learning the factorized approximations of the transition probabilities that require many inference steps to capture the dependencies between tokens for accurate generation \citep{kang2026parallelbench}. To overcome the factorization errors inherent in discrete state-space parameterizations, continuous state-space parameterizations have emerged as an alternative by embedding the flow trajectories in simplex space \citep{stark2024dirichlet, davis2024fisher, tang2025gumbel} or in a latent embedding space \citep{cheng2025alpha, chen2026langflow, dieleman2022continuous, hu2026elf, deschenaux2026language, yang2026continuous, jo2026continuous}. In this work, we model the autonomous flow transporting a Gaussian prior $\mu_0$ on $\mathbb{R}^d$, $d:= VL$, to the data distribution $\mu_1$ supported on the vertex set $\mathcal{V}:=\{e_1, \dots, e_V\}^L$ of the product of simplices $(\Delta^{V-1})^L\subset\mathbb{R}^d$, which is the construction used in discrete flow maps \citep{potaptchik2026discrete, lee2026flow}. 

\subsection{Time-dependent and autonomous flows}
Our framework develops the discrete state space instantiation of \textit{Beckmann Transport Models (BTMs)} \citep{lee2026beckmann}, an approach to the problem of transporting states between distributions $\mu_0$ and $\mu_1$ via a time-independent velocity field $b$ that satisfies the divergence condition $\nabla\cdot (\nu b )=\mu_0-\mu_1$. In contrast to the time-dependent velocity field of standard flow matching \citep{lipman2022flow, albergo2022building, albergo2025stochastic, liu2022flow} that transports $x$ along $b_s(x)=\mathbb{E}_{x_0,x_1}[\dot I_s|I_s=x]$ which is the expectation of the interpolant velocity passing $x$ at time $s$, the autonomous velocity field additionally averages over the time coordinate of the interpolant $s\in [0,1]$, yielding the expectation over all interpolant velocities that pass $x$ at any time. In our work, we derive and empirically validate the unique properties of the autonomous flow on the simplex, specifically showing that it yields straighter flows with favorable Lipschitz properties and stationary attractors at the vertex. We establish the fact that the autonomous velocity field is non-vanishing approaching each vertex but vanishes when reaching it, partitioning the ambient space around the simplex into stationary basins of attraction. 

\subsection{Flow maps in discrete state space}
There exists a long line of work on distilling and constructing few-step generative models that aim to reduce the number of function evaluations needed to integrate the generative ODE or SDE. These include consistency models \citep{song2023consistency, song2023improved}, shortcut models \citep{frans2024one}, MeanFlow \citep{guo2025splitmeanflow}, rectified flow \citep{liu2022flow}, flow maps \citep{boffi2024flow, sabour2025align, boffi2025build}, and many of these frameworks have been extended to the discrete state space \citep{roos2026categorical, lee2026flow, potaptchik2026discrete} and variable-length generation \citep{tang2026expanding}. Most related to our work are flow map language models \citep{lee2026flow} with posterior refinement \citep{agarwal2026posterior} and discrete flow maps \citep{potaptchik2026discrete}, which learn a two-time map $(s,t)$ along the interpolant between a Gaussian prior and the simplex vertices. Despite this shared goal, DBTM differs from flow maps in several aspects, summarized in Table \ref{tab:flowmap-vs-dbtm}. 

Flow map models are explicitly time-conditioned and are typically trained by distilling a pretrained teacher flow $b_t$, requiring supervision over all $(s,t)$ time pairs together with a careful balance of the diagonal (teacher matching) and off-diagonal consistency objectives. DBTM instead trains a single autonomous model conditioned only on the current state, with no teacher and no time embedding, reducing the space of functions that need to be learned during training. For a flow map, each NFE is a discretized jump along the generative trajectory, whereas for DBTM each NFE refines a complete clean proposal, so intermediate iterates remain in the data space and can be treated as a reasoning trace. Since inference only ever visits a small subset of the $(s,t)$ grid, flow map models spend training time on time pairs that are never used at sampling time. Training DBTM from noise-data interpolants instead covers exactly the states encountered at inference. Finally, the autonomous formulation admits a natural stopping criterion: sampling iterations stop once the proposal reaches its fixed point or once all tokens pass a confidence threshold, whereas flow map sampling requires the step schedule to be fixed in advance.

\begin{table}[h!]
  \centering
  \caption{\textbf{Comparison of DBTM with Flow Map Language Models (FMLM) and Discrete Flow Maps (DFM).} Both generate in a fixed small number of function evaluations, but flow maps discretize an ODE between trained $(s,t)$ pairs and need a teacher and time conditioning, whereas DBTM refines a clean proposal with a single autonomous map. The rows contrast the two methods.}
  \label{tab:flowmap-vs-dbtm}
  \small
  \setlength{\tabcolsep}{6pt}
  \renewcommand{\arraystretch}{1.15}
  \begin{tabularx}{\linewidth}{@{}l >{\raggedright\arraybackslash}X >{\raggedright\arraybackslash}X@{}}
    \toprule
     & \textbf{Flow Map LMs / DFM} & \textbf{DBTM (Ours)} \\
    \midrule
    Few-step generation            & Yes & Yes \\
    Time conditioning              & Yes ($s$ and $t$ conditioning) & No \\
    Teacher model                  & Yes & No \\
    Each NFE is                    & Discretized step along ODE & Refinement of clean proposal \\
    Self-stopping                  & No & Yes \\
    Training complexity            & High; requires diagonal/off-diagonal balancing and supervision over all $(s,t)$ time pairs & Low; trains a single model conditioned only on the current state \\
    Training-inference mismatch   & Yes; many $(s,t)$ time jumps are unused at inference & No; training from noise-data interpolants covers all states seen at inference \\
    Self-correction / reasoning trace & No & Yes \\
    \bottomrule
  \end{tabularx}
\end{table}

\section{Background on Autonomous Flows}
\label{app:extended-background}
\subsection{Notation}
\paragraph{Sequences and the simplex}
$V$ is the vocabulary size, $L$ the sequence length, $\ell\in[L]$ a position and $j,k$ vertex indices. $e_j$ is a one-hot vertex of the simplex $\Delta^{V-1}$, $(\mathbb{R}^V)^L$ the ambient space, $x^\ell$ the $\ell$-th position of a sequence, and $M_1$ the data manifold of $V^L$ vertex configurations. $\mu_0$ is the Gaussian prior with noise scale $\sigma$ and $\mu_1$ the data distribution.

\paragraph{Interpolants and time}
$s,t,u$ denote time coordinates, $I_t=\alpha_tx_0+\beta_tx_1$ the interpolant between $x_0\sim \mu_0$ and $x_1\sim \mu_1$, $\dot I_t$ its velocity and $\mu_t$ its marginal. $\alpha_t=(1-t)^a$ is the interpolant time schedule with exponent $a$. For the partial-context interpolant, $t_\ell$ and $\alpha^\ell_t$ are per-position, $\mathcal{C}\subset[L]$ is the clean context set, $f$ the clean fraction, and $c_\ell$ the context indicator.

\paragraph{Flows and the transport map}
$b_t$ is the time-dependent flow matching velocity and $b$ the autonomous (time-independent) velocity. $\nu$ is the occupation measure, $j=\nu b$ the current, $\nu^{(j)},J^{(j)}$ their per-vertex components and $w_j$ the occupancy weights. $X_s$ the autonomous flow, $\tau(x_0)$ the hitting time of $M_1$ and $T$ the autonomous transport map, the long-time limit of $X_s$.

\paragraph{Learned objects and refinement}
$T_\theta=\text{softmax}(f_\theta)$ is the learned map, $q^\ell_\phi$ the per-token quality head on the trunk hidden state $h^\ell$, and $\text{sg}(\cdot)$ the stop-gradient. Losses are $\mathcal{L}_{\text{transport}},\mathcal{L}_{\text{bnd}},\mathcal{L}_{\text{anchor}},\mathcal{L}_{\phi},\mathcal{L}_{\text{ril}},\mathcal{L}_{\text{ril-}\phi}$ each weighted by a $\lambda$ scalar. At inference, $\hat x_{r,k}$ is the partially committed sequence at refinement round $r$ of $k$ total rounds, $\kappa$ is the commit threshold, $\mathcal{C}_r$ is the committed index set after round $r$, and NFE is the number of network evaluations, which coincides with the number of refinement rounds $k$. 

\subsection{Divergence Condition}
\label{app:divergence-cond}
\begin{proposition}[Autonomous flow satisfies divergence condition]
\label{prop:divergence}
    For uniformly sampled $t\sim \mathcal{U}([0,1])$, the autonomous velocity $b(x):=\mathbb{E}_{t,x_0,x_1}[\dot I_t|I_t=x]$ satisfies the divergence condition:
    \begin{align}
        \nabla\cdot(b\nu)=\mu_0-\mu_1
    \end{align}
\end{proposition}

\textit{Proof.} For a family of interpolants $(\alpha_t, \beta_t)_{t\in [0,1]}$, the time-dependent velocity, by the chain rule $\dot I_t=\dot \alpha_tx_0+\dot \beta_tx_1$, is:
\begin{align}
    b_t(x)=\mathbb{E}[\dot I_t|I_t=x]=\dot \alpha_t\mathbb{E}[x_0|I_t=x]+\dot \beta_t\mathbb{E}[x_1|I_t=x]=\dot \alpha_t\eta_0(t,x)+\dot \beta_t\eta_1(t, x)\label{proof-eq:b_t}
\end{align}
where we define $\eta_0(t,x):=\mathbb{E}[x_0|I_t=x]$ and $\eta_1(t, x):=\mathbb{E}[x_1|I_t=x]$. $b_t$ satisfies the continuity equation with time-dependent marginal $\mu_t$:
\begin{align}
    \partial_t\mu_t+\nabla \cdot(b_t\mu_t)=0
\end{align}
The autonomous flow is given by integrating over $t\in [0,1]$ of the time-dependent flow:
\begin{align}
    b(x)=\mathbb{E}_t[b_t(x)|I_t=x]=\frac{\int_0^1 b_t(x)\mu_t(x)\pi(t)dt }{\int_0^1 \mu_t(x)\pi(t)dt}=\frac{1}{\nu(x)}\int_0^1b_t(x)\mu_t(x)dt\label{eq:proof-bx-1}
\end{align}
where $\pi(t)$ is the distribution from which $t$ is sampled (typically uniform on $[0,1]$). For uniform $\pi(t)$, we can integrate the continuity equation over $t\in [0,1]$ to get:
\begin{align}
    \int_0^1\partial_t\mu_tdt+\nabla \cdot\int_0^1b_t\mu_tdt=0,
    \label{eq:proof-bx-2}
\end{align}
The first integral reduces to $\int_0^1\partial_t\mu_tdt=\mu_1-\mu_0$ and the second integral can be simplified using the definition of $b(x)$ in \eqref{eq:proof-bx-1} to get $(\mu_1-\mu_0)+\nabla\cdot(b\nu)=0$ or $\nabla\cdot(b\nu)=\mu_0-\mu_1$. \hfill $\square$ 

\begin{corollary}[Endpoint parameterization]
    Let $q_x(t):=\mu_t(x)/\nu(x)$ be the probability density of $t\in [0,1]$ given $x$. Then, the endpoint parameterization of the autonomous flow is given by:
    \begin{align}
        b(x)=\mathbb{E}_{t\sim q_x(t)}\left[\lambda_t(\eta_1(t,x)-x)\right]=Z(x)\left(\bar\eta_1(t,x)-x\right)\label{prop-eq:endpoint-b}
    \end{align}
    where $Z(x):=\mathbb{E}_{t\sim q_x}[\lambda_t]$ is a scalar and $\bar\eta_1(t,x):=\tfrac{1}{Z(x)}\mathbb{E}_{t\sim q_x}[\lambda_t\eta_1(t,x)]$ is the averaged endpoint prediction.
\end{corollary}
\textit{Proof.} Given \eqref{proof-eq:b_t}, we can write $x=\alpha_t\eta_0(t,x)+\beta_t\eta_1(t,x)$ and for $t<1$, we can substitute $\eta_0=(x-\beta_t\eta_1)/\alpha_t$ to get the endpoint form:
\begin{align}
    b_t(x)=\frac{\dot \alpha_t}{\alpha_t}x+\left(\dot \beta_t-\frac{\dot \alpha_t\beta_t}{\alpha_t}\right)\eta_1(t,x)\label{proof-eq:b_t2}
\end{align}
Since $\beta_t=1-\alpha_t$, we have $\dot \beta_t= -\dot \alpha_t$ and the bracket simplifies to $-\dot \alpha\left(1+\tfrac{1-\alpha_t}{\alpha_t}\right)=-\dot \alpha_t/\alpha_t=:\lambda_t$. Therefore, \eqref{proof-eq:b_t2} reduces to: 
\begin{align}
    b_t(x)=\lambda_t(\eta_1(t,x)-x), \quad \lambda_t:=-\frac{\dot \alpha_t}{\alpha_t}=\frac{a(1-t)^{a-1}}{(1-t)^a}= \frac{a}{1-t}
\end{align}
Substituting into \eqref{eq:proof-bx-1} and writing $\nu(x)=\int_0^1\mu_t(x)dt$, we get:
\begin{align}
    b(x)=\frac{1}{\nu(x)}\int_0^1\lambda_t(\eta_1(t,x)-x)\mu_t (x)dt =\int_0^1\lambda_t(\eta_1(t,x)-x)q_x(t)dt
\end{align}
which is the first equality in \eqref{prop-eq:endpoint-b}. Since $q_x$ integrates to one, $Z(x)<\infty$ and splitting the integrand linearly gives $b(x)=Z(x)\bar\eta_1(t,x)-Z(x)x=Z(x)(\bar\eta_1(t,x)-x)$, which is the second equality. \hfill $\square$

\begin{remark}[Using the transport map to integrate the autonomous flow]\label{remark:autonomous-flow}
    Since $\eta_1(t,x)$ is a conditional expectation of $x_1\in M_1$, $\bar\eta_1(t,x)$ is a convex combination of these and lies on the simplex $\bar\eta_1(t,x)\in (\Delta^{V-1})^L$ for every $x$. Therefore, the autonomous velocity is fully determined by a map onto the product of simplices, which is exactly what $T_\theta:=\mathrm{softmax}(f_\theta)$ predicts. The trajectory of the autonomous flow is unchanged with respect to $Z(x)$ up to a time reparameterization, so we only need to learn the map. 
\end{remark}

\section{Theoretical Results}
Here, we present the proofs and derivations for the theoretical results of our method. In App \ref{app:dynamics-proofs}, we show the favorable properties of the time-dependent and autonomous flows, including the stationary attractors at the vertices. In App \ref{app:phase-transition-proof}, we derive the commitment time of the stochastic interpolant to a vertex of the simplex as a function of the prior variance, the simplex dimension, and the exponent in the interpolant schedule. In App \ref{app:occupation-measure-proof}, we derive the form of the occupation measure defining the time spent at each point in ambient space. Finally, we show in App \ref{app:vertex-vanish-proof} that the autonomous flow velocity vanishes at the vertices.

\subsection{Dynamics of Time-Dependent and Autonomous Flows}
\label{app:dynamics-proofs}

\subsubsection{Autonomous Flow as Time-Averaged Dynamics}
\label{app:time-average}
\restate{\timeaverage}

\begin{proof}
    Let $u\sim\mathrm{Unif}[0,1]$ be independent of $(x_0,x_1)$, so that $(u,I_u)$ has joint density $(t,x)\mapsto \mu_t(x)$ on $[0,1]\times\mathbb{R}^d$. Marginalizing in $t$, the density of $I_u$ is $\bar \mu(x):=\int_0^1\mu_s(x)ds$, and by Bayes' rule the conditional density of $u$ given $I_u=x$ is $\mu_t(x)/\bar \mu(x)$ wherever $\bar \mu(x)>0$. By independence, $\mathbb{E}[\dot I_u\mid u=t,\,I_u=x]=b_t(x)$, so the tower property gives:
    \begin{align}
        b(x):=\mathbb{E}_{u,x_0,x_1}[\dot I_u\mid I_u=x]
        =\mathbb{E}\big[\,b_u(x)\,\big|\,I_u=x\big]
        =\int_0^1\frac{\mu_t(x)}{\bar \mu(x)}b_t(x)dt
    \end{align}
    which is \eqref{eqn:defn auto velocity}.
\end{proof}

\subsubsection{Attractors}
\label{app:attractors}
\begin{restatable}[Attractors of Autonomous Flow]{proposition}{attract}\label{prop:attractors auto}
    In the dynamical system 
    \begin{equation}
        \frac{d}{dt}X_s=b(X_s),
    \end{equation}
    the only attractors are the vertices of simplex $e_1,\dots, e_V$ at all time $s\in [0,\infty)$. 
\end{restatable}
\begin{proof}
     The appendix \ref{app:vertex-vanish-proof} already proves that velocities vanish at the vertices $e_1,\dots,e_V$. By Proposition \ref{thm:convergence-simplex}, given any initial point $x_0$, the trajectory $X_s$ with $X_{s=0}=x_0$ approaches one of the vertices as $s\to \infty$. Therefore, the basin of attraction of the vertices $e_1,\dots, e_V$ partitions the space $\mathbb R^V$. Any additional attractor comes with a non-zero measure of basin of attraction, so there are no additional attractors. 
\end{proof}

Now we study the dynamics of time-dependent flow matching
\begin{equation}
\label{eqn:time dependent flow}
    \frac{d}{dt}X_t=b_t(X_t).
\end{equation}
We study the attractors of $b_t(x_t)$ for fixed time $t$. As $t$ varies, the attractors of $b_t(x_t)$ move. The trajectory of \eqref{eqn:time dependent flow} is pulled by attractors that move with time. We study $b_t(X_t)$ making two assumptions below:
\begin{assumption}
\label{assu:time dependent flow}
    \leavevmode
    \begin{enumerate}
        \item[(i)] The prior is Gaussian $x_0\sim \mathcal N(c,\, \sigma^2\boldsymbol{I}_V)$ centered at $c\in \mathbb R^V$ with variance $\sigma ^2$ and the prior is symmetric with $c_i=c_j$ for all $i$ and $j$;
        \item[(ii)] The interpolant $I_t=(1-t)x_0+tx_1$ is linear.
    \end{enumerate}
\end{assumption}
Note that Assumption \ref{assu:time dependent flow} (ii) is without loss of generality, because any straight line interpolant $I_t=\alpha_tx_0+\beta_tx_1$ is a time reparameterization of the linear interpolant $I_t=(1-t)x_0+tx_1$. Our results hold for any linear interpolant up to a time reparameterization. The following lemma says that regardless of where the prior is centered, the attractors of $b_t(x_t)$ lie on the simplex $\Delta^{V-1}$ for any $t$. 

\begin{lemma}[Attractors Lie on Simplex]
For Gaussian prior $x_0\sim\mathcal N(c,\sigma^2\boldsymbol{I}_V)$ and any interpolant $I_t=\alpha_tx_0+\beta_tx_1$ and any time $t$, all attractors lie on the simplex $\Delta^{V-1}$.
\end{lemma}
\begin{proof}
    We state without proof that the velocity $b_t(x_t)$ under Assumption \ref{assu:time dependent flow} is
    \begin{equation}
    \label{eqn:form of time dependent flow}
        b_t(x)=\frac{1}{1-t}(D_t(x)-x),
    \end{equation}
    where the denoiser $D_t(x)$ takes the form
    \begin{equation}
        \label{eqn:time dependent denoiser}
        D_t(x)=\sum_iw_i(x,t)e_i,\quad \text{with}\quad w_i(x,t)=\softmax_i\left(\log p_i+\frac{tx_i}{\sigma^2(1-t)^2}-\frac{tc_i}{\sigma^2(1-t)}\right),
    \end{equation}
    where $c_i$ is the $i$-th coordinates of the center $c$ of the prior. For fixed $t$, an attractor $x=(x_i)_{i=1,\dots, V}$ satisfies the equation $b_t(x)=0$, that is 
    \begin{equation}
        \label{eqn:attractor eqn}
        x_i=\softmax_i\left(\log p_i+\frac{tx_i}{\sigma^2(1-t)^2}-\frac{tc_i}{\sigma^2(1-t)}\right),
    \end{equation}
    for $i=1,\dots,V$. Since $x=(x_i)$ takes the form of a softmax, $x$ lie on the simplex $\Delta^{V-1}$. Notice that $c_i=c_j$ by Assumption \ref{assu:time dependent flow} (i). The softmax equation \eqref{eqn:attractor eqn} becomes $c$-independent
    \begin{equation}
        \label{eqn:attractor eqn indpt c}
        x_i=\softmax_i\left(\log p_i+\frac{tx_i}{\sigma^2(1-t)^2}\right).
    \end{equation}
\end{proof}
\begin{lemma}[Asymptotic Formulae for Attractor Locations]
\label{lemma:asymptotic formula}
\leavevmode
    \begin{enumerate}
        \item[(i)] For small $t\gtrsim 0$, the coordinates of the unique attractor $A^*$ are 
        \begin{equation}
        \label{eqn:early attractors}
            x_k=p_k+\frac{p_k\left(p_k-m_2\right)}{\sigma^2}t+\mathcal O(t^2),
        \end{equation}
        for $k=1,\dots, V$, and $m_2=\sum_jp_j^2$ is the second moment of the target distribution.

        \item[(ii)] For large $t\lesssim 1$, there are $V$ attractors labeled by $j=1,\dots,V$. Let $\beta=\frac{t}{\sigma^2(1-t)^2}$. The coordinates of the attractors $A_j$ are respectively given by expansions in large $\beta$
        \begin{equation}
        \label{eqn:late attractors}
            x_k=\delta_{kj}+c_{kj}\,e^{-\beta}+\mathcal{O}(\beta\,e^{-2\beta}),
        \end{equation}
        for $\delta_{kj}$ the Kronecker delta and the coefficients $c_{jk}$ are given by
        \begin{equation}
        \label{eqn:coefficient cases}
        c_{kj}=
            \begin{cases}
                \frac{p_k}{p_j},\quad &k\neq j\\
                \frac{p_j-1}{p_j},\quad &k=j
            \end{cases}
        \end{equation}
    \end{enumerate}
\end{lemma}
\begin{proof}
    (i) is shown by setting $x_k=p_k+u_k\, t+\mathcal O(t^2)$ and plug into equation \eqref{eqn:attractor eqn} to solve for coefficients $u_k$.

    For (ii), let \(u=1-t\) and set \(x_k=\delta_{jk}+r_k\). Taking the ratio of the \(k\)-th and \(j\)-th equation of \eqref{eqn:attractor eqn indpt c} gives
    \begin{equation}
        \frac{x_k}{x_j}
        =
        \frac{p_k}{p_j}
        \exp\left(
        \frac{t(x_k-x_j)}{\sigma^2(1-t)^2}
        \right).
    \end{equation}
    Since \(x_j\to 1\) and \(x_k\to 0\) for \(k\neq j\) as \(t\to 1\),
    \begin{equation}
        x_k
        =
        \frac{p_k}{p_j}
        \exp\left(-\frac{t}{\sigma^2(1-t)^2}\right)
        (1+o(1)),
        \qquad k\neq j.
    \end{equation}
    The expression for \(x_j\) then follows from \(\sum_k x_k=1\).
\end{proof}
\begin{remark}
The implication of Lemma \ref{lemma:asymptotic formula} (i) is that, as $t$ increases from $0$, the unique attractor $A^*$ moves towards the vertex $l$ with the largest target probability, $p_l=\arg\max_k\,p_k$. Since the first-order coefficient $p_k(p_k-m_2)\sim p_k^2$ at $k=l$ is much larger than that at $k\neq l$. If a vertex $k$ has $p_k<p_l^2$, and hence $p_k-m_2<0$, then the attractor will move away from vertex $k$ due to the negative first-order coefficient.
\end{remark}

Lemma \ref{lemma:asymptotic formula} (ii) establishes the existence of an attractor near each vertex at time $t\lesssim 1$. We are interested in when the attractor $A_j$ associated with vertex $j$ appears. In Proposition \ref{prop:largest attractor}, we show that the attractor $A_l$ corresponding to the vertex with the largest target probability $p_l=\arg\max_k\,p_k$ forms at $t=0$. We call $A_l$ the \textit{dominant attractor}. Attractor $A_j$ where $j\neq l$ forms much later. In Proposition \ref{prop:attractor formation times} and Corollary \ref{thm:larger attractors form earlier}, we show that the time $t_j$ at which attractor $A_j$ forms is given in terms of the ratio between the target probability on the $j$-th vertex and the highest target probability $p_l$. 

\begin{proposition}[The Dominant Attractor]
    \label{prop:largest attractor}
    The attractor $A_l$ corresponding to the vertex with the largest target probability $p_l=\arg\max_kp_k$ forms at time $t=0$.
\end{proposition}
\begin{proof}
    We show that there is an attractor $A^*$ that exists for all time $t$, and that attractor $A^*$ moves to vertex $l$, the vertex with the largest target probability. 
    
    First, we notice that off-simplex directions are attracted to the vertex. The vector field $b_t(x)$ restricting the simplex $x\in \Delta^{V-1}$ has the same attractors as the vector field $x\in \mathbb R^V$. The off-simplex direction is the vector $\mathbf{1}$ of all 1s. Project the trajectory \eqref{eqn:time dependent flow} onto $\mathbf 1$, we have
    \begin{equation}
        \label{eqn:projection to off simplex}
        \frac{d}{dt}\langle X_t,  \mathbf 1\rangle =\frac{1}{1-t}(\langle D_t(X_t), \mathbf 1\rangle-
        \langle X_t,\mathbf{1}\rangle ).
    \end{equation}
    We rewrite scalar $s=\langle X_t,\mathbf 1\rangle $, and recall that $\langle D_t(X_t), \mathbf 1\rangle =1$. The entries of $D_t(X_t)$ sum to 1 according to equation \eqref{eqn:time dependent denoiser}. Equation \eqref{eqn:projection to off simplex} becomes an equation of the projection $s$
    \begin{equation}
        \label{eqn:projection s}
        \frac{ds}{dt}=\frac{1-s}{1-t}.
    \end{equation}
    whose solution approaches $s=1$ exponentially as $t$ increases. Since projections onto the off-simplex direction approach 1 exponentially, the off-simplex direction attracts. Hence the vector field $b_t(x)$ restricting to the simplex $x\in \Delta^{V-1}$ has the same attractors as the unrestricted vector field for $x\in \mathbb R^V$.

    We focus on $b_t(x)$ for $x$ on the simplex for fixed $t$. Let $x=(x_i)_{i=1,\dots, V}$ and $p=(p_i)_{i=1,\dots,V}$ both be distributions. Let $\beta_1=\frac{t}{\sigma^2(1-t)^2}$ and consider the potential function
    \begin{equation}
        \label{eqn:lyapunov}
        V(x)=\kl(x\,\|\,p)-\frac{\beta_1}{2}\|x\|^2.
    \end{equation}
    We show that $V(x)$ is a Lyapunov function. Denoting $P_i$ as the $i$-th softmax in \eqref{eqn:time dependent denoiser} and $Z$ is the shared denominator among $P_i$ for $i=1,\dots,V$, the gradient of $V(x)$ against the $i$ coordinate is 
    \begin{equation}
    \label{eqn:grad of lyapunov}
    \nabla_i V(x)=-\beta_1x_i+\log x_i-\log p_i+1=-\log\frac{P_i}{x_i}-\log Z+1.
    \end{equation}
    The change of potential along $b_t(x)$ is
    \begin{align}
        \label{eqn:proof lyapunov}
        b_t(x)\cdot\nabla V(x)&=\frac{1}{1-t}\sum_i(P_i-x_i)(-\log\frac{P_i}{x_i}-\log Z+1)\nonumber\\
        &=-\frac{1}{1-t}\sum_i(P_i-x_i)\log\frac{P_i}{x_i},
    \end{align}
    where we have moved the constants $-\log Z+1$ out of the sum and applied $\sum_iP_i=\sum x_i=1$. Rewriting the above equation, we have
    \begin{equation}
        \label{eqn:two divergences}
        b_t(x)\cdot \nabla V(x)=-\frac{1}{1-t}(\kl(P\,\|\,x)+\kl(x\,\|\,P))\leq0,
    \end{equation}
    where $P=(P_i)_{i=1,\dots, V}$. The equality in \eqref{eqn:two divergences} is obtained precisely when $x=P$. That is, $x$ is an attractor as defined in equation \eqref{eqn:attractor eqn}. Let $x^*$ be the position of such an attractor. Define the Voronoi cell $C=\{x\,|\,x\in \Delta^{V-1},\ x_l>x_i \text{, for all $i\neq l$}\}$ as the set of points on the simplex whose closest point is the vertex $l$. It is clear that the minimizer $x^*$ of $V(x)$ is in cell $C$; otherwise, permuting the coordinates of $x^*$ yields a strictly lower value $V(x)$. Additionally, the minimizer of $V$ is unique by equation \eqref{eqn:grad of lyapunov}, where $\nabla_iV(x)=0$ has a unique solution for generic values of $p_i$.

    We have shown there is a unique attractor $A^*$ corresponding to the minimum of $V(x)$, and $A^*$ is closest to the vertex $l$. Now we show that as $t\to 1$, the attractor $A^*$ approaches vertex $l$. As $t\to 1$, we have $\beta_1\to \infty$, and the quadratic term in $V(x)$ dominates. Under the constraint $\sum_ix_i=1$, the quadratic term is minimized by $x^*$ with $x^*_j=1$ for some $j$ and $x^*_i=0$ for $i\neq j$. But since $x^*\in C$, it must be the vertex $l$ with $x^*_l=1$.
\end{proof}

\begin{proposition}
\label{prop:attractor formation times}
    \textup{(Times when Other Vertex Attractors Form)}
    Let $p_l$ be the largest target probability and $m$ be the number of probabilities comparable to but less than $p_l$. For every $j\neq l$, the time $t_j$ when vertex attractor $A_j$ appear is given by
    \begin{equation}
        \label{eqn:attractor appearance times}
        \frac{t_j}{(1-t_j)^2}\simeq\sigma^2\beta^*,\quad\text{where }\ \  \beta^*-\log \beta^*=\log\frac{p_l}{p_j}+1+
        \log(m+1)
    \end{equation}
    Furthermore, one can write $t_j$ in terms of the Lambert $W$-function:
    \begin{equation}
        \frac{t_j}{(1-t_j)^2}=-\sigma^2\,W\left(-\frac{p_j}{p_l(1+m)e}\right),
    \end{equation}
    where $e$ is the base of the natural logarithm.
\end{proposition}
\begin{proof}
    Attractor $A_j$ with coordinates $x=(x_i)_{i=1,\dots,V}$ satisfies the attractor equation \eqref{eqn:attractor eqn}. Let $\beta_1=\frac{t}{\sigma^2(1-t)^2}$ and $\beta_2=\frac{t}{\sigma^2(1-t)}$ denote the $t$-dependence in the equation. Taking the ratio between the $i$-th coordinate equation and the $j$-th coordinate equation, the normalization of the softmax cancels, and we have: 
    \begin{equation}
        \label{eqn:attractor ratio eqn}
        \frac{x_i}{x_j}=\frac{p_i}{p_j}\,\exp(\beta_1(x_i-x_j)-\beta_2(c_i-c_j))=\frac{p_i}{p_j}\exp(\beta_1(x_i-x_j)).
    \end{equation}
    For fixed $j$, change variables from $p_i$, $x_i$ and $c_i$ to $h_i$, $\xi_i$ and $\theta_i$ by including the scale of $\beta_1$ into the variables.
    \begin{equation}
        \label{eqn:change of variable 1}
        h_i=\log \frac{p_i}{p_j};\quad x_i=\frac{\xi_i}{\beta_1};
    \end{equation}
    To normalize $x_i$, we set
    \begin{equation}
        \label{eqn:change of variable 2}
        S=\sum_{i\neq j}\xi_i;
    \end{equation}
    and the normalization condition is
    \begin{equation}
        \label{eqn:normalizations}
        x_j=1-\frac{S}{\beta_1}
    \end{equation}
    Denote $\delta=\frac{S}{\beta_1}$. $\delta$ is small as attractor $A_j$ is near vertex $j$ by Lemma \ref{lemma:asymptotic formula} (ii). Rewriting equation \eqref{eqn:attractor ratio eqn} in the new variables and constants, we have
    \begin{equation}
    \label{eqn:attractor eqn new variables}
        \frac{\xi_i/\beta_1}{1-\delta}=\exp(h_i+\xi_i-\beta_1+S).
    \end{equation}
    We wish to expand the left-hand side of \eqref{eqn:attractor eqn new variables} to leading order in $\delta$. The expansion gives
    \begin{equation}
    \label{eqn:expand lhs}
        \frac{\xi_i}{\beta_1}\left(1+\mathcal O\left(\delta\right)\right)=\exp (h_i+\xi_i-\beta_i+S).
    \end{equation}
    Dropping the terms higher order in $\delta$ and rearranging \eqref{eqn:expand lhs} we have 
    \begin{equation}
    \label{eqn:attractor eqn after expand}
        \xi_i\,e^{-\xi_i}=\beta_1\, e^{h_i-\beta_1+S}.
    \end{equation}
    The attractor $A_j$ exists when the equation above has a solution $\xi_i$ and $S$ satisfying the consistency equation \eqref{eqn:change of variable 2}. Let the solution to \eqref{eqn:attractor eqn after expand} be $\xi_i(S)$ and define 
    \begin{equation}
        \label{eqn:gs}
        G(S)=\sum_{i\neq j}\xi_i(S).
    \end{equation}
    Attractor $A_j$ correspond to $\xi_i(S^*)$ and $S^*$ satisfying the consistency equation
    \begin{equation}
    \label{eqn:fixed point}
        G(S^*)=S^*.
    \end{equation}
    $S^*$ is a fixed point of the map $G$. We know the derivative of the map $G$. Take the log of both sides of equation \eqref{eqn:attractor eqn after expand} and differentiate with respect to $S$; we have the derivatives
    \begin{equation}
    \label{eqn:derivatives}
        \frac{d\xi_i}{dS}=\frac{\xi_i}{1-\xi_i};\quad \frac{dG}{dS}=\sum_{i\neq j}\frac{\xi_i}{1-\xi_i}.
    \end{equation}
    Notice that $G(0)>0$ as a sum of positive terms for finite $\beta_1$ and $\beta_2$. Therefore $S^*$ is the first $S$ such that $G(S)-S=0$, and $G(S)-S$ must cross 0 from above, and the derivative of $G(S)-S$ at $S^*$ is negative:
    \begin{equation}
    \label{eqn:derivative less than 1}
        G'(S^*)-1<0.
    \end{equation}
    We stress that for the dominant attractor $A_l$, we have $G(0)=0$ and $S^*=0$, and the derivative condition \eqref{eqn:derivative less than 1} does not apply.

    To solve equation \eqref{eqn:attractor eqn after expand} given the right-hand side, we notice that the left-hand side function $\xi_i\to \xi_i\, e^{-\xi_i}$ has two branches. For a given right-hand side, there are two solutions $\xi_i$, and we are interested in the solution $\xi_i\in [0,1]$. First we notice that $\xi_i\to 0$ as $t\to 1$ because
    \begin{equation}
        \label{eqn:lower range of xi}
        \xi_i=\beta_1x_i=\beta_1x_j\,e^{h_i}\,e^{\beta_1(x_i-x_j)}\leq\beta_1\, e^{h_i}\,e^{-\beta_1(1+o(1))}\to 0,
    \end{equation}
    where we have used that $x_j$ and $x_i$ tending to 1 and 0 respectively as $t\to 1$. As $t$ and $\beta_1$ decreases, solutions $\xi_i(S^*)$ and $S^*$ to equation \eqref{eqn:attractor eqn after expand} and \eqref{eqn:fixed point} vary continuously. We argue that the solutions $\xi_i(S^*)<1$. If not, consider the first time $t^*$ such that some $\xi_i(S^*)$ approaches 1, and the rest of $\xi _k(S^*)$ are below 1. Then we have 
    \begin{equation}
        \frac{\xi_i(S^*)}{1-\xi_i(S^*)}\to  \infty;\quad 0<\frac{\xi_k(S^*)}{1-\xi_k(S^*)}<\infty,
    \end{equation}
    for $k\neq i$. And the sum of the expressions \eqref{eqn:derivatives} is infinity:
    \begin{equation}
        G'(S^*)=\sum_{i\neq j}\frac{\xi_i}{1-\xi_i}\to \infty,
    \end{equation}
    as $t\to  t^*$, contradicting \eqref{eqn:derivative less than 1} where the derivative $G'(S^*)$ must be less than 1. We have shown that the rescaled coordinates of $A_j$ have $\xi_i(S^*)\in (0,1)$.

    Considering equation \eqref{eqn:attractor eqn after expand} for the largest target probability $p_l$ while taking the log of \eqref{eqn:attractor eqn after expand}, we have 
    \begin{equation}
        \label{eqn:log attractor equation}
        \beta_1-\log \beta_1=h_l+\xi_l-\log\xi_l+S.
    \end{equation}
    Recall that $p_l$ is the largest target probability. Subtracting the $l$-th equation with the $i$-th equation, we eliminate $\beta_1$ and obtain
    \begin{equation}
        \label{eqn:attractor equation subtract}
        \xi_i-\log\xi_i=\xi_l-\log\xi_l+(h_l-h_i).
    \end{equation}
    The map $\xi_i\to \xi_i-\log\xi_i$ is rapidly decreasing on $\xi_i\in (0,1)$. Since $h_l-h_i=\log p_l/p_i$, if the $p_l\gg p_i$, then $\xi_l\gg \xi_i$. The sum $S$ in \eqref{eqn:change of variable 2} is dominated by $m$ terms $\xi_s$, where each term corresponds to a probability $p_s$ that is comparable to the largest probability $p_l$.

    At time $t_j$, the vertex $A_j$ disappears when the equation $G(S)-S=0$ is tangent to $0$ at $S=S^*$. For $t>t_j$, the curve $G(S)-S$ crosses $0$ transversally at $S=S^*$. The tangent condition is
    \begin{equation}
        \label{eqn:tangent condition}
        G'(S^*)=1,
    \end{equation}
    where $G'(S^*)$ is summing $\frac{\xi_i}{1-\xi_i}$ from \eqref{eqn:derivatives}. Since $\xi_l\gg\xi_i$, we have $G'(S^*)$ is dominated by $m$ terms $\xi_s$, each corresponds to a probability $p_s$ that is comparable to the largest probability $p_l$. We may assume all $\xi_s=\xi_l=\xi$ for some $\xi$ for all $s$ because of equation \eqref{eqn:attractor equation subtract} and $h_l-h_i=\log{p_l}/{p_s}\simeq 0$. Derivative $G'(S^*)$ equaling 1 becomes
    \begin{equation}
        m\,\frac{\xi}{1-\xi}=1\quad\Longleftrightarrow\quad \xi=\frac{1}{1+m},
    \end{equation}
    Plugging the value of $\xi_l=\xi$ back to equation \eqref{eqn:log attractor equation}, we have the critical $\beta_1^*$ when $t_j$ disappears is when
    \begin{equation}
    \label{eqn:threshold time equation}
        \beta_1^*-\log \beta_1^*=h_l+\frac{1}{1+m}+\log(1+m)+ \frac{m}{1+m}=h_l+1+\log(1+m),
    \end{equation}
    where we can solve for $\beta_1^*$ and hence $t_j$ with the Lambert $W$-function:
    \begin{equation}
        \label{eqn:lambert function}
        \beta_1^*=-W\left(-\frac{p_j}{p_l(1+m)e}\right),
    \end{equation}
    where $\beta_1^*=t_j/\sigma^2(1-t_j)^2$.
    \end{proof}
Note that Proposition \ref{prop:attractor formation times} applies to both origin- and simplex-centered priors. 

\begin{corollary}[Attractors Form in Order]
    \label{thm:larger attractors form earlier}
     Let $p_l$ be the largest target probability. For any $j,k\neq l$, if $p_j>p_k$, then $t_j<t_k$. That is, attractor $A_j$ forms earlier than attractor $A_k$.
\end{corollary}
\begin{proof}
    For the right-hand side of \eqref{eqn:threshold time equation}, there are two choices of $\beta^*_1$, one greater than 1 and the other less than 1. We argue that we select the solution $\beta_1^*>1$. The Jacobian of the vector field $b_t(x)$ in equation \eqref{eqn:form of time dependent flow} is
    \begin{equation}
        \label{eqn:jacobian}
        J(x)=\frac{1}{1-t}\left(J\left(\softmax\left(\log p_i+\frac{tx_i}{\sigma^2(1-t)^2}-\frac{tc_i}{\sigma^2(1-t)}\right)\right)-I\right)
    \end{equation}
    where the Jacobian of the softmax vector field can be computed and simplified using \eqref{eqn:form of time dependent flow}. 
    \begin{equation}
    \label{eqn:expand jacobian}
        J(x)=\frac{1}{1-t}(\beta_1(\diag x-xx^{\top})-I)\propto\beta_1(\diag x-xx^{\top})-I,
    \end{equation}
    where the positive coefficient $1/(1-t)$ is irrelevant. At the time $t_j$ when the $j$-th attractor appears, the coordinates of the attractor have 
    \begin{equation}
        \label{eqn:attractor appear determinant}
        \det J(x)=0.
    \end{equation}
    Consider the map $F:\mathbb R^V\times [0,1]\to \mathbb R^V$ via
    \begin{equation}
        F: (x, \beta_1)\to  b_{t(\beta_1)}(x),
    \end{equation}
    where $t(\beta_1)$ expresses $t$ in terms of $\beta_1=\frac{t}{\sigma^2(1-t)^2}$. The Jacobian is the $x$-derivative of the map $F(x,\beta_1)$. If at $\beta_1^*$ and the attractor $x(\beta_1^*)$, the determinant
    \begin{equation}
        \label{eqn:inverse function thm condit}
        \det J(x)=\det D_xF(x,\beta_1)\neq0,
    \end{equation}
    then by Inverse Function Theorem, there exists a neighbourhood $B$ around $\beta_1^*$ and an neighbourhood $U$ around $x^*$ such that $x(\beta_1):B\to  U$ is a differentiable and satisfies
    \begin{equation}
        \label{eqn:inverse function thm statement}
        F(x(\beta_1),\beta_1)=0,
    \end{equation}
    for $\beta_1\in B$. In other words, $x(\beta_1)$ is a family of fixed points of $b_{t(\beta_1)}(x)$. There is a lower time $\beta'_1=\inf B$ that is smaller than the threshold $\beta_1^*$, contradicting the assumption that $\beta_1^*$ corresponds to the smallest time for attractor $A_j$ to exist. Therefore $\det J(x)=0$ at the threshold $\beta_1^*$ for attractor $A_j$.

    Write $M=\diag x-xx^{\top}$, then according to equation \eqref{eqn:expand jacobian}, the matrix $\beta_1M-I$ is singular at $\beta_1^*$ and attractor $x=x(\beta_1^*)$. Therefore, $1/\beta_1^*$ is an eigenvalue of the matrix $M$. It is well-known that the largest eigenvalue of a matrix of the form $M$ is upper-bounded by $\max_i x_i$, and we have
    \begin{equation}
        \label{eqn:bound of beta}
        \frac{1}{\beta_1^*}\leq\max_i x_i\leq 1.
    \end{equation}
    We have shown that $\beta_1^*\geq1$. In this range, the mapping $\beta_1\to \beta_1-\log\beta_1$ is increasing. In equation \eqref{eqn:attractor appearance times}, a larger $p_j$ corresponds to a smaller $\beta_1^*$, and thus a smaller appearance time $t_j$.
\end{proof}
\begin{proposition}[Most Attractors Form Late]
    \label{prop:other attractors form late}
    Let the prior be standard Gaussian $\mathcal N(c,\mathbb{I})$ for $c$ subject to Assumption \ref{assu:time dependent flow} (i). Except the dominant attractor $A_l$, no other attractors form earlier than $t=\frac{1}{2}$. 
\end{proposition} 
\begin{proof}
    We show that $t_j\geq\frac{1}{2}$ for all $j\neq l$ by a closer look at equation \eqref{eqn:bound of beta}. The eigenvalue $1/\beta_1^*$ of $M=\diag x-xx^{\top}$ is smaller than the largest eigenvalue $\lambda_{\max}$ which we can write as
    \begin{equation}
        \label{eqn:largest eigenvalue}
        \frac{1}{\beta_1^*}\leq\lambda_{\max}=\max_{\|v\|=1}v^{\top}Mv=\max_{\|v\|=1}\var_x(v).
    \end{equation}
    Where $x\in \Delta^{V-1}$ is a distribution and $\var_x(v)$ takes the variance of $v$ over $x$. Let $a=\max_ix_i$, $b=\min_ix_i$ and $c=(a+b)/2$ be the midpoint. By the bias-variance decomposition
    \begin{equation}
        \mathbb E_x[(v-c)^2]=\var_x(v)+(\mathbb E_x[v]-c)^2,
    \end{equation}
    and the variance is bounded by
    \begin{equation}
    \label{eqn:bound variance}
    \var_x(v)
    \leq \mathbb{E}_x\!\left[(v-c)^2\right]
    = \sum_i x_i(v_i-c)^2
    \leq \sum_i x_i\left(\frac{a-b}{2}\right)^2
    \leq \frac{a^2+b^2}{2}
    \leq \frac{\lVert v\rVert^2}{2}
    = \frac{1}{2}.
    \end{equation}
    Combining equation \eqref{eqn:bound variance} and \eqref{eqn:largest eigenvalue} gives $\beta_1^*\geq 2$. When the prior has variance $\sigma^2=1$, we have $t_j\geq1/2$.
\end{proof}

Proposition \ref{prop:attractor formation times} and Corollary \ref{thm:larger attractors form earlier} show that for target distribution $(p_j)_{j=1,\dots,V}$, there are $V$ attractors, one corresponding to each vertex. The larger the probability $p_j$ for vertex $j$, the earlier the attractor $A_j$ forms. From time $t=0$ to 1, attractors form one by one in the interior of the simplex. By Proposition \ref{prop:largest attractor}, the first formed attractor corresponds to the vertex/token of the highest probability $p_j$. Once an attractor $A_j$ forms, it moves towards vertex $j$ according to Lemma \ref{lemma:asymptotic formula} (i) and (ii). A trajectory that samples a low-probability vertex is pulled away at early times $t$ by the attractors that form earlier, which correspond to high-probability vertices/tokens. 

Figure \ref{fig:basins} illustrates this on a four-vertex example: the second attractor appears only after $t=0.5$ despite its probability $0.38$ being nearly equal to the dominant $0.4$, so the dominant basin is the only one present for the first half of the flow. Figure \ref{fig:empirical} confirms the picture at scale: the predicted appearance times \eqref{eqn:attractor appearance times} match simulation up to $V=100$ (Figure \ref{fig:empirical}A), no non-dominant attractor forms before $t=0.5$ in any simulation up to $V=5000$ (Figure \ref{fig:empirical}B), and among non-dominant attractors, larger $p_j$ form earlier (Figure \ref{fig:empirical}C).

\begin{figure}[h!]
    \centering
    \includegraphics[width=1.0\linewidth]{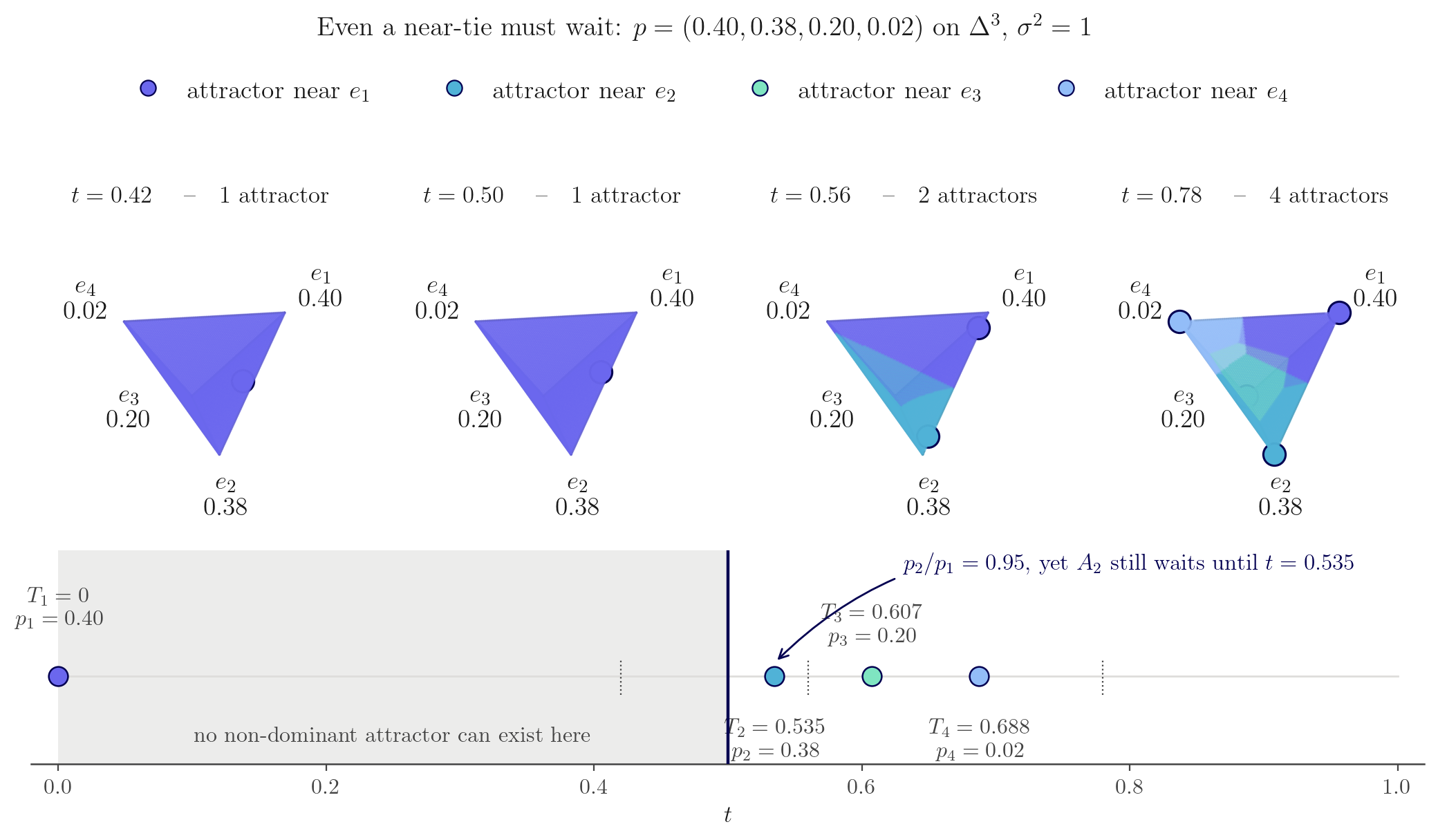}
    \caption{For target probability $p=(0.4,0.38, 0.2, 0.02)$, the attractor corresponding to the vertex with probability 0.38 forms later than $t=0.5$, even though the probability 0.38 is only slightly smaller than the largest probability.}
    \label{fig:basins}
\end{figure}

\begin{figure}[h!]
    \centering
    \includegraphics[width=1.0\linewidth]{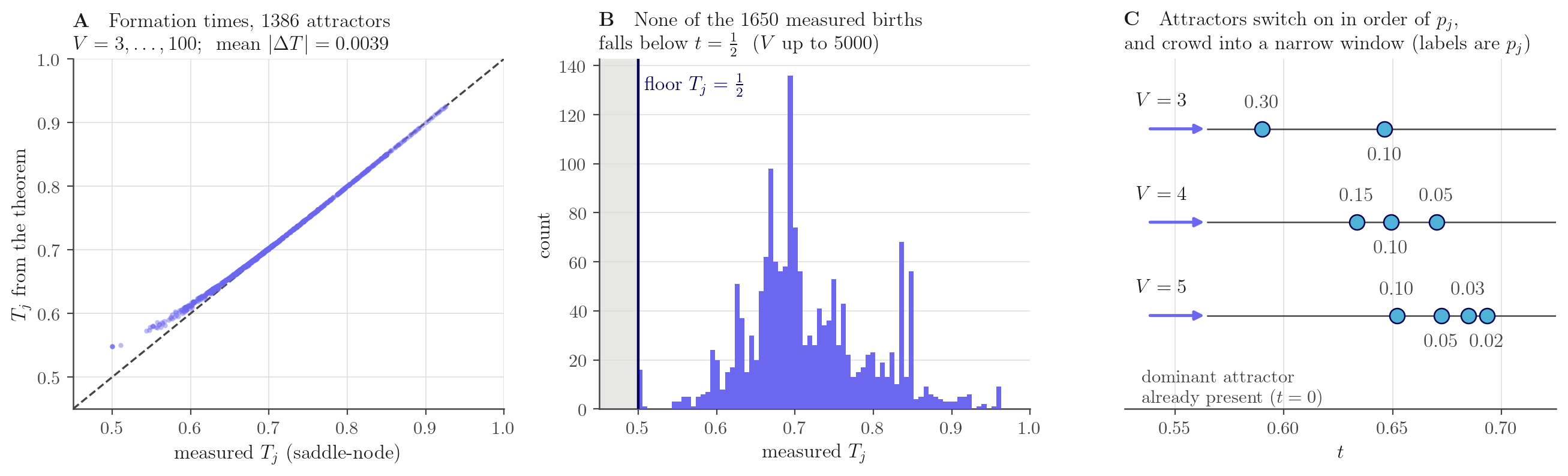}
    \caption{\textbf{(A)} Formula \eqref{eqn:attractor appearance times} predicts the attractor appearance times $t_j$ agrees with simulations. The attractors are simulated up to dimension 100. \textbf{(B)} No simulated non-dominant attractors up to dimension 5000 form earlier than $t=0.5$. \textbf{(C)} Non-dominant attractors form earlier for larger probabilities $p_j$.}
    \label{fig:empirical}
\end{figure}

\subsubsection{Continuity of the Autonomous Velocity}

\begin{proposition}[Continuity of the Autonomous Velocity Field]
    \label{prop:continuity at the vertices}
    Let $\mu_0$ be a prior on $\mathbb R^V$ that decays exponentially towards infinity. Let the interpolant be $I_t=\alpha_tx_0+\beta_tx_1$. Then the autonomous velocity $b(x)$ in equation \eqref{eqn:defn auto velocity} is continuous at the vertices when $\dot \alpha_t=0$ at $t=1$. That is $b(x)\to 0$ as $x\to  e_i$ for any vertex $e_i$.
\end{proposition}
\begin{proof}
    We will use Proposition \ref{prop:occupancy-current} to show that for $x\to  e_i$ for vertex $e_i$, the autonomous velocity is bounded by
    \begin{equation}
        \label{eqn:bounds of autonomous field}
        \|b(x)\|\leq\frac{\|J^{(i)}(x)\|+\|\sum_{j\neq i} J^{(j)}(x)\|}{|\nu^{(i)}(x)|}\to  0.
    \end{equation}
    Let $x\in B_r(e_i)$ the ball of radius $r$ around $e_i$ where we take $r\to 0$. Let $\epsilon_i=x-e_i$ and $\hat\epsilon_i$ be the normalized $\epsilon$. The length scale in the integral formula \eqref{eqn:per vertex current} and \eqref{eq:simplex-drift} is 
    \begin{equation}
        v=\frac{r}{\alpha_t}\to  \infty,
    \end{equation}as $t\to 1$. The prior $\mu_0$ is evaluated at far-field $e_i+v\,\hat\epsilon_i$. Rewrite the integrals $J^{(i)}(x)$ and $\nu^{i}(x)$ in terms of the new variables $v$ and $r$ we have
    \begin{align}
        J^{(i)}(x)&=-p_i\, r^{-(V-1)}\hat\epsilon_i\int_r^{\infty}v^{V-1}\mu_0(e_i+v\,\hat{\epsilon_i})\,dv;\\
        \nu^{(i)}(x)&=p_i\,r^{-(V-1)}\int_r^{\infty}\frac{v^{V-2}\mu_0(e_i+v\,\hat\epsilon_i)}{-\dot\alpha_{t(v)}}\,dv,
    \end{align}
    where $t(v)$ expresses time $t$ in terms of $v$. Since $\mu_0$ decays exponentially towards infinity, there exists an exponential decay that bounds $\mu_0(x)<Ce^{-\lambda x}$ for some $C,\lambda>0$ for all $x$. Up to relaxing $v_0$ to a looser bound $v_1$, the integral in $J^{(i)}(x)$ is upper bounded by 
    \begin{equation}
        \label{eqn:bound of ji}M=\sup_{\|\hat\epsilon_i\|=1}\int_0^{\infty}v^{V-1}\mu_0(e_i+v\,
        \hat\epsilon_i)\,dv<\sup_{\|\hat\epsilon_i\|=1}\left\{\int_0^{v_1}v^{V-1}\mu_0(e_i+v\,\hat\epsilon_i)\,dv\right\}+\int_{v_1}^{\infty}Cv^{V-1}e^{-\lambda v}\,dv<\infty.
    \end{equation}
    and $J^{(i)}(x)$ is upper bounded by $r^{V-1}\|J^{(i)}(x)\|\leq p_iM$. Now we show that for $j\neq i$, the current $J^{(j)}(x)$ is independent of $r$, the distance to the vertex $e_i$ which $x$ is approaching. Since $x$ is approaching $e_i$, the distance from $x$ to vertex $e_j$ is lower bounded by some $l>0$. Bound $\mu_0(e_j+\frac{\epsilon_i}{\alpha_t})\leq Ce^\lambda e^{-\lambda l/\alpha_t}$ and let $L$ be the upper bound of $|\dot\alpha_t|<L$.
    \begin{equation}
    \label{eqn:bound of jj}
        \|J^{(j)}(x)\|\leq p_j L\, \int_0^1\alpha_t^{-(V-1)}Ce^{\lambda}e^{-\lambda l/\alpha_t}\,dt=M_j<\infty,
    \end{equation}
    where $M_j$ is independent of $x$.

    Now we lower bound the denominator of equation \eqref{eqn:bounds of autonomous field}, $\nu^{(i)}(x)$. Define function $h(\delta)$ such that 
    \begin{equation}
        \label{eqn:defn of bound quantity}
        h(\delta):=\sup\{-\dot\alpha_t\,|\,\, \alpha_t\leq\delta\}.
    \end{equation}
    As $\delta\to 0$, we have $h(\delta)\to -\dot\alpha_1$, the derivative of $\alpha_t$ evaluated at $t=1$. The denominator $\nu^{(i)}(x)$ is bounded by 
    \begin{equation}
        \label{eqn:bound of nu}
        r^{V-1}\nu^{(i)}(x)\geq\frac{p_i}{h(r)}\int_1^{\infty}v^{V-2}\mu_0(e_i+v\,\hat\epsilon_i)\, dv\geq \frac{p_i\,c_0}{h(r)},
    \end{equation}
    where $c_0$ is defined as 
    \begin{equation}
        \label{eqn:inf of coefficient}
        c_0=\inf_{\|\hat\epsilon_i\|=1}\int_1^{\infty}v^{V-2}\mu_0(e_i+v\,\hat\epsilon_i)\,dv.
    \end{equation}
    The infimum is taken for positive values over a compact set $\{\hat{\epsilon_i}\,|\,\|\hat{\epsilon_i}\|=1\}$ and thus exists and is positive. Combining equation \eqref{eqn:bound of ji}, \eqref{eqn:bound of jj} and \eqref{eqn:bound of nu}, we have the final bound of the autonomous velocity
    \begin{equation}
        \|b(x)\|\leq\frac{p_iMr^{-(V-1)}+\sum_{j\neq i}M_j}{p_ic_0r^{-(V-1)}/h(r)}=\frac{h(r)}{p_ic_0}\,\left(p_iM+\sum_{j\neq i}M_jr^{V-1}\right)\to 0,
    \end{equation}
    when $h(r)\to 0$ as $r\to 0$ which happens when $\dot\alpha_t\to 0$ as $t\to 1$.
\end{proof}

\subsection{Deriving the Phase Transition and Commitment Time}
\label{app:phase-transition-proof}
\label{app:optimal-t}
\restate{\optimalt}

\begin{proof} We want to determine the optimal $t_{\text{anchor}}$ for the anchor loss:
\begin{align}
    \mathcal{L}_{\text{anchor}}:=\sum_{\ell}\boldsymbol{1}[t>t_{\text{anchor}}]\text{CE}(T(I_t)_\ell, x_{1, \ell})
\end{align}
which preserves the diversity of generations. This means finding the smallest $t_{\text{anchor}}$ such that $[t_{\text{anchor}}, 1]$ is after the phase transition has occurred. 

\textbf{Bayes Posterior.} Since the transport map is $T_\theta(\cdot)=\text{softmax}(f_\theta(\cdot))$ acting on an intermediate state $x=I_t\in \mathbb{R}^d$, independent of time, the minimizer of $\mathcal{L}_{\text{anchor}}$ given $x$ is the Bayes posterior:
\begin{align}
    T^\star(x)=\mathbb{E}[x_1|I_t=x, t\geq t_{\text{anchor}}]
\end{align}
This raises the question: \textit{for which $t_{\text{anchor}}$ is the interpolant at $I_t$ committed to the target $x_1$?} This can also be thought of as identifying the phase transition time in the autonomous flow.

\textbf{Gibbs measure definition.}
Given $x_1=e_k$ and $x_0=z\sim \mathcal{N}(0, \sigma^2I_V)$ with noise scale $\sigma$, the interpolant is given by:
\begin{align}
    x=I_t=(1-\alpha_t)e_k+\alpha_tz, \quad x|j=k\sim \mathcal{N}((1-\alpha_t)e_j, \alpha_t^2\sigma^2I_V)
\end{align}
By Bayes' rule:
\begin{align}
    p(j=k| x)&=\frac{\pi_jp(x|j=k)}{\sum_{l}\pi_lp(x|j=l)}\\
    p(x|j=k)&=\mathcal{N}((1-\alpha_t)e_j, \alpha_t^2\sigma^2I_V)=(2\pi\alpha^2\sigma^2)^{-V/2}\exp \left(-\frac{\|x-(1-\alpha)e_j\|^2}{2\alpha ^2\sigma^2}\right)
\end{align}
Since $(2\pi\alpha^2\sigma^2)^{-V/2}$ doesn't depend on $j$, it cancels in $ p(j=k| x)$. Given $\langle x, e_j\rangle =x_j$ and $\|e_j\|^2$, we can expand the squared distance as: 
\begin{align}
    \|x-(1-\alpha)e_j\|^2=\|x\|^2- 2(1-\alpha)x_j +(1-\alpha) ^2
\end{align}
The $j$-independent terms cancel, and $p(x|j=k)$ reduces to: 
\begin{align}
    p(x|j=k)\propto \exp \left(\frac{(1-\alpha)x_j}{\alpha^2 \sigma^2}\right)=:e^{\beta(t) x_j}, \quad \beta(t):=\frac{1-\alpha}{\alpha^2\sigma^2}
\end{align}
Substituting into the posterior gives:
\begin{align}
    p(j=k| x)&=\frac{\pi_je^{\beta(t)x_j}}{\sum_{l}\pi_le^{\beta(t)x_l}}\label{eq:j=k}
\end{align}
which is the Gibbs or Boltzmann distribution over vocabulary $V$ at inverse temperature $\beta(t)$, tilted by the prior distribution $\pi$, which we assume to be uniform $\pi_j=1/V$ for simplicity. As $t\to 1$ and $\alpha\to 0$, we have $\beta\to \infty$ and the posterior converges to the Dirac delta on the vertex $p(\cdot|x)\to \delta_k$. As $t\to 0$ and $\alpha\to 1$, we have $\beta\to 0$ and the posterior converges to the prior distribution $p(\cdot| x)\to \pi$. 

\textbf{Signal-to-Noise Ratio.}
The interpolant is split into a data component and a noise component: 
\begin{align}
    I_t=\underbrace{(1-\alpha)e_k}_{\text{signal}}+\underbrace{\alpha z}_{\text{noise}}, \quad \alpha=\alpha_t=(1-t)^a, z\sim \mathcal{N}(0,\sigma^2I_V)
\end{align}
where the signal is on the coordinate $k$ scaled with amplitude $(1-\alpha)$ and the noise spreads over all $V$ coordinates with standard deviation $\alpha\sigma$. So, the signal-to-noise ratio (SNR) per coordinate is:
\begin{align}
    \rho:=\rho(t)=\frac{\text{signal amplitude}}{\text{noise std}}=\frac{1-\alpha}{\alpha\sigma}
\end{align}
which can be interpreted as how much larger the signal is above the noise floor at time $t$. For $t\to 1$ near the data, $\rho\to \infty$ and the target token is obvious, and for $t\to 0$ near initial noise, $\rho \to 0$ and the target is unknown. The SNR $\rho$ is monotonically increasing in $t$, and $\sigma$ and $\alpha$ determine how fast it increases.

\textbf{Per-Particle Energies.}
Since \eqref{eq:j=k} is a Boltzmann distribution, we can define the energy $U_j$ for the vertex $j$ as:
\begin{align}
    U_j:=\beta(t)x_j=\beta(t)(1-\alpha)\boldsymbol{1}[j=k]+\beta(t)\alpha z_j
\end{align}
With SNR $\rho=(1-\alpha)/\alpha\sigma$ and $\beta(t)=(1-\alpha)/\alpha^2\sigma^2$, we can expand the energy as:
\begin{align}
    U_j&=\frac{(1-\alpha)^2}{\alpha^2\sigma^2}\boldsymbol{1}[j=k]+\frac{1-\alpha}{\alpha\sigma^2} z_j\nonumber\\
    &=\underbrace{\left(\frac{1-\alpha}{\alpha\sigma}\right)^2}_{=\rho^2}\boldsymbol{1}[j=k]+\underbrace{\frac{1-\alpha}{\alpha\sigma^2} z_j}_{=:h_j}
\end{align}
Since $z_j \sim \mathcal{N}(0,\sigma^2)$, the variance of $h_j$ is:
\begin{align}
    \text{Var}(h_j)=\left(\frac{1-\alpha}{\alpha\sigma^2}\right)^2\sigma^2=\left(\frac{1-\alpha}{\alpha\sigma}\right)^2=\rho^2
\end{align}
So the energy of the posterior only depends on the SNR $\rho$:
\begin{align}
    U_j&=\rho^2\boldsymbol{1}[j=k]+h_j, \quad h_j\sim \mathcal{N}(0,\rho^2)
\end{align}

\textbf{Connection to Random Energy Model (REM).}
This has a nice connection with the Random Energy Model (REM) \citep{derrida1980random} in statistical physics, which describes a system with $M$ states with i.i.d. random energies and partition function (normalization constant) $Z=\sum_ie^{-\beta_{\text{phys}}\mathcal{E}_i}$. The normalization of our posterior $p(j=k|x)$ is:
\begin{align}
    Z=\sum_{j=1}^Ve^{U_j}=\underbrace{e^{\rho^2+h_k}}_{\text{target signal}}+\underbrace{\sum_{j\neq k}e^{h_j}}_{=:S}
\end{align}
where the sum of all non-target vertices $S$ is the REM partition function with $M=V-1$ states, i.i.d. Gaussian energies $h_j\sim \mathcal{N}(0, \rho^2)$ at physical temperature $\beta_{\text{phys}}=1$. Then, the probability of recovering $j=k$ from $x$ can be interpreted as whether the energy of state $k$ outweighs the competing energies $j\neq k$.

\textbf{Phase Transition Point. }
Now, we want to find the point where the target energy dominates. For $M=V-1$ states with i.i.d. $\mathcal{N}(0,\rho^2)$ energies, we can count how many exceed a level $\bar h$:
\begin{align}
    \mathbb{E}[\#\{j:h_j\geq \bar h\}]=M\text{Pr}(h_1>\bar h)\approx M\left(\frac{\rho}{\bar h\sqrt{2\pi}}\exp\left(\frac{-\bar h^2}{2\rho^2}\right)\right)\approx \exp\left(\log M-\frac{\bar h^2}{2\rho^2}\right)
\end{align}
The maximum non-target energy occurs when the expected count is equal to 1, i.e., $\log M=\frac{\bar h^2}{2\rho^2}\implies h^\star=\rho \sqrt{2\log M}$. Given $M=V-1$, $\log M\approx \log V$ and: 
\begin{align}
    \max_j h_j\approx \rho\sqrt{2\log M}\approx \rho\sqrt{2\log V}
\end{align}
Since the target coordinate has energy $\rho^2$, recovery of the target occurs when $\rho^2$ surpasses the largest non-target energy $\rho\sqrt{2\log V}$. The critical point is when $\rho^2=\rho\sqrt{2\log V}$, which for $\rho> 0$ yields:
\begin{align}
    \rho=\sqrt{2\log V}
\end{align}
We can define:
\begin{align}
    \theta:=\frac{\rho}{\sqrt{2\log V}}
\end{align}
so that the critical point occurs when $\theta=1$. 

\textbf{Argmax Recovery. }
Taking the argmax coordinate $\hat k=\arg\max_jp(k=j|x)=\arg\max_jU_j$ also leads to the critical point. The argmax coordinate is correct (i.e., $\hat k=k$) iff the energy of the target state $k$ dominates every non-target state:
\begin{align}
    U_k> \max_{j \neq k }U_j\iff \rho^2+h_k>\max_{j\neq k}h_j
\end{align}
where $h_k\sim \mathcal{N}(0, \rho^2)$ is negligible given $\rho^2$. Then, we can count the number of non-targets that dominate the target as:
\begin{align}
    \mathbb{E}[\#\{j\neq k : h_j > \rho^2\}]\approx Ve^{-\rho^2/2}=\exp((1-\theta^2) \log V)=V^{1-\theta^2}
\end{align}
which recovers the same critical point:
\begin{enumerate}
    \item[(i)] $\theta>1$: $V^{1-\theta^2}\to 0$ and none of the non-target states dominate the target, and the argmax is correct $\hat k=k$.
    \item[(ii)] $\theta<1$: $V^{1-\theta^2}\to \infty$ and many non-targets dominate the target and $\hat k$ is a yields a random coordinate. 
\end{enumerate}
The critical point occurs when $V^{1-\theta^2}=1\iff \theta=1$ and the target signal equals the maximum non-target $\rho=\sqrt{2\log V}$.

\textbf{Sharpness of Phase Transition. }
The maximum of $V$ Gaussians concentrates at $\max_{j \neq k}h_j =\rho\sqrt{2\log V}+ \mathcal{O}(\rho/\sqrt{\log V})$ with Gumbel fluctuations. Writing $\rho =\sqrt{2\log V}(1+\delta)$ and $\theta=1+\delta$, the exponent becomes $1-\theta^2\approx -2\delta$ and the probability of error is approximately: 
\begin{align}
    \text{Pr}(\text{error})\approx V^{-2\delta}=e^{-2\delta \log V}
\end{align}
which is near 0 just above zero (i.e., $\delta >0$) and huge just below $\delta < 0$. The phase transition occurs within the window when $2\delta \log V=\mathcal{O}(1)$ of width $\delta =\mathcal{O}(1/\log V)$ which goes to zero as $V\to \infty$. Therefore, the probability curve sharpens to a step function as $V\to \infty$: 
\begin{align}
    \text{Pr}(\hat k=k)\to \boldsymbol{1}[\rho> \sqrt{2\log V}]\quad \text{as}\quad V \to \infty
\end{align}

\textbf{Optimal $t_{\text{anchor}}$. }
For $\alpha_t=(1-t)^a$, the SNR is:
\begin{align}
    \rho(t)= \frac{1-(1-t)^a}{(1-t)^a\sigma}
\end{align}
is monotone in $t$. Let $t^\star$ denote the optimal $t_{\text{anchor}}$, we can invert $\rho(t^\star)=\sqrt{2\log V}$ to get the optimal $\alpha^\star$:
\begin{align}
    \sqrt{2\log V}=\frac{1-\alpha^\star}{\alpha^\star\sigma}\implies \alpha^\star=\frac{1}{1+\sigma\sqrt{2\log V}}
\end{align}
and the optimal $t^\star$:
\begin{align}
    \boxed{t^\star=1-\left(1+\sigma\sqrt{2\log V}\right)^{-1/a}}
\end{align}
which concludes the proof.
\end{proof}

\subsubsection{Training Implications}
Since the simplex dimension appears only in the logarithm $\sqrt{2\log V}$, when $V$ becomes large, large increases in $V$ result in small increases in $t^\star$ (Figure \ref{fig:flow-dynamics}D). The exponent on the interpolant schedule $a$ is inversely related to $t^\star$ and moves the curve significantly for $a\in \{1, 2,3\}$ as shown in Figure \ref{fig:flow-dynamics}D and H. Finally, increasing the scale $\sigma$ of the prior Gaussian $\mu_0=\mathcal{N}(\boldsymbol{0}, \sigma^2\boldsymbol{I}_V)^{\otimes L}$ results in a much larger increase in $t^\star$, since the initial point injects more uncertainty, resulting in a later commitment time. We show the impact of these parameters empirically in App \ref{app:hyperparameters}.

\begin{table}[t]
\centering
\begin{tabular}{l cccc}
\toprule
& LM1B & OpenWebText & GSM8K & Sudoku \\
\midrule
Vocabulary $|V|$              & 30{,}522 & 50{,}257 & 49{,}152 & 12 \\
Noise scale $\sigma$          & 1.0 & 1.0 & 1.0 & 1.0 \\
Schedule exponent $a$         & 1 & 1 & 1 & 1 \\
$\sqrt{2\log|V|}$             & 4.545 & 4.653 & 4.648 & 2.229 \\
Optimal $t^*_{\text{anchor}}$ & 0.820 & 0.823 & 0.823 & 0.690 \\
\bottomrule
\end{tabular}
\caption{\textbf{Optimal anchor threshold $t^*_{\text{anchor}}$ per task.} Computed from the phase transition $\rho(t^*)=\sqrt{2\log|V|}$, inverted through the schedule $\alpha_t=(1-t)^a$ to $t^*_{\text{anchor}}=1-\big(1+\sigma\sqrt{2\log|V|}\big)^{-1/a}$. All tasks use $\sigma=1$ and $a=1$.}
\label{tab:optimal-anchor}
\end{table}

\subsection{Occupation Measure of Autonomous Flow on the Simplex}
\label{app:occupation-measure-proof}

\begin{proposition}[Explicit form of $\nu$ and $j$ on simplex]\label{prop:occupancy-current}
Let $x_0\sim \mu_0$, with $\mu_0$ being a prior density on $\mathbb{R}^V$ that decays exponentially towards infinity, and let $x_1\sim \mu_1:=\sum_{j=1}^V p_j\delta_{e_j}$ be independent of $x_0$, where $\{e_1,\dots,e_V\}$ are the vertices of the simplex $\Delta^{V-1}$. Let $I_t=\alpha_tx_0+\beta_tx_1$ with $\alpha_t+\beta_t=1$ and $\alpha_0=1$, $\alpha_1=0$. Fix $x\notin\{e_1,\dots,e_V\}$ and let $\epsilon_j:=x-e_j$. We have
    the per-vertex occupancy $\nu^{(j)}(x)$ and current $J^{(j)}(x)$ are given by:
    \begin{align}
    \label{eqn:per vertex current}
        \nu^{(j)}(x)&=p_j\int_0^1\alpha_t^{-V}\,\mu_0\left(\frac{\epsilon_j+\alpha_te_j}{\alpha_t}\right)dt\\
        J^{(j)}(x)&=p_j\int_0^1\alpha_t^{-V-1}\dot\alpha_t\,\epsilon_j\, \mu_0\left(\frac{\epsilon_j+\alpha_te_j
        }{\alpha_t}\right)dt,
    \end{align}    
    and for $\nu(x)>0$, the autonomous drift $b=j /\nu$ takes the form:
    \begin{align}
        b(x)=\frac{\sum_jJ^{(j)}(x)}{\sum_j\nu^{(j)}(x)}
        \label{eq:simplex-drift}
    \end{align}
\end{proposition}
\begin{proof} For categorical distributions on the simplex, we can write the target distribution as a weighted mixture of point masses $\mu_1=\sum_jp_j\delta_{x_j}$ where $x_j=e_j$ and $e_1, \dots, e_d$ are the $d$ vertices of the simplex $\Delta^{d-1}$, each with probability $p_j$.

Rather than taking a ratio of measures, we define $\mu_t$ and $j_t$ weakly so that the per-vertex decomposition of the autonomous flow can be derived via expectations. For any test function $\varphi\in C^\infty(\mathbb{R}^V)$, we write:
\begin{align}
    \int\varphi(x)\mu_t(x)dx=\mathbb{E}\left[\varphi(I_t)\right], \qquad \int\varphi(x)j_t(x)dx=\mathbb{E}\left[\dot I_t\varphi(I_t)\right]
\end{align}
which is consistent with the definition $b_t=j_t/\mu_t$ since taking the ratio or the right-hand identities yields $\mathbb{E}[\dot I_t|I_t=x]\mu_t(x)$ when tested against $\varphi$. Restricting $x_1=e_j$ gives us the per-vertex definitions of $\mu^{(j)}_t$ and $J^{(j)}_t$ as:
\begin{align}
     \int\varphi(x)\mu^{(j)}_t(x)dx=\mathbb{E}\left[\varphi(I_t)\boldsymbol{1}_{\{x_1=e_j\}}\right], \qquad \int\varphi(x)J^{(j)}_t(x)dx=\mathbb{E}\left[\dot I_t\varphi(I_t)\boldsymbol{1}_{\{x_1=e_j\}}\right]
\end{align}
with the occupancy $\nu^{(j)}(x)=\int_0^1\mu_t^{(j)}(x)dt$ and autonomous current $J^{(j)}(x)=\int_0^1j_t^{(j)}(x)dt$.

\textbf{Step 1: Per-vertex occupancy.} Since $\mu_1$ is a sum of point masses, $\nu$ and $j$ can be decomposed for each vertex $\nu^{(j)}$ and $J^{(j)}$. Therefore, we can consider one atom first. If we fix atom $j$, then $x_1=e_j$ and the only randomness in the interpolant is $x_0\sim \mathcal{N}(0, I_d)$. The interpolant at time $t$ is given by:
\begin{align}
    I_t=\alpha_tx_0+\beta_te_j
\end{align}
Fixing $I_t=x$, we can solve for $x_0(x,t)$ to get:
\begin{align}
    x_0(x,t)=\frac{x-\beta_te_j}{\alpha_t}=\frac{x-(1-\alpha_t)e_j}{\alpha_t}=\frac{\epsilon_j+\alpha_te_j}{\alpha_t},
\end{align}
where we have denoted $\epsilon_j=x-e_j$ and substituted $\beta_t=1-\alpha_t$. Substituting into the Gaussian density of $x_0$ gives a factor of:
\begin{align}
    G_t(x)=\mu_0\left(x_0(x,t)\right)= \mu_0\left(\frac{\epsilon_j+\alpha_te_j}{\alpha_t}\right)
\end{align}
Since the map $x_0(x,t)\mapsto I_t=\alpha_t x_0+\beta_te_j$ is just a scaling of $x_0$ by $\alpha_t$ and a constant shift, its derivative is $\alpha_tI_d$, and the density transforms via the inverse Jacobian $\alpha_t^{-d}$. Since the probability of selecting vertex $j$ is $p_j$, the pointwise density $\mu_t^{(j)}(x)$ at $x$ at time $t$ is:
\begin{align}
    \mu_t^{(j)}(x)=p_j\alpha_t^{-V}G_t(x) =p_j\,\alpha_t^{-V}\mu_0\left(\frac{\epsilon_j+\alpha_t\,e_j}{\alpha_t}\right)
\end{align}
Then, the occupancy measure $\nu^{(j)}(x)$ is just $\mu_t^{(j)}(x)$ integrated over time $t\in [0,1]$ given by:
\begin{align}
    \nu^{(j)}(x)=\int_0^1\mu_t^{(j)}(x)dt=p_j\int_0^1\alpha_t^{-V}\mu_0\left(\frac{\epsilon_j+\alpha_t\,e_j}{\alpha_t}\right)dt\label{eq:occupancy}
\end{align}
Since $\alpha_t\to 0$ as $t\to 1$, we have $\frac{\epsilon_j+\alpha_te_j}{\alpha_t}$ goes to infinity as $t\to  1$. Hence the prior $\mu_0(\frac{\epsilon_j+\alpha_t e_j}{\alpha_t})$ decays exponentially as $t\to 1$ offsetting the divergent $\alpha_t^{-V}$, and the integral in equation \eqref{eq:occupancy} is finite.

\textbf{Step 2: Per-vertex current.} Given $x_0(x,t)=e_j$, the velocity $\dot I_t$ is determined pointwise by $x$ since $x_0$ is a function of $I_t$ for all $I_t=x$. Differentiating $I_t$ and substituting $x_0=(\epsilon_j+\alpha_te_j)/\alpha_t$, we get:
\begin{align}
    \dot I_t=\dot\alpha_tx_0+\dot \beta_te_j=\frac{\dot \alpha_t}{\alpha_t}(\epsilon_j+\alpha_te_j)+\dot \beta_te_j=\frac{\dot \alpha_t}{\alpha_t}\,\epsilon_j
\end{align}
where we have used $\dot\alpha_t+\dot\beta_t=0$. Since $j= b\nu=\int_0^1b_t\mu_tdt$, we:
\begin{align}
    J^{(j)}(x)=p_j\int_0^1\alpha_t^{-V}\underbrace{\left(\frac{\dot \alpha_t}{\alpha_t}\,\epsilon_j\right)}_{\dot I_t}\mu_0\left(\frac{\epsilon_j+\alpha_t\,e_j}{\alpha_t}\right)dt=p_j\int_0^1\alpha_t^{-V-1}\dot\alpha_t\,\epsilon_j\, \mu_0\left(\frac{\epsilon_j+\alpha_te_j
        }{\alpha_t}\right)dt,
    \label{eqn:per-vertex current}
\end{align}
by a similar argument to integral \eqref{eq:occupancy}, the integral \eqref{eqn:per-vertex current} is finite.

\textbf{Step 3: Summing over vertices. }
Since $\mu_1$ is supported on $\{e_1, \dots,e_V\}$ and the events $\{x_1=e_j\}$ partition the sample space such that $\sum_j\boldsymbol{1}_{\{x_1=e_j\}}= 1$ almost surely. For any test function $\varphi$, linearity of expectation gives:
\begin{align}
    \sum_j\int\varphi(x)J^{(j)}_t(x)dx=\sum_j\mathbb{E}\left[\dot I_t\varphi(I_t)\boldsymbol{1}_{\{x_1=e_j\}}\right]=\mathbb{E}\bigg[\dot I_t\varphi(I_t)\sum_j\boldsymbol{1}_{\{x_1=e_j\}}\bigg]=\mathbb{E}\left[\dot I_t\varphi(I_t)\right]=\int \varphi j_tdx
\end{align}
and since this holds for all $\varphi\in C^\infty(\mathbb R^V)$, we can conclude $j_t(x)= \sum_jj_t^{(j)}(x)$-a.e.. Analogously, $\mu_t=\sum_j\mu_t^{(j)}$ by the same argument with $\dot I_t$ replaced by $1$. Integrating over $t\in [0,1]$ and exchanging the sum and integral operations, we have:
\begin{align}
    \nu(x)=\int\sum_j\mu^{(j)}_t(x)dt=\sum_j\nu^{(j)}(x), \quad J(x)=\int\sum_jJ^{(j)}_t(x)dt=\sum_jJ^{(j)}(x)
\end{align}
Since $b=j/\nu$, we conclude:
\begin{align}
    b(x)=\frac{\sum_jJ^{(j)}(x)}{\sum_j\nu^{(j)}(x)}
\end{align}
\end{proof}

\subsection{Autonomous Velocity Vanishes at the Vertices}
\label{app:vertex-vanish-proof}

\begin{theorem}[Autonomous Velocity Vanishes at the Vertices]\label{thm:attractor-vertices}
   Assume the interpolant $\alpha_t\sim(1-t)^a$ as $t\to 1$ with $a\geq1/V$, and assuming the prior $\mu_0(x)$ tends to zero exponentially as $x\to \infty$, then at the vertices of the simplex, for all $i=1,\dots,V$, the autonomous velocity field $b(e_i)=0$ when evaluated at vertices $e_i$.
\end{theorem}
\begin{proof}
For given $e_i$, we evaluate $J^{(j)}(e_i)$ and $\nu^{(j)}(e_i)$ for each $j$. We compute $\nu^{(j)}(e_i)$ and $J^{(j)}(e_i)$ for the case $j=i$ and $j\neq i$ separately.
First when $j=i$, we have $\epsilon_i=e_i-e_i=0$ and 
\begin{align}
    \nu^{(i)}(e_i)=p_i\,\mu_0(e_i)\int_0^1\alpha_t^{-V}\, dt
\end{align}
For common choices of interpolant where $\alpha_t\sim(1-t)^a$ as $t\to 1$ with $a\geq1/V$, the integral $\int_0^1\alpha_t^{-V}\,dt$ diverges, and $\nu^{(i)}(e_i)$ is positive infinity.

When $j\neq i$, we have $\epsilon_j=e_i-e_j=e_{ij}$ where $e_{ij}$ is the simplex edge between vertex $i$ and $j$.
\begin{align}\label{eqn:occupany-jnoti}
    \nu^{(j)}(e_i)=p_j\int_0^1\alpha_t^{-V}\mu_0\left(\frac{e_{ij}}{\alpha_t}+e_j\right) dt
\end{align}
As $t\to1$, we have $\alpha_t\to 0$. Assuming $\mu_0(x)$ tends to 0 exponentially fast as $x\to \infty$, the integral in \eqref{eqn:occupany-jnoti} is finite. The denominator of $b(e_i)$ is $\nu(x)=\sum_j\nu^{(j)}(x)$. The denominator is infinite.

The numerator of $b(e_i)$ is $j(e_i)=\sum J^{(j)}(e_i)$. Consider when the index $j=i$. Since $\epsilon_i=0$, we have $j^{(i)}(e_i)=0$. When the index $j\neq i$, for the same reason as $\nu^{(j)}(e_i)$, we have $J^{(j)}(e_i)$ is a finite number. 

Overall, for all $i=1,\dots, V$, the autonomous velocity field $b(e_i)$ is a ratio between a sum of finite vectors over positive infinity, and hence $b(e_i)=0$.
\end{proof}

\subsection{Convergence of Autonomous Flow on the Simplex}
\restate{\convergence}

\textit{Proof.} \textbf{Notation.} 
Let $e_j \in M_1$ be a vertex on the simplex and the state $x=e_j+r\omega$ and define $\omega:=(x-e_j) /|x-e_j|$ as the unit vector pointing toward $x$ from $e_j$ and $r:=|x-e_j|$ the distance. If the interpolant is $I_t=(1-t)x_0+tx_1$, then the time-dependent velocity at $x$ at time $t$ is related to the denoiser $m(t)$ by: 
\begin{align}
    b_t(x)=\frac{m_t(x)-x}{1-t}, \quad m_t(x)=\mathbb{E}[x_1|I_t=x]\label{lemma-eq:b_t}
\end{align}

To get the time-independent velocity, we take the posterior weight over times $t$ with density proportional to $\sigma^{d-2}\mu_0(e_j+\sigma \omega)$, where $\mu_0(e_j+\sigma \omega)$ is the prior evaluated at the implied initial noisy state $x_0=(x-te_j)/(1-t)=e_j+\sigma \omega$ given $x$. Taking the average over the time-dependent speed gives the arrival speed $\kappa (\omega)$ toward vertex $e_j$ along direction $\omega$:
\begin{align}
    \kappa(\omega):=\frac{\int_0^\infty\sigma \cdot \sigma^{d-2}\mu_0(e_j+\sigma \omega)d\sigma}{\int_0^\infty\sigma^{d-2}\mu_0(e_j+\sigma \omega)d\sigma}, \quad 0<\kappa_-:=\min_{|\omega|=1}\kappa(\omega)\leq \kappa_+<\infty
\end{align}
where $|b(e_j+r\omega)|\to \kappa(\omega)$ is the speed in direction $\omega$ as $r\downarrow 0$, $\kappa_-$ is the minimum speed over all directions, and $\kappa_+$ is the maximum speed over directions. We simplify notation by writing:
\begin{align}
    \kappa_j(\omega) :=\frac{N_j(\omega)}{D _j(\omega)}, \quad N_j(\omega):=\int_0^\infty\sigma^{d-1}\mu_0(e_j+\sigma\omega)d\sigma,\quad D_j(\omega):=\int_0^\infty\sigma^{d-2}\mu_0(e_j+\sigma \omega)d\sigma
\end{align}
Given $x_1=e_j$ and $\alpha_t:=1-t$, we write the marginal distribution at time $t$ as:
\begin{align}
    \mu_t(x)=\sum_{e_l \in M_1}\pi_{e_l}\alpha_t^{-d}\mu_0\left(\frac{x-te_l}{\alpha_t}\right)\label{proofeq:mu_t}
\end{align}

\textbf{Autonomous velocity approaching vertex. }
We first prove a Lemma that gives the speed at which the autonomous flow approaches a vertex. 

\begin{lemma}[Autonomous velocity approaching vertex]\label{lemma:approach-vertex}
    Let $e_j\in M_1$ be a vertex on the simplex and $\omega :=x-e_j$ with $|\omega|=1$ be the unit vector pointing in the direction from $x$ to $e_j$. Then,
    \begin{align}
        \nu(x)=\pi_{e_j}r^{1-d}D_j(\omega)(1+o(1))\label{lemma-eq:nu}
    \end{align}
    and $b(e_j+\sigma\omega)\to -\kappa_j(\omega)\omega$ in the limit $r\downarrow 0$, uniformly in $\omega$. Therefore, $b$ is discontinuous at $e_j$, where its value at the vertex is $b(e_j)=0$ while its limits approaching $e_j$ radially are non-zero and direction-dependent.
\end{lemma}
\textit{Proof.} 
Let $x:=e_j+r \omega$ and $\sigma:=r/\alpha$. Then, $\alpha=r/\sigma$ and $|d\alpha|=r\sigma^{-2}d\sigma$. The argument in $\mu_0$ in $\mu_t(x)$ \eqref{proofeq:mu_t} for $e_j$ takes the form: 
\begin{align}
    \frac{x-te_j}{\alpha_t}=\frac{r\omega+\alpha_t e_j}{\alpha_t} = e_j+\sigma \omega
\end{align}
so we can write $\mu_t(x)$ as:
\begin{align}
    \mu_t(x)=\pi_{e_j}\left(\frac{\sigma}{r}\right)^d\mu_0(e_j+\sigma\omega)(1+O(\alpha))
\end{align}
where $O(\alpha)$ contains the summands for $e_l\neq e_j$ which becomes negligible after integration over $t\in [0,\infty)$. Therefore, the occupation measure $\nu(x)$ is given by:
\begin{align}
    \nu(x)&=\pi_{e_j}\int_0^\infty\left(\frac{\sigma}{r}\right)^d\mu_0(e_j+\sigma \omega)\frac{rd\sigma}{\sigma^2}(1+o(1))\nonumber\\
    &= \pi_{e_j}r^{1-d}\underbrace{\int_0^\infty\sigma^{d-2}\mu_0(e_j+\sigma \omega)d\sigma}_{D_j(\omega)}(1+o(1))\nonumber\\
    &= \pi_{e_j}r^{1-d}D_j(\omega)(1+o(1))
\end{align}
which is the expression in \eqref{lemma-eq:nu}. To get $b(x)$, we can express $b_t(x)$ in \eqref{lemma-eq:b_t} with $m_t(x):=e_j+O(e^{-c/\alpha^2})$ as: 
\begin{align}
    b_t(x)=\frac{e_j-x}{\alpha}+O\left(\alpha^{-1}e^{-c/\alpha^2}\right)=-\frac{r}{\alpha}\omega+ o(1)=-\sigma\omega+o(1)
\end{align}
where the rescaled radius $\sigma$ is the speed. Therefore, the numerator of the autonomous flow $b(x)$ is:
\begin{align}
    \int_0^1\mu_t(x)b_t(x)dt=-\omega\pi_{e_j}r^{1-d}\int_0^\infty\sigma^{d-1}\mu_0(e_j+\sigma \omega)d\sigma(1+o(1))=-\omega\pi_{e_j}r^{1-d}N_j(\omega)(1+o(1))
\end{align}
and dividing by $\nu(x)$ yields $b(x)=-\omega N_j(\omega)/D_j(\omega)+o(1)=-\kappa(\omega)\omega$, which is not dependent on $r$. \hfill $\square$

Now, we prove Proposition \ref{thm:convergence-simplex}. 

By Lemma \ref{lemma:approach-vertex}, set a small constant $r_1>0$ such that:
\begin{align}
    \langle b(e_j+r\omega), \omega\rangle\leq -\frac{\kappa_-}{2}\quad \text{for all }e_j\in M_1,|\omega|=1,0<r\leq r_1,\label{proof-eq:be_j}
\end{align}
so the ball with radius $r_1$ around each vertex $\{B(e_j,r_1)\}_{e_j\in M_1}$ are pairwise disjoint.

\textbf{Local absorption in finite time.} Suppose that $X_s(x_0)\in B(e_j,r_1)$ for some finite time $s< \infty$ and some $e_j \in M_1$. Let $h(s):=|X_s(x_0)-e_j|$ be the distance from the vertex. Over any time interval where $0< h\leq r_1$, $h(t)$ is locally Lipschitz and $\dot h (s)=\langle b(X_s),(X_s-e_j)/h(s)\rangle \leq -\kappa_-/2$ by \eqref{proof-eq:be_j}. Therefore, $h(s)$ decreases at rate at least $\kappa_-/2$ and must vanish at some finite time:
\begin{align}
    \tau\leq s+\frac{2h(s)}{\kappa_-}\leq s+\frac{2r_1}{\kappa_-}< \infty\label{eq:local-absorption}
\end{align}

\textbf{Absorption.} From $\dot h(s)\leq -\kappa_-/2$, we have shown that $X_{s^\star}(x_0)$ arrives at $e_j$ in some finite time. Now, we argue that for all $s>s^\star$, $X_s(x_0)=e_j$ by contradiction. Suppose there exists $X_s(x_0)\neq e_j$ for $s> s^\star$, where $h(t)\leq r_1$ on $[s^\star,s]$. Then, we have shown $\dot h (s)\leq -\kappa_-/2$ a.e. on $\{h>0\}\cap [s^\star,s]$ and since $h(s)$ is absolutely continuous with $h(s^\star)=0$, we can write:
\begin{align}
    h(s)=\int_{s^\star}^s\dot h (u) du\leq -\frac{\kappa_-}{2}(t-t^\star)<0
\end{align}
which is a negative distance $h(s)$, a contradiction. Thus, there exists no $X_s(x_0)$ that leaves $e_j$ and $X_s(x_0)=e_j$ for all $s>s^\star$.

\begin{lemma}[Absorbance from far field]\label{lemma:far-field}
    For $b$ satisfying the divergence condition in Proposition \ref{prop:divergence}, $b(x)=-x(1+o(1))$ for $|x|\to \infty$ so there exists radius $R_0< \infty $ with $\langle b(x),x\rangle <0$ for $|x|\geq R_0$. Let $B(0,R_0)$ be the ball with radius $R_0$ centered at the origin, then $X_{s}\in B(0,R_0)$ for all $s\geq t_0$ if $X_{t_0}\in B(0,R_0)$ (no points escape the ball) and for every $x_0$ there is a finite $t_0$ with $X_{s}(x_0)\in B(0,R_0)$ for all $s\geq t_0$ (all points end up in the ball).
\end{lemma}

\begin{proof}
    Since $b(x)=-x(1+o(1))$ as $|x|\to \infty$, by fixing $R_0$, we have:
    \begin{align}
        \langle b(x),x\rangle\leq -\tfrac{1}{2}|x|^2\quad\text{for }|x|\geq R_0\label{eq:dotbx}
    \end{align}
    Let $h(s):= |X_s|$ be the distance from the origin and differentiate: 
    \begin{align}
        \frac{d}{dt}\tfrac{1}{2}|X_s|^2=\langle\dot X_s,X_s\rangle=\langle b(X_s),X_s\rangle <0\quad \text{whenever}\quad |X_s|\geq R_0\label{eq:decreasing-radius}
    \end{align}
    where $\langle b(x),x\rangle$ is the radial component of the velocity scaled by $|x|$, which is negative exactly when the field $b(x)$ points inward toward the origin.

    \textbf{No escape. }We prove by contradiction. Suppose a point $X_{t_0}\in B(0,R_0)$ escapes to $h(t_1)>R_0$ at some later time $t_1>t_0$. Let $s_1$ be the last time between $[t_0,t_1]$ where the trajectory was at radius $R_0$ or below: $s_1:= \mathrm{sup}\{s\in[t_0,t_1]:h(s)\leq R_0\}$. Then, $h(s_1)=R_0$ by continuity of the trajectory and the distance $h> R_0$ for all $(s_1,t_1]$ by definition of $s_1$ as the last time where $h\leq R_0$. But by \eqref{eq:decreasing-radius}, the radius must be decreasing $\tfrac{d}{ds}\tfrac{1}{2}h^2<0$ whenever $h\geq R_0$ which means $h(t_1)<h(s_1)=R_0$, a contradiction. Therefore, $X_s\in B(0,R_0)$ for all $s\geq t_0$.

    \textbf{Entry in finite time. }While $h(s)\geq R_0$, \eqref{eq:dotbx} gives $\dot h(s)=\langle b(x),x\rangle\leq -\tfrac{1}{2}h(s)$, so $h(s)\leq |x_0|e^{-s/2}$. Therefore, $h(s)<R_0$ for all $s>t_0:= \max (0,2\log (|x_0|/R_0))<\infty $, and by the no escape proof above, the trajectory remains in $B(0,R_0)$ for all times after entry.
\end{proof}

Now, we continue to prove (i) and (ii) of Proposition \ref{thm:convergence-simplex}.

\textbf{Convergence.} Given that all points in $\{B(e_j,r_1)\}$ are absorbed on the vertices in finite time, we now show that any initial point $x_0$ converges to $M_1$. Let $U:=\bigcup_{e_j \in M_1}B(e_j,r_1)$ and $E:=B(0,R_0)\setminus U$, and define the set of initial points $x_0$ whose autonomous flow $X_s(x_0)$ does not land in $U$ for all $t\geq 0$:
\begin{align}
    N:=\{x_0:X_s(x_0)\notin U\text{ for all }s\geq 0\}
\end{align}
By Lemma \ref{lemma:far-field}, all points $x_0$ land in the ball $B(0,R_0)$ at finite time $t_0$ and remain inside the ball for all $s\geq t_0$, so for $x_0\in N$, the flow trajectory lies in the set $E$ for all $s\geq t_0$. By Lemma \ref{lemma:trapped} below, we show that the set $N$ is Lebesgue-null $\mathrm{Leb}(N)=0$ (zero $d$-dimensional volume) such that $\mu_0\ll \mathrm{Leb}$ and $\mu_0(N)=0$. 

In the case where $x_0\notin N$, the autonomous flow enters $B(e_j,r_1)$ for some $e_j\in M_1$ at a finite time $s$ and by \eqref{eq:local-absorption}, it reaches $e_j$ at some time $\tau(x_0) \leq s+2r_1/\kappa_-$ and by the absorption proof, it remains there for all $t\geq \tau$. Setting $x^\star(x_0):=e_j$ gives $X_s(x_0)=x^\star(x_0)$ for all time $t\geq \tau(x_0)$.

\begin{lemma}\label{lemma:trapped}
    $\mathrm{Leb}(N)=0$ where $N:=\{x_0:X_t(x_0)\notin U\text{ for all }t\geq 0\}$
\end{lemma}

\begin{proof}
    Given $\nabla\cdot (\nu b ) =\mu_0-\mu_1$ and $\mu_1(E)=0$, we have:
    \begin{align}
        \nabla \cdot b(x)=\frac{\mu_0(x)-\langle \nabla \nu(x),b(x)\rangle }{\nu(x)}
    \end{align}
    for $x\in E$. Let $J_t:=\mathrm{det}(DX_t(x))$ be the Jacobian determinant of the autonomous flow $x\mapsto X_t(x)$, which acts like a change-of-variables factor $\mathrm{Leb}(X_t(A))=\int_AJ_t(x)dx$. Along any trajectory of the autonomous flow, we have:
    \begin{align}
        \frac{d}{dt}[\nu(X_t(x))J_t]=\mu_0(X_t(x))J_t
    \end{align}
    So for the set $N_n:=\{x_0:X_t(x_0)\notin U\text{ for all }t\geq t_n\}$, the occupation measure of the autonomous flow is $m(t):=\nu(X_t(N_n))=\int_{N_n}\nu(X_t(x))J_tdx$ for $t\geq t_n$. Differentiating, we have:
    \begin{align}
        \dot m(t)=\int_{N_n}\mu_0(X_t(x))J_tdx=\int_{X_t(N_n)}\mu_0(y)dy\geq c_0\mathrm{Leb}(X_t(N_n))\geq \frac{c_0}{\nu_+}m(t)
    \end{align}
    where we apply the change of variables $y:=X_t(x)$, then $\mu_0(E)\geq c_0>0$ on $E$, then $\nu\leq \nu_+ $ gives $\nu(A)\leq \nu_+\mathrm{Leb}(A)$. By Grönwall's inequality, $m(t)\geq m(t_n)\exp(\tfrac{c_0}{\nu_+}(t-t_n))$ for $t\geq t_n$. Since $X_t(N_n)\subseteq E$ for all $t\geq t_n$ by definition, we have $m(t)\leq \nu(E)\leq \nu_+\mathrm{Leb}(E)<\infty$. The exponentially growing quantity cannot be bounded by infinity, forcing $m(t_n)=0$. Since $\nu(E)\geq \nu_->0$, we have $\mathrm{Leb}(X_{t_n}(N_n))=0$ and $X_{t_n}$ is a diffeomorphism on the neighborhood of $E$ where $b\in C^1$, we conclude $\mathrm{Leb}(N_n) =0$. 
\end{proof}

\textbf{Finite hitting time. }
Let $x_0\notin M_1$ be a point in $\mu_0$ and let $t_{\text{entry}}$ be the first time when $X_s(x_0)\in B(x^\star(x_0),r_1)$, which we proved to be finite in the proof of (i). Then, from \eqref{eq:local-absorption}, we have:
\begin{align}
    \tau (x_0)\leq t_{\text{entry}}+\frac{2r_1}{\kappa_-}<\infty
\end{align}
and $X_s(x_0)=x^\star(x_0)$ for all $s\geq \tau (x_0)$ by (ii). Note that by Lemma \ref{lemma:approach-vertex}, the trajectory approaching $x^\star(x_0)=e_j$ arrives at speed $-\kappa_j(\omega)\omega$ and stops at $e_j$ where $b(e_j)=0$ by Theorem \ref{thm:attractor-vertices}. \hfill $\square$

\subsection{Eulerian and Conservation Equations}
\label{app:eulerian-conservation}
\restate{\conservation}

\begin{proof} 
We define $M_1=\{e_1, \dots, e_V\}$ for the set of vertices of $\Delta^{V-1}$ on which the autonomous flow $b$ vanishes by Theorem \ref{thm:attractor-vertices}. This proof follows closely from that of Theorem 2 and Proposition 2 of \citet{lee2026beckmann}, which we restate here in the simplex setting.

\textbf{Eulerian equation off $M_1$.} Since $b$ has no time dependence, for $x\notin M_1$ and $s,u\geq 0$ with $s+u<\tau(x)$, the uniqueness of the autonomous flow $X_s(x)$ gives the semigroup identity:
\begin{align}
    X_{s+u}(x) =X_u(X_s(x))\label{app-eq:semigroup}
\end{align}
By differentiating \eqref{app-eq:semigroup} in $s$ at $s=0$ and applying the chain rule on the RHS, we get:
\begin{align}
    \partial_sX_s(x)=\nabla X_s(x)\cdot\partial_sX_s(x)\vert_{s=0}=b(x) \cdot \nabla X_s(x)
\end{align}
with $X_0(x)=x$ by definition of the flow. For $x\in M_1$, $b(x)=0$ by Theorem \ref{thm:attractor-vertices}, so the constant curve $X_s(x)=x$ is the unique solution, which yields \eqref{eq:eulerian-equation}.

\textbf{Conservation equation of the transport map.} 
For $x\notin M_1$ and $s\in [0,\tau(x))$, by Proposition \ref{thm:convergence-simplex}, the autonomous ODE trajectory passing $x$ converges to a vertex, which is also the same trajectory passing $X_s(x)$. Therefore, the transport map $T$, defined as the long-time limit of the autonomous flow, satisfies:
\begin{align}
    T(X_s(x))=T(x)\quad \forall s\in [0,\tau(x))\label{eq:map-semigroup}
\end{align}
Differentiating \eqref{eq:map-semigroup} in $s$ at $s=0$ yields $b(x)\cdot \nabla T(x) =0$, which is the conservation equation in \eqref{eq:conservation-equation}. Since $b(x)=0$ for $x\in M_1$, $T(x)=\lim_{s\to \tau(x)}X_s(x)=x$ satisfies the boundary condition.

\textbf{Uniqueness of $T$.} 
Let $\tilde T$ be any solution of the conservation equation that is continuous along the flow trajectories. Then $\tfrac{d}{ds}\tilde T(X_s(x))=b(X_s(x))\cdot \nabla\tilde T(X_s(x))=0$ so $s\mapsto \tilde T(X_s(x))$ is constant. Letting $s\to \tau (x)$ and using $\tilde T=\text{id}$ on $M_1$ gives $\tilde T(x)=\lim_{s\to \tau(x)}X_s(x)=T(x)$. 
\end{proof}

\clearpage
\section{Refinement-in-Loop Details and Additional Results}
\label{app:ril-training}
A core component of our training approach for DBTM is concentrating training time on states that are likely to be seen during inference, and therefore, minimizing the training-inference mismatch. To do this, we mimic the inference procedure of BTM, consisting of clean sequence proposals, renoising of low-quality tokens, and refining the proposal during training using the current model, and minimize the loss on these self-generated rollout states. 

\subsection{On-policy Rollouts}
\label{app:on-policy-rollouts}
The standard method of constructing the context-dependent interpolant is to sample the fraction of clean tokens $f\sim \mathcal{U}[0,1]$ and apply an i.i.d. Bernoulli($f$) per position to determine whether it is in $\mathcal{C}$. However, this creates a mismatch between training and inference since the context fed back into the model during refinement is not random. Instead, the context is largely dependent on how 'easy' it is to be generated at the initial application of the transport map $T_\theta$ and the subsequent refinement maps. To align the context seen during inference to the ones used for training, we apply an \textbf{on-policy rollout step} to generate a set of partial context sequences for training by running the same map, renoise, and refine steps as done during inference with a predetermined commitment condition. 

\paragraph{Confidence-Based Selection}
To determine which tokens to commit during the training-time rollouts, we use a confidence-based selection scheme. Given a coupling $(x_0,x_1)$, for a rollout of depth $k$, let $\hat x_{r,k}:=\text{renoise}(T_\theta(\hat x_{r-1,k}))$ denote the state after $r$ refinement rounds without gradient tracking where $\hat x_{0,k}=x_0$ is the initial noise state and $\hat x_{k,k}$ is the final clean sequence after $k$ rounds. After each round $r$, we perform the following steps on the output of the previous round $\hat x_{r-1, k}$:
\begin{enumerate}
    \item [(i)] Apply the parameterized transport map to get a clean proposal sequence $T_\theta(\hat x_{r-1},k)$.
    \item [(ii)] Rank each position $\ell$ by the model's predicted confidence on the clean token value. Given the one-hot vector $x^\ell_1$, compute the confidence at each position $\ell$ as $q^\ell :=\langle T_\theta(\hat x_{r-1},k)^\ell , x^\ell_1\rangle$ to get an ordered set $\mathcal{R}_r$ of the uncommitted positions.
    \item [(iii)] Since there are $k$ total rounds, each round selects at minimum $n_r=\lceil |\mathcal{R}_r|/(k-r+1)\rceil$ highest confidence tokens to commit.
\end{enumerate}
Since the map $T_\theta$ is trained with the partial context interpolant, with more context yielding a sharper prediction task, this selection scheme uses the model's own predictions to distinguish key structural tokens that are easy to predict without context from difficult, highly context-dependent tokens, and trains on the partial context sets it is likely to encounter during inference.

\paragraph{Committed vs. Non-Committed Clean Context}
We perform two methods of refinement-in-loop training: either the clean context at each rollout step is \textit{committed} or \textit{not committed}. For the committed setting, at each iteration $i$ of the rollout with depth $r$, we apply the current trained map $T_\theta(\hat x_{i-1, r})$, commit all tokens with $q^\ell \geq \kappa$ or at least $n_r=\lceil |\mathcal{R}_r|/(k-r+1)\rceil$ (defined as the set $\Delta \mathcal{C}_r$) where $\mathcal{R}_r$ are the uncommitted tokens, and reset the remaining tokens back to their coupled noise state $x_0^\ell$ for $\ell \in [L] \setminus \mathcal{C}_r$. After the rollout step, the loss in \eqref{loss:ril} is computed on only the uncommitted tokens $[L] \setminus \mathcal{C}_r$ at each rollout depth $r$:
\begin{align}
    \mathcal{L}_{\text{ril-c}}(\theta):=\mathbb{E}_{x_0,x_1}\bigg[\sum_{k\in \mathcal{K}}\sum_{r=1}^k\sum_{\ell\in[L]\setminus \mathcal{C}_r}\text{CE}\left(T_\theta(\hat x_{r-1,k})^\ell, x_1^\ell\right)\bigg]\label{loss:ril-c}
\end{align}
where $\texttt{ril-c}$ denotes the loss for the committing rollout setting. We can also define the transport loss \eqref{loss:transport} and anchor loss \eqref{eq:anchor-loss} on these refinement-generated context sets, with $\ell\in [L]\setminus \mathcal{C}_r$ being the uncommitted tokens that the loss is computed over.

For the non-committed setting, at each iteration $i$ of the rollout with depth $r$, we apply the current trained map $T_\theta(\hat x_{i-1, r})$, but instead \textit{hold the highest confidence tokens in set $ \mathcal{C}_r$}, such that they remain in their mapped categorical state $\hat x^\ell_{i,r}\gets T_\theta(\hat x_{i-1, r})^\ell\in \Delta^{V-1}$ for $\ell \in \mathcal{C}_r$ and can be changed in the next round of refinement. The remaining tokens are similarly reset to their coupled noise state $\hat x^\ell_{i,r}\gets x_0^\ell$ for $\ell \in [L]\setminus \mathcal{C}_r$. After rollout, the loss is now computed on \textit{all the tokens in the map output} at each rollout depth $r$:
\begin{align}
    \mathcal{L}_{\text{ril-nc}}(\theta):=\mathbb{E}_{x_0,x_1}\bigg[\sum_{k\in \mathcal{K}}\sum_{r=1}^k\sum_{\ell\in[L]}\text{CE}\left(T_\theta(\hat x_{r-1,k})^\ell, x_1^\ell\right)\bigg]\label{loss:ril-nc}
\end{align}
where the sum is taken over all token positions $\ell\in[L]$ instead of only the uncommitted set and $\texttt{ril-nc}$ denotes the loss for the non-committed rollout setting. 

\paragraph{Matching Deeper Rollouts with Self-Distillation}
Since the noise-data coupling is generally not well-defined for unconditional text generation where there is no prompt and the space of possible clean outputs is large, supervising intermediate rollouts only with clean data samples can prevent the model from generating diverse text since it forces intermediate states to produce a single output even when that output is not aligned with its generative trajectory.

Therefore, to encourage coherent yet diverse supervision, we incorporate a form of self-distillation into the objective. After a warm-up stage where the model starts to generate coherent text for some $k\in \mathcal{K}$, we can use the categorical outputs of the map after the $k$th round $T_\theta(\hat x_{k-1,k})\in (\Delta^{V-1}) ^L$ as the target to which the map applied to earlier refinement rounds $r< k$ is supervised:
\begin{align}
    \mathcal{L}_{\text{soft-ril-c}}(\theta)&:=\mathbb{E}_{x_0,x_1}\bigg[\sum_{k\in \mathcal{K}}\sum_{r=1}^k\sum_{\ell\in[L]\setminus \mathcal{C}_r}\text{CE}\left(T_\theta(\hat x_{r-1,k})^\ell, \text{sg}(T_\theta(\hat x_{k-1,k})^\ell)\right)\bigg]\label{loss:soft-ril-c}\\
    \mathcal{L}_{\text{soft-ril-nc}}(\theta)&:=\mathbb{E}_{x_0,x_1}\bigg[\sum_{k\in \mathcal{K}}\sum_{r=1}^k\sum_{\ell\in[L]}\text{CE}\left(T_\theta(\hat x_{r-1,k})^\ell, \text{sg}(T_\theta(\hat x_{k-1,k})^\ell)\right)\bigg]\label{loss:soft-ril-nc}
\end{align}
where $\texttt{soft}$ denotes that the target can be any categorical distribution on the interior of the simplex rather than only one-hot vectors on the vertices. 

\subsection{Refinement-in-Loop for Reasoning}
\paragraph{Setup}
In the Sudoku setting, the prior state $x_0$ is not just Gaussian noise but a $9\times 9$ grid with a set of filled-in numbers as clues depending on the difficulty (easy: 40, medium: 35, hard: 30), which each correspond exactly to one solution $x_1$. Therefore, given $x_0$, there is no uncertainty about what the answer $x_1$ should be, so $x_1$ should be the only supervision target, and self-distillation is unnecessary, since the self-generated supervision target can only be either equivalent to $x_1$ or incorrect.

Instead, we test the effect of refinement-in-loop only with the defined $(x_0,x_1)$ data couplings and no self-distillation. We set the maximum depth to $k=4$ NFEs (initial map + three refinement rounds), and for each round $r\in \{1,\dots,k\}$ perform refinement by committing the positions with the highest ground-truth confidence, reset the uncommitted positions to their prior state $x_0^\ell$ for $\ell \in [L]\setminus \mathcal{C}_r$ to form a set of self-rollout states $\hat x_{r,k}$ for supervising against the ground truth via the \texttt{ril-c} loss in \eqref{loss:ril-c}. A summary of experiment settings is provided in Table \ref{tab:ril-setup}.

\paragraph{Results}
In Table \ref{tab:reasoning-results}, we show that DBTM with refinement-in-loop (DBTM + ril) achieves superior performance against all baselines and difficulty levels with only 4 NFEs, $32\times$ fewer NFEs than autoregressive, discrete diffusion, and continuous diffusion baselines, and matched NFEs with the closest flow map language model baseline, FMLM+ \citep{agarwal2026posterior}. Notably, on Sudoku Hard, DBTM + ril surpasses FMLM+ by $16.1\%$ at 16 NFEs and $13.4\%$ at 4 NFEs. As shown in Figure \ref{fig:sudoku-solve-acc}, DBTM trained on the clean-context interpolant can converge to high solve accuracy without ril training, but DBTM + ril \textit{significantly} accelerates convergence while achieving higher accuracy. 

\begin{table}[h!]
\centering
\small
\setlength{\tabcolsep}{4pt}
\resizebox{\textwidth}{!}{%
\begin{tabular}{lll}
\toprule
\textbf{Variable} & \textbf{LM1B} & \textbf{Sudoku} \\
\midrule
Refinement Depth $k$ & $4$ & $4$ \\
Supervise Depths (\texttt{all} depths to $k$ or \texttt{one} randomly sampled) & \texttt{all} & \texttt{all} \\
Samples Per Data & 1 & 1 \\
Selection Scheme & confidence on $x_1$ & confidence on $x_1$ \\
Confidence Threshold $\kappa$ & $0.9$ & $0.9$ \\
Noise-Data Coupling & OT coupling & Clues and clean grid + OT coupling \\
Fresh noise & \texttt{off} & \texttt{off} \\
\midrule
Self-distill & \texttt{on} & not necessary given defined clue-solution coupling \\
Commitment Scheme & \texttt{commit/non-commit} & N/A \\
Soft target & \texttt{true} & N/A \\
Temperature & $0.7$ & N/A \\
Self-distill start step & $100{,}000$ & N/A \\
Teacher refresh frequency (steps) & $5000$ & N/A \\
Supervise Depths $r$ & $\{0,1,2,3\}$ & N/A \\
\bottomrule
\end{tabular}
}
\caption{\textbf{Refinement-in-loop training setup for LM1B and Sudoku.} Both tasks roll out to depth $k{=}4$ and supervise every depth, committing by confidence at $\kappa{=}0.9$ without fresh noise. LM1B adds self-distillation after a $100$k-step warmup with soft targets at temperature $0.7$ and a teacher refreshed every $5000$ steps; Sudoku needs none, since its clue-solution coupling already fixes the target.}
\label{tab:ril-setup}
\end{table}

\begin{figure}[h!]
    \centering
    \includegraphics[width=\linewidth]{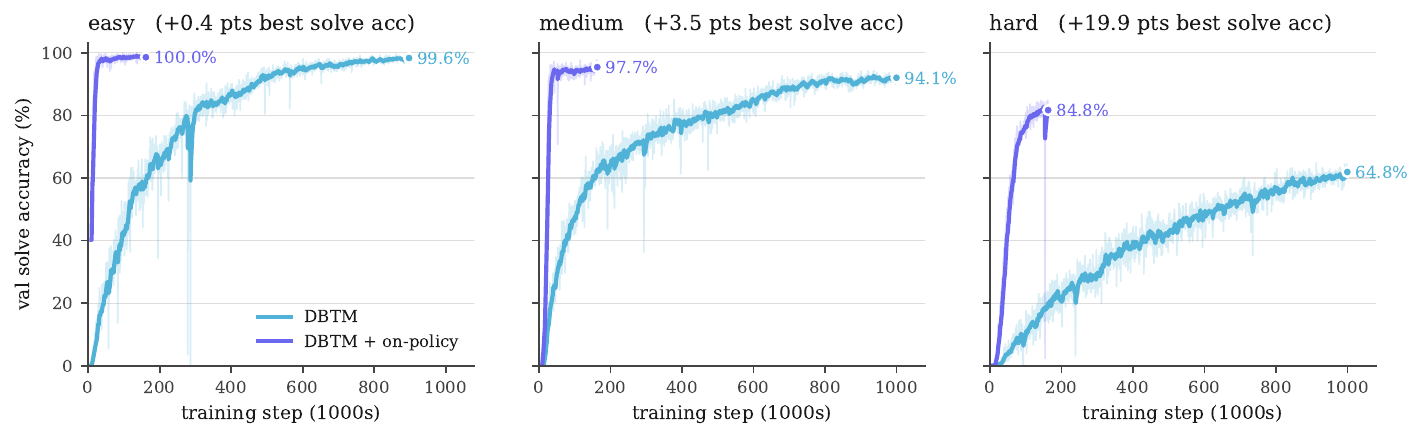}
    \caption{\textbf{Sudoku validation solve accuracy with and without refinement-in-loop training.} Exact-solve accuracy on the held-out puzzles over training steps at 4 NFEs. The clean-context training reaches high accuracy on its own, but ril training converges significantly faster and to a higher ceiling on every difficulty.}
    \label{fig:sudoku-solve-acc}
\end{figure}

\subsection{Refinement-in-Loop for Unconditional Language Modeling}

\paragraph{Setup}
In the language modeling task, there is no definitive coupling between noise and data, since, given a noise sample $x_0$, any coherent text sequence is a valid target sample $x_1$ and there is no \textit{'right'} or \textit{'wrong'} answer like in the reasoning case. Defining a random coupling between noise and data and supervising the on-policy rollouts from the noise to match the data is susceptible to mode collapse, since the noise can be close to generating a completely different, but still coherent, sequence before it is forced to match an arbitrary target. This differs from the base training procedure where the intermediate states are constructed from the interpolant defined as a function of \textit{both} $x_0$ and $x_1$. 

We avoid mode collapse in two complementary ways. First, we use minibatch optimal transport \citep{tong2023improving} to define the coupling $(x_0,x_1)$ for supervised ril training and test both the \textit{committed} and \textit{non-committed} clean context methods described in App \ref{app:on-policy-rollouts} with either the target $x_1$ being committed or held at the categorical state at the positions where the map is most confident. Second, we perform self-distillation after a warmup phase in parallel to the data-supervised training as described in App \ref{app:on-policy-rollouts}, where the target is a \textit{function} of the noise input $x_0$, making the self-generated and refined sequence a well-defined target. We use the LM1B dataset as a testbed for these methods and provide a summary of experiment settings in Table \ref{tab:ril-setup}.

\paragraph{Results}
In Figure \ref{fig:generative-frontier-lm1b} and Table \ref{tab:decode-frontier-lm1b}, we show that DBTM with refinement-in-loop (commit) achieves the widest Pareto frontier of all methods, with low Gen. PPL across all entropy values from $3.90-4.30$. Notably, DBTM + ril achieves the lowest PPL along the frontier curve across all NFE counts and methods. When comparing commit vs. non-commit ril training, we found that the commit mode consistently achieves lower Gen. PPL across NFE counts and selection thresholds with stable entropy around $4.00$ (Table \ref{tab:commit-vs-noncommit}). Figure \ref{fig:onpolicy-gate-lm1b} plots the Gen. PPL and sample entropy from small validation batches over training steps. Visually, Gen. PPL rises early in training as entropy approaches the data entropy, and when on-policy training turns on after the $100\mathrm{K}$- step warmup, Gen. PPL starts to drop while entropy remains stable. In Figure \ref{fig:onpolicy-gate-lm1b}A-C, the commit method (purple curve) shows a steeper drop in Gen. PPL than the non-commit method (turquoise curve).

\begin{figure}
    \centering
    \includegraphics[width=\linewidth]{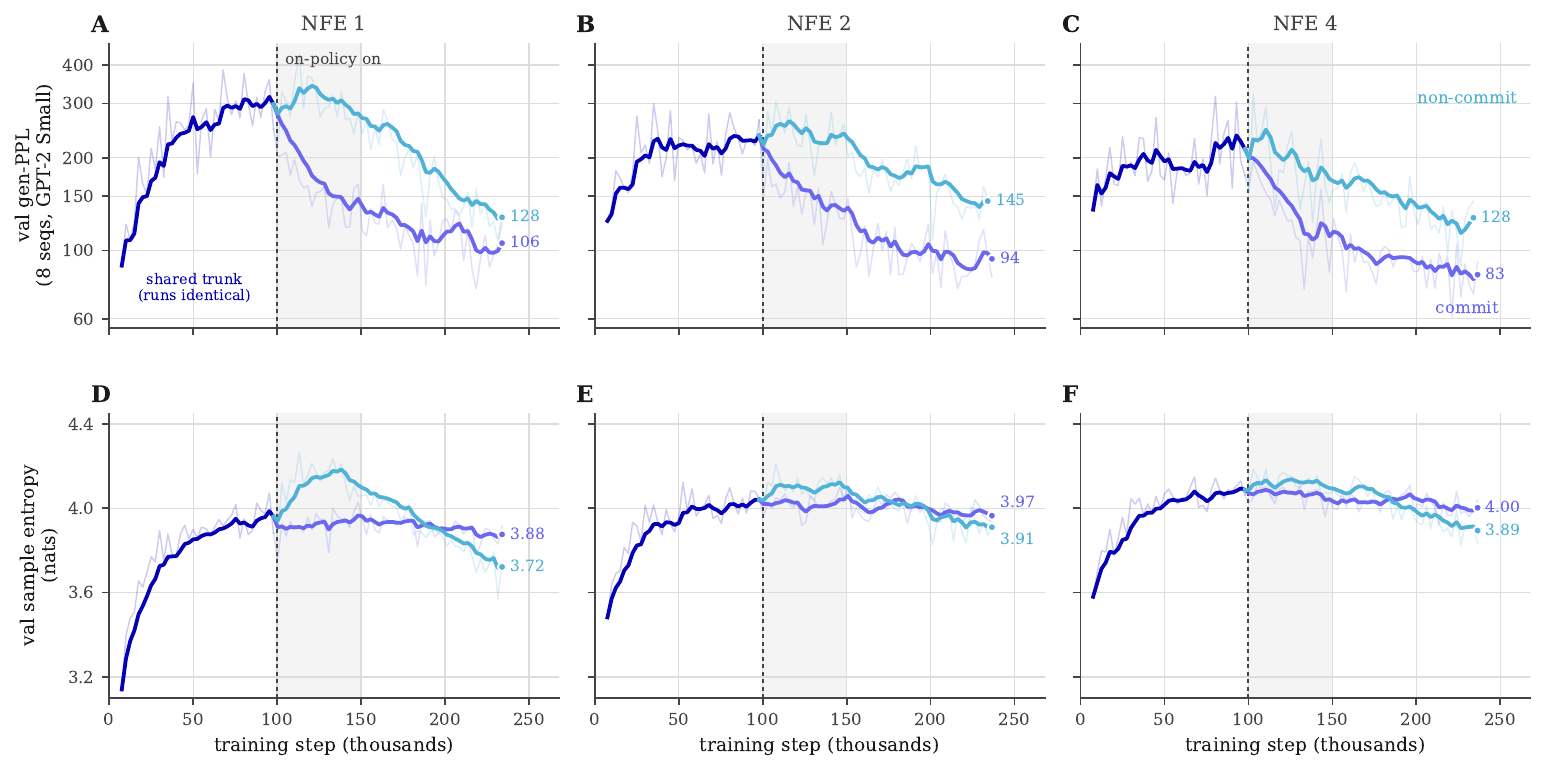}
    \caption{\textbf{Effect of refinement-in-loop training on LM1B.} Gen-PPL (GPT-2-small for efficiency) and sample entropy on a single validation batch of $4$ sequences over training (not directly comparable to the table values). Gen-PPL rises early as entropy approaches the data entropy, and when on-policy training turns on after the $100$K-step warmup, Gen-PPL drops while entropy holds. Panels A--C compare committed (purple) and non-committed (turquoise) rollouts; the committed variant descends faster at every NFE.}
    \label{fig:onpolicy-gate-lm1b}
\end{figure}

\begin{table}[h!]
\centering
\small
\setlength{\tabcolsep}{4pt}
\resizebox{\textwidth}{!}{%
\begin{tabular}{l ccccc ccccc}
\toprule
& \multicolumn{5}{c}{\textbf{commit}} & \multicolumn{5}{c}{\textbf{non-commit}} \\
& \multicolumn{5}{c}{Commit Threshold $\kappa$} & \multicolumn{5}{c}{Commit Threshold $\kappa$} \\
\cmidrule(lr){2-6}\cmidrule(lr){7-11}
NFE & 0.5 & 0.7 & 0.9 & 0.99 & 0.999 & 0.5 & 0.7 & 0.9 & 0.99 & 0.999 \\
\midrule
1   & 86.0 (3.92) & 86.0 (3.92) & 86.0 (3.92) & 86.0 (3.92) & 86.0 (3.92) & 136.2 (3.88) & 136.2 (3.88) & 136.2 (3.88) & 136.2 (3.88) & 136.2 (3.88) \\
2   & 81.0 (4.05) & 81.0 (4.05) & 81.0 (4.05) & 81.0 (4.05) & 81.0 (4.05) & 129.0 (3.92) & \textbf{127.7} (3.98) & 139.4 (4.04) & 142.4 (4.05) & 142.4 (4.05) \\
4   & \textbf{70.1} (4.03) & 72.9 (4.04) & 73.6 (4.05) & 73.6 (4.05) & 73.6 (4.05) & 124.6 (3.91) & 114.5 (3.96) & \textbf{110.9} (4.02) & 118.0 (4.08) & 124.1 (4.10) \\
8   & 63.8 (4.02) & 61.5 (4.01) & \textbf{61.1} (4.00) & 62.2 (4.01) & 62.3 (4.01) & 123.6 (3.91) & 108.4 (3.94) & 96.6 (3.99) & \textbf{95.1} (4.03) & 97.0 (4.06) \\
16  & 61.0 (4.02) & 57.8 (4.00) & 54.4 (3.97) & \textbf{53.7} (3.96) & 54.3 (3.96) & 121.1 (3.91) & 104.6 (3.94) & 87.9 (3.95) & 79.8 (3.98) & \textbf{77.7} (3.99) \\
\bottomrule
\end{tabular}
}
\caption{\textbf{Commit vs.\ non-commit refinement-in-loop training on LM1B, over NFE $\times$ commit threshold $\kappa$.} The lowest Gen-PPL in each row of each arm is \textbf{bolded}. The on-policy rollout that produces the distillation target either hard-commits its confident tokens into the supervised state (\emph{commit}) or re-maps the soft state without committing (\emph{non-commit}), and both arms are trained for 200K steps with all other parameters held constant. The commit setting is used as default in ril training.}
\label{tab:commit-vs-noncommit}
\end{table}

\clearpage
\section{Hyperparameter Discussion}
\label{app:hyperparameters}
Below, we discuss each ablated hyperparameter, what it controls, and the behavior that the sweeps in Table~\ref{tab:ablation-lm1b} show. App \ref{app:training-hyperparams} describes the training-time knobs of the map, and App \ref{app:inference-hyperparams} describes the inference-time knobs that can be changed on a fixed checkpoint.

\subsection{Training Hyperparameters}
\label{app:training-hyperparams}
\paragraph{Adaptive Loss}
To stabilize training, we scale the losses with a gradient-detached weight that downweights large mismatches in the predicted and true target. Denoting the transport map prediction $q(x)=T(x)\in \Delta^{V-1}$ and $p(x)=\texttt{sg}(T(x)+\dot I_t(x)\cdot \nabla T(x))$, we write the mismatch as $\Delta(x)=q(x)-p(x)$ and the weight as:
\begin{align}
    w(x):=\texttt{sg}\left[\left(\|\Delta(x)\|^2_2+c\right)^{-r}\right]
\end{align}
where $c$ and $r$ are tunable hyperparameters.

\paragraph{Minimum Anchor Time}
The threshold $t_{\text{anchor}}$ gates the anchor loss, which supervises $T_\theta(I_t)$ directly toward $x_1$ only for $t>t_{\text{anchor}}$ and sets how much of the interpolant is treated as already committed. We prove in App. \ref{app:optimal-t} that the optimal anchor time aligns closely with the time of the phase transition and perform an ablation over $t_{\text{anchor}}$ on LM1B in Table \ref{tab:anchor-ablation}.

\paragraph{Noise Scale}
The prior scale $\sigma$ in $\mu_0=\mathcal{N}(0,\sigma^2\boldsymbol{I}_V)^{\otimes L}$ sets the radius of the Gaussian hypersphere around the simplex, and enters the commitment time through the signal-to-noise ratio $\rho(t)=(1-\alpha_t)/\alpha_t\sigma$ (See App \ref{app:optimal-t} and Figure \ref{fig:flow-dynamics}). Raising it to $\sigma{=}1.2$ improves Gen-PPL at every NFE budget (178.1 vs.\ 247.1 at NFE 1, 61.9 vs.\ 70.1 at NFE 8) but drops entropy below 4 (Table \ref{tab:ablation-lm1b}).

\paragraph{Schedule Exponent}
The exponent $a$ in $\alpha_t=(1-t)^a$ controls how quickly the interpolant resolves onto a vertex, and for larger $a$, this reduces the commitment time $t^\star$ (Figure \ref{fig:flow-dynamics}). Setting $a{=}1$ reduces to the linear interpolant $I_t=(1-t)x_0+tx_1$. 

\paragraph{Adaptive Loss Constant}
The constant $c$ in the adaptive weight $w(x)=\texttt{sg}[(\|\Delta(x)\|_2^2+c)^{-r}]$ determines how strongly a large prediction mismatch is downweighted, with small $c$ giving a more aggressive reweighting. It is the least sensitive knob we swept, where moving $c$ from $1$ to $10^{-5}$ changes Gen-PPL negligibly and leaves entropy unchanged, demonstrating that our method is robust to minor tuning in the objective. 

\subsection{Inference Hyperparameters}
\label{app:inference-hyperparams}
\paragraph{Rounds of Refinement}
Each round is one map application and one NFE with no additional NFE from computing the confidence or quality score, so the number of rounds $K$ is equivalent to the number of jumps taken with a flow map or integration steps taken with a flow matching model. Quality improves consistently with increasing $k$ until the sampler saturates when every position is committed to a token.

\paragraph{Commit Scorer}
The scorer ranking positions for commitment is either the map's own softmax confidence (rounding error of argmax token) or the learned per-token quality $q_\phi$. All positions with score larger than the confidence threshold $\kappa$ are committed. The sampler commits a minimum of $n_r=\lceil|\mathcal{R}_r|/(k-r+1)\rceil$ highest-scoring uncommitted positions at round regardless of whether their score exceeds $\kappa$, which guarantees the sampler finishes within its NFE budget. 

\paragraph{Confidence Threshold}
The threshold $\kappa$ decides which positions are committed rather than renoised, with position $\ell$ frozen when $q^\ell\geq\kappa$. Its effect depends on whether the per-round minimum commit count $n_r=\lceil |\mathcal{R}_r|/(k-r+1)\rceil$ is met.

\paragraph{Commit Temperature}
Temperature used to scale the proposed categorical probabilities before taking the softmax $T_\theta(x_0):=\text{softmax}(f_\theta(x_0) /\texttt{temp})$. When $\texttt{temp}=0.0$, the committed token is the argmax. 

\paragraph{Prior Scale}
The standard deviation $\sigma$ of the Gaussian prior distribution from which the initial noise latent at each position $\ell$ are sampled $x^\ell_0\sim \mathcal{N}(\boldsymbol{0}, \sigma^2\boldsymbol{I}_V)$.

\paragraph{Repetition Penalty}
For unconditional language modeling, the commit rule reads each position's marginal independently within a round can cause repetition between tokens committed within a single round. As is standard for autoregressive decoding \citep{keskar2019ctrl}, we apply a repetition penalty $\lambda\geq0$ that divides $T_\theta(x)^\ell_j$ by $(1+n_j(\mathcal{C}))^{\lambda}$ on uncommitted tokens before decoding and scoring, where $n_j(\mathcal{C})$ counts committed positions holding token $j$, which has no effect for 1 NFE sampling. We find that after refinement-in-loop training, the penalty is not required for coherent sampling, demonstrating that training on contexts seen during inference improves generation quality.

\section{Ablations and Additional Results}
\paragraph{Overview}
We provide the results for hyperparameter ablations and additional experiments in Tables \ref{tab:ablation-lm1b} to \ref{tab:decode-frontier-owt}.
\begin{itemize}
    \item Table \ref{tab:language-results-nfe} shows the Gen. PPL and entropy for the full sweep of NFEs up to 1024 on LM1B and OWT.
    \item Table \ref{tab:gsm8k-nfe-ril} shows the accuracy for the full sweep of NFEs for TinyGSM/GSM8K against the state-of-the-art few-step baseline FMLM+.
    \item Table \ref{tab:autonomous-field-horizon} compares integrating the autonomous field over increasing time horizons via the endpoint parameterization from Remark \ref{remark:autonomous-flow} with applying the one-step map on LM1B.
    \item Table \ref{tab:ablation-lm1b} shows hyperparameter ablations on LM1B which are used to inform the OWT experiments.
    \item Table \ref{tab:sudoku-threshold-sweep} shows a sweep of NFE and commit threshold $\kappa$ on each difficulty of Sudoku.
    \item Table \ref{tab:anchor-ablation} shows the results from sweeping the minimum anchor time $t_{\text{anchor}}\in \{0.60, 0.75, 0.80, 0.85\}$ during DBTM + ril training for LM1B. 
    \item Table \ref{tab:decode-frontier-lm1b} shows the numerical results corresponding to the frontier curves in Figure \ref{fig:generative-frontier-lm1b} for LM1B, with the Gen-PPL at matched entropy on the Pareto frontiers generated from sweeping inference parameters for each method.
    \item Table \ref{tab:decode-frontier-owt} shows the numerical results corresponding to the frontier curves in Figure \ref{fig:discrete-btm} for OWT, with the Gen-PPL at matched entropy on the Pareto frontiers generated from sweeping inference parameters for each method.
\end{itemize}

\paragraph{Integrating the Autonomous Flow}
We empirically determine the hitting time on LM1B by integrating the time-independent velocity implied by the trained map over a range of horizons $H$ (Table \ref{tab:autonomous-field-horizon}). At short horizons ($H=1$), the integrated trajectory does not reach the correct basin of attraction, but by the time the horizon increases to $H=4$, the integrated field has settled onto a fixed point $x\in M_1$ where $b(x)=0$, and sample quality no longer improves as the horizon grows to $H=8$. However, we empirically observe that applying the one-step map $T_\theta(x)$ consistently outperforms integrating the implied velocity field, even with fine discretization steps ($N_\text{steps}=64$). This supports scaling refinements of the one-step map rather than scaling integration steps, and is the property that motivates the refinement scheme of Section \ref{sec:self-correction}.

\begin{table}[h!]
\centering
\small
\setlength{\tabcolsep}{4pt}
\resizebox{\textwidth}{!}{%
\begin{tabular}{ll cccccccc}
\toprule
& & \multicolumn{8}{c}{\textbf{NFE}} \\
\cmidrule(lr){3-10}
\textbf{Dataset} & \textbf{Method} & 1 & 2 & 4 & 8 & 16 & 128 & 256 & 1024 \\
\midrule
 & FMLM                   & 119.34 (4.16) & 110.19 (4.21) & 98.76 (4.21) & 86.28 (4.20) & 78.06 (4.21) & 57.73 (4.20) & 51.03 (4.18) & 41.84 (4.15) \\
 & DFM (PSD)              & 94.08 (4.06) & 87.42 (4.08) & 78.89 (4.10) & 69.90 (4.10) & 64.90 (4.11) & 58.38 (4.10) & 56.59 (4.10) & 56.31 (4.11) \\
 & DFM (ESD)              & 68.11 (3.79) & 77.60 (4.11) & 71.53 (4.13) & 65.61 (4.13) & 59.92 (4.13) & 55.70 (4.11) & 56.88 (4.12) & 58.27 (4.12) \\
 \rowcolor{mybg}
 & \textbf{DBTM} & 201.20 (4.03) & 167.05 (4.19) & 119.99 (4.20) & 90.19 (4.16) & 71.42 (4.11) & 58.81 (4.03) & 58.81 (4.03) & 58.81 (4.03) \\
 \rowcolor{mybg}
 \multirow{-5}{*}{\textbf{LM1B}} & \textbf{DBTM + ril} & 85.50 (3.92) & 81.00 (4.05) & 73.54 (4.04) & 60.86 (4.00) & 54.48 (3.97) & 51.26 (3.95) & 51.26 (3.95) & 51.26 (3.95) \\
\midrule
 & FMLM                   & 168.30 (5.17) & 133.29 (5.25) & 111.31 (5.26) & 88.27 (5.26) & 68.06 (5.24) & 28.73 (4.88) & 20.26 (4.70) & 9.10 (4.13) \\
 & DFM (PSD)              & 180.29 (4.91) & 152.83 (5.03) & 122.32 (5.10) & 98.54 (5.11) & 82.51 (5.09) & 56.00 (5.00) & 51.81 (4.97) & 47.82 (4.97) \\
 & DFM (ESD)              & 5.33 (0.26) & 108.91 (5.15) & 77.08 (5.27) & 62.98 (5.23) & 55.03 (5.18) & 41.90 (5.04) & 39.08 (5.00) & 36.48 (4.95) \\
& FMLM+ & 378.53 (5.34) & 114.46 (5.14) & 41.16 (4.77) & 25.15 (4.53) & 21.58 (4.42) & 20.79 (4.40) & 20.72 (4.39) & 20.60 (4.37) \\
 \rowcolor{mybg}
 & \textbf{DBTM} (attn) & 249.8 (4.77) & 156.3 (5.06) & 123.1 (5.18) & 107.9 (5.20) & 98.6 (5.20) & 86.7 (5.20) & 83.2 (5.20) & 80.4 (5.22) \\
 \rowcolor{mybg}
& \textbf{DBTM + ril} (attn) & 133.0 (4.94) & 109.6 (5.24) & 115.5 (5.42) & 120.0 (5.48) & 121.0 (5.50) & 120.5 (5.54) & 117.6 (5.55) & 117.6 (5.55) \\
 \rowcolor{mybg}
 & \textbf{DBTM} (lin) & 188.2 (4.97) & 76.2 (5.14) & 52.2 (5.26) & 42.2 (5.28) & 37.0 (5.27) & 31.6 (5.25) & 30.7 (5.25) & 30.1 (5.25) \\
  \rowcolor{mybg}
 \multirow{-7}{*}{\textbf{OWT}}  & \textbf{DBTM + ril} (lin) & 65.2 (4.95) & 62.2 (5.38) & 57.5 (5.52) & 52.3 (5.56) & 48.8 (5.57) & 42.2 (5.57) & 41.0 (5.56) & 40.9 (5.56) \\
\bottomrule
\end{tabular}
}
\caption{\textbf{Full NFE sweep for LM1B and OWT.} Generative perplexity (GPT-2 Large; $\downarrow$) and entropy ($\uparrow$) across NFEs for LM1B and OWT compared to few-step flow map baselines FMLM, discrete flow maps (DFM), and FMLM with posterior refinement (FMLM+). (attn) denotes the same DiT architecture as FMLM, and (lin) denotes the linear attention architecture with matched parameters described in App. \ref{app:language-exp-details}. FMLM and FMLM+ values are evaluated with published checkpoints, and DFM values are taken from the paper. All values are computed from 1024 sequences. For DBTM and FMLM+, NFE is the number of refinement rounds. DBTM LM1B checkpoints are evaluated at 200K steps, and DBTM OWT checkpoints are evaluated at 100K steps.}
\label{tab:language-results-nfe}
\end{table}

\begin{table}[h!]
\centering
\small
\setlength{\tabcolsep}{5pt}
\begin{tabular}{l ccccccc}
\toprule
& \multicolumn{7}{c}{\textbf{GSM8K} (NFE)} \\
\cmidrule(lr){2-8}
Method & 1 & 2 & 4 & 8 & 16 & 32 & 64 \\
\midrule
FMLM+ & 0 & 0.6 & \textbf{5.7} & 9.7 & \textbf{14.5} & 16.6 & 15.0 \\
\rowcolor{mybg} \textbf{DBTM} & \textbf{0.8} & \textbf{1.1} & 5.7 & \textbf{10.1} & 14.3 & \textbf{16.8} & \textbf{16.2} \\
\bottomrule
\end{tabular}
\caption{\textbf{GSM8K accuracy across NFE for DBTM vs.\ FMLM+.} Final-answer accuracy (\%, $\uparrow$) on the full 1319-problem GSM8K test set, one greedy sample per question ($k{=}1$) from a single checkpoint per row, where NFE is the number of refinement rounds and both samplers spend exactly one network evaluation per round (\textbf{bold} $=$ better of the two in the column). FMLM+ \citep{agarwal2026posterior} is the published Posterior Refinement checkpoint with refinement threshold $0.99$ and DBTM is the refinement-in-loop arm at $q{=}0.8$. Both methods are evaluated at 250K steps, commits ranked by softmax confidence, and fresh noise redrawn each round.}
\label{tab:gsm8k-nfe-ril}
\end{table}

\begin{table}[t]
\centering
\begin{tabular}{l ccc c}
\toprule
& \multicolumn{3}{c}{\textbf{Autonomous field} (horizon $H$)} & \textbf{One-step map} \\
\cmidrule(lr){2-4}\cmidrule(lr){5-5}
& $H{=}1$ & $H{=}4$ & $H{=}8$ & NFE 1 \\
\midrule
Gen-PPL~$\downarrow$ & 3571.8 & 276.1 & 293.2 & 247.1 \\
Entropy              & 4.822  & 4.378          & 4.386 & 4.115 \\
\bottomrule
\end{tabular}
\caption{\textbf{Autonomous flow horizon sweep on LM1B.} Integrating the learned autonomous field $b$ with 64 Euler steps up to horizon $H$, compared against a single application of the one-step map $T_\theta$ (DBTM without ril). At $H{=}1$, the flow has not converged and generations remain incoherent. At $H{=}4$, quality saturates, indicating convergence to the manifold. The one-step map outperforms the fully integrated field in one NFE, demonstrating the advantage of directly learning $T$.}\label{tab:autonomous-field-horizon}
\end{table}

\begin{table}[h!]
\centering
\small
\begin{tabular}{l cc cc}
\toprule
& \multicolumn{2}{c}{\textbf{One-step} (NFE 1)} & \multicolumn{2}{c}{\textbf{Refinement} (NFE 8)} \\
\cmidrule(lr){2-3}\cmidrule(lr){4-5}
Configuration & Gen-PPL~$\downarrow$ & Ent. & Gen-PPL~$\downarrow$ & Ent. \\
\midrule
\textit{Noise scale $\sigma$} & & & & \\
\rowcolor{mybg} \textbf{1.0} & 247.12 & 4.11 & 70.08 & 4.02 \\
1.2 & 178.15 & 3.94 & 61.90 & 3.98 \\
\midrule
\textit{Schedule exponent $a$} & & & & \\
\rowcolor{mybg} \textbf{1} & 247.12 & 4.11 & 70.08 & 4.02 \\
2 & 221.32 & 4.03 & 77.38 & 3.98 \\
\midrule
\textit{Adaptive-loss constant $c$} & & & & \\
\rowcolor{mybg} \textbf{1} & 247.12 & 4.11 & 70.08 & 4.02 \\
$10^{-5}$ & 260.31 & 4.12 & 69.00 & 4.01 \\
\midrule
\textit{Commit scorer} & & & & \\
\rowcolor{mybg} \textbf{Softmax confidence} & 247.12 & 4.11 & 70.08 & 4.02 \\
Learned $q_\phi$ (shared head) & 247.12 & 4.11 & 65.31 & 3.75 \\
\midrule
\textit{Refinement rule} & & & & \\
Iterate the map $T_\theta^{\circ k}$ & 178.15 & 3.94 & 501.24 & 4.23 \\
\rowcolor{mybg} \textbf{Renoise + remap} & 178.15 & 3.94 & 61.90 & 3.98 \\
\bottomrule
\end{tabular}
\caption{\textbf{Ablation results on LM1B.} Each block varies one knob off the default parameters (clean-context per-token $\alpha$, learned weighting, adaptive transport $r{=}0.5$, CE endpoint, OT coupling, linear interpolation, uniform time, batch 128) at $t_{\text{anchor}}{=}0.75$, $\sigma{=}1.0$, $a{=}1$, $c{=}1$; \colorbox{mybg}{shaded} $=$ the setting the paper uses.}
\label{tab:ablation-lm1b}
\end{table}

\begin{table}[h!]
\centering
\small
\setlength{\tabcolsep}{4pt}
\resizebox{\linewidth}{!}{
\begin{tabular}{l cccc>{\columncolor{mybg}}c c cccc>{\columncolor{mybg}}c c cccc>{\columncolor{mybg}}c}
\toprule
& \multicolumn{5}{c}{\textbf{Easy}} & & \multicolumn{5}{c}{\textbf{Medium}} & & \multicolumn{5}{c}{\textbf{Hard}} \\
\cmidrule(lr){2-6}\cmidrule(lr){8-12}\cmidrule(lr){14-18}
NFE (Rounds) & 0.5 & 0.7 & 0.9 & 0.99 & 0.999 & & 0.5 & 0.7 & 0.9 & 0.99 & 0.999 & & 0.5 & 0.7 & 0.9 & 0.99 & 0.999 \\
\midrule
1 & 63.65 & 63.65 & 63.65 & 63.65 & 63.65 & & 39.50 & 39.50 & 39.50 & 39.50 & 39.50 & & 1.60 & 1.60 & 1.60 & 1.60 & 1.60 \\
2 & \textbf{94.40} & 94.20 & 92.20 & 88.65 & 85.10 & & \textbf{85.75} & 83.15 & 78.70 & 70.30 & 63.45 & & \textbf{17.70} & 13.35 & 8.80 & 6.05 & 5.55 \\
4 & 97.35 & 98.05 & 98.85 & \textbf{99.70} & 99.50 & & 92.25 & 94.60 & 96.30 & \textbf{97.65} & 97.30 & & 80.15 & 84.65 & \textbf{86.05} & 85.70 & 84.60 \\
8 & 97.20 & 97.90 & 98.75 & \textbf{99.80} & 99.75 & & 92.55 & 94.90 & 96.95 & 98.85 & \textbf{99.15} & & 81.60 & 88.70 & 93.10 & 94.60 & \textbf{95.00} \\
16 & 97.40 & 98.20 & 99.05 & 99.75 & \textbf{99.90} & & 93.45 & 95.35 & 97.50 & 98.85 & \textbf{99.35} & & 83.35 & 89.60 & 94.85 & 96.85 & \textbf{97.45} \\
\bottomrule
\end{tabular}
}
\caption{\textbf{Commit threshold sweep for DBTM + ril on Sudoku.} Exact-solve accuracy (\%) on the 2000-puzzle held-out set over NFE (refinement rounds) $\times$ commit threshold $\kappa\in\{0.5,0.7,0.9,0.99,0.999\}$ for the checkpoints reported in Table \ref{tab:reasoning-results}. \colorbox{mybg}{Shaded} $=$ the fixed $\kappa=0.999$ used there. The best in each row within each difficulty is \textbf{bolded}; NFE 1 is a one-shot decode and does not depend on $\kappa$. Lower thresholds commit more positions per round and plateau early, so at NFE 8-16 accuracy increases with $\kappa$ on every difficulty, and $\kappa=0.999$ is best or within one puzzle of best. At NFE 2 the ordering reverses, since a strict threshold leaves most positions to be forced by the commit floor in the final round, and at NFE 4 the optimum sits slightly below the strictest setting.}
\label{tab:sudoku-threshold-sweep}
\end{table}

\begin{table}[h!]
\centering
\resizebox{\linewidth}{!}{
\begin{tabular}{c c>{\columncolor{mybg}}c cc c c>{\columncolor{mybg}}c cc}
\toprule
$t_{\text{anchor}}$ & \multicolumn{4}{c}{0.60} & & \multicolumn{4}{c}{\textbf{0.75}$^\star$} \\
\cmidrule(lr){2-5}\cmidrule(lr){7-10}
NFE (Rounds) & 0.7 & 0.9 & 0.99 & 0.999 & & 0.7 & 0.9 & 0.99 & 0.999 \\
\midrule
1  & 47.4 (3.65) & 47.4 (3.65) & 47.4 (3.65) & 47.4 (3.65) & & 92.9 (3.93) & 92.9 (3.93) & 92.9 (3.93) & 92.9 (3.93) \\
2  & 41.3 (3.86) & 41.3 (3.86) & 41.3 (3.86) & 41.3 (3.86) & & 88.0 (4.06) & 88.0 (4.06) & 88.0 (4.06) & 88.0 (4.06) \\
4  & 35.8 (3.88) & 36.1 (3.89) & 36.1 (3.89) & 36.1 (3.89) & & 79.5 (4.05) & 81.0 (4.06) & 81.0 (4.06) & 81.0 (4.06) \\
8  & 29.1 (3.87) & 29.7 (3.86) & 29.9 (3.86) & 29.9 (3.86) & & 66.5 (4.02) & 66.5 (4.01) & 68.3 (4.01) & 68.3 (4.01) \\
16 & 27.1 (3.86) & 26.0 (3.84) & 25.9 (3.84) & 25.9 (3.84) & & 62.4 (4.00) & 58.4 (3.97) & 59.1 (3.96) & 59.4 (3.96) \\
\midrule
$t_{\text{anchor}}$ & \multicolumn{4}{c}{0.80} & & \multicolumn{4}{c}{0.85} \\
\cmidrule(lr){2-5}\cmidrule(lr){7-10}
NFE (Rounds) & 0.7 & 0.9 & 0.99 & 0.999 & & 0.7 & 0.9 & 0.99 & 0.999 \\
\midrule
1  & 110.8 (3.97) & 110.8 (3.97) & 110.8 (3.97) & 110.8 (3.97) & & 174.9 (4.12) & 174.9 (4.12) & 174.9 (4.12) & 174.9 (4.12) \\
2  & 109.2 (4.09) & 109.2 (4.09) & 109.2 (4.09) & 109.2 (4.09) & & 170.6 (4.23) & 170.6 (4.23) & 170.6 (4.23) & 170.6 (4.23) \\
4  & 99.0 (4.09)  & 100.6 (4.10) & 100.5 (4.09) & 100.5 (4.09) & & 155.5 (4.23) & 156.2 (4.23) & 156.2 (4.23) & 156.2 (4.23) \\
8  & 82.9 (4.06)  & 84.1 (4.04)  & 85.1 (4.05)  & 85.1 (4.05)  & & 125.4 (4.18) & 125.2 (4.17) & 126.9 (4.18) & 126.9 (4.18) \\
16 & 76.0 (4.04)  & 71.6 (4.00)  & 72.4 (3.99)  & 72.7 (4.00)  & & 112.6 (4.16) & 103.5 (4.12) & 105.4 (4.12) & 105.2 (4.12) \\
\bottomrule
\end{tabular}
}
\caption{\textbf{Effect of the minimum anchor time $t_{\text{anchor}}$ for DBTM + ril training on LM1B.} Gen-PPL (entropy) over NFE (refinement rounds) $\times$ threshold grid across $t_{\text{anchor}}\in \{0.60,0.75,0.80,0.85\}$, where $^\star$ indicates the default $t_{\text{anchor}}=0.75$ for LM1B and OWT. Columns within each anchor are commit thresholds ($\kappa\in \{0.7, 0.9, 0.99, 0.999\}$) and rows are NFE (refinement rounds); the $\kappa=0.9$ columns used in the main results are highlighted. Thresholds maintain constant performance at low rounds and take effect only at rounds 8--16, where stricter thresholds lower Gen-PPL. }
\label{tab:anchor-ablation}
\end{table}

\begin{table}[h!]
\centering
\small
\setlength{\tabcolsep}{5pt}
\resizebox{\textwidth}{!}{%
\begin{tabular}{l l cc c ccccccc}
\toprule
& & \multicolumn{2}{c}{\textbf{Best}} & & \multicolumn{7}{c}{\textbf{Gen-PPL at matched entropy}} \\
\cmidrule(lr){3-4}\cmidrule(lr){6-12}
NFE & Method & Gen-PPL~$\downarrow$ & at Ent. & Ent.\ range & 3.90 & 4.00 & 4.10 & 4.15 & 4.20 & 4.25 & 4.30 \\
\midrule
 & FMLM                 & 112.8 & 4.124 & 4.12--4.18 & -- & -- & -- & \textbf{116.2} & -- & -- & -- \\
\rowcolor{mybg} \cellcolor{white} & DBTM                 & 199.5 & 4.045 & 4.04--4.63 & -- & -- & 252.0 & 370.2 & 459.3 & 557.3 & 685.9 \\
\rowcolor{mybg} \cellcolor{white}\multirow{-3}{*}{1} & DBTM $+$ ril         & \textbf{86.0} & 3.917 & 3.92--4.40 & -- & \textbf{122.4} & \textbf{210.5} & 242.2 & \textbf{306.2} & \textbf{388.1} & \textbf{497.7} \\
\midrule
 & FMLM                 & 104.6 & 4.196 & 4.20--4.32 & -- & -- & -- & -- & \textbf{105.2} & \textbf{122.9} & \textbf{139.4} \\
\rowcolor{mybg} \cellcolor{white} & DBTM                 & 147.0 & 4.084 & 4.08--4.22 & -- & -- & 147.7 & 156.7 & 177.5 & -- & -- \\
\rowcolor{mybg} \cellcolor{white}\multirow{-3}{*}{2} & DBTM $+$ ril         & \textbf{71.6} & 3.967 & 3.97--4.25 & -- & \textbf{75.1} & \textbf{89.1} & \textbf{99.3} & 112.0 & 128.5 & -- \\
\midrule
 & FMLM                 & 93.4 & 4.207 & 4.21--4.32 & -- & -- & -- & -- & -- & 109.0 & 124.8 \\
\rowcolor{mybg} \cellcolor{white} & DBTM                 & 101.4 & 4.046 & 4.05--4.32 & -- & -- & 105.7 & 115.5 & 140.7 & 151.5 & 164.1 \\
\rowcolor{mybg} \cellcolor{white}\multirow{-3}{*}{4} & DBTM $+$ ril         & \textbf{62.5} & 3.926 & 3.93--4.40 & -- & \textbf{69.7} & \textbf{79.8} & \textbf{86.4} & \textbf{93.8} & \textbf{103.0} & \textbf{113.6} \\
\midrule
 & FMLM                 & 85.3 & 4.190 & 4.19--4.30 & -- & -- & -- & -- & 85.4 & 98.8 & -- \\
\rowcolor{mybg} \cellcolor{white} & DBTM                 & 74.8 & 3.986 & 3.99--4.31 & -- & 75.5 & 85.8 & 93.2 & 109.5 & 117.1 & 126.2 \\
\rowcolor{mybg} \cellcolor{white}\multirow{-3}{*}{8} & DBTM $+$ ril         & \textbf{49.5} & 3.867 & 3.87--4.45 & \textbf{50.8} & \textbf{60.1} & \textbf{68.6} & \textbf{73.2} & \textbf{78.7} & \textbf{85.4} & \textbf{92.7} \\
\midrule
 & FMLM                 & 76.7 & 4.219 & 4.22--4.31 & -- & -- & -- & -- & -- & 89.6 & 115.7 \\
\rowcolor{mybg} \cellcolor{white} & DBTM                 & 60.3 & 3.910 & 3.91--4.28 & -- & 65.2 & 76.0 & 85.4 & 90.5 & 96.1 & -- \\
\rowcolor{mybg} \cellcolor{white}\multirow{-3}{*}{16} & DBTM $+$ ril         & \textbf{44.2} & 3.832 & 3.83--4.47 & \textbf{46.6} & \textbf{57.3} & \textbf{62.1} & \textbf{66.1} & \textbf{70.4} & \textbf{75.8} & \textbf{81.7} \\
\bottomrule
\end{tabular}
}
\caption{\textbf{LM1B decode frontier shown in Figure \ref{fig:generative-frontier-lm1b}: DBTM vs.\ FMLM, swept over several inference-time knobs.} Each method is swept over the knobs that trace its own coherence-diversity trade-off, all of its sweeps are pooled, and the Pareto frontier is taken (a point survives if no other has both higher entropy and lower Gen-PPL). The best in the column within each NFE block are \textbf{bolded}. Blank cells indicate entropy values where the method produces no point on the Pareto frontier. \textbf{Swept parameters.} \emph{DBTM} (both rows): (i) \textbf{commit temperature} $\texttt{temp}\in[0,1.2]$, (ii) \textbf{prior scale} $\sigma\in[0.6,5]$, (iii) \textbf{commit threshold} $\kappa\in[0.1,0.999]$, (iv) \textbf{repetition penalty} $\lambda\in\{0,0.5,1\}$. \emph{FMLM}: (i) \textbf{churn} $\gamma\in[0,0.99]$, the noise re-injected between sampler steps, and (ii) the same \textbf{prior scale} $\sigma$, applied to its initial latent.}
\label{tab:decode-frontier-lm1b}
\end{table}

\begin{table}[h!]
\centering
\small
\setlength{\tabcolsep}{4pt}
\resizebox{\textwidth}{!}{%
\begin{tabular}{l l cc c cccccccccc}
\toprule
& & \multicolumn{2}{c}{\textbf{Best}} & & \multicolumn{10}{c}{\textbf{Gen-PPL at matched entropy}} \\
\cmidrule(lr){3-4}\cmidrule(lr){6-15}
NFE & Method & Gen-PPL~$\downarrow$ & at Ent. & Ent.\ range & 5.00 & 5.10 & 5.20 & 5.30 & 5.35 & 5.40 & 5.45 & 5.50 & 5.60 & 5.70 \\
\midrule
 & FMLM & 168.2 & 5.174 & 5.17--5.40 & -- & -- & 177.9 & 249.0 & 306.0 & 372.4 & -- & -- & -- & -- \\
 & FMLM$+$ & 378.5 & 5.344 & 5.34--5.50 & -- & -- & -- & -- & 384.4 & 433.9 & 509.4 & -- & -- & -- \\
\rowcolor{mybg} \cellcolor{white} & DBTM (attn) & 369.9 & 5.001 & 4.90--5.16 & 369.3 & 457.1 & -- & -- & -- & -- & -- & -- & -- & -- \\
\rowcolor{mybg} \cellcolor{white} & DBTM $+$ ril (attn) & 253.3 & 5.171 & 5.17--5.54 & -- & -- & 274.9 & 361.6 & 423.6 & 485.6 & 547.6 & 609.6 & -- & -- \\
\rowcolor{mybg} \cellcolor{white} & DBTM (lin) & 234.0 & 5.102 & 4.97--5.31 & 196.4 & 233.3 & 291.0 & 366.1 & -- & -- & -- & -- & -- & -- \\
\rowcolor{mybg} \cellcolor{white}\multirow{-6}{*}{1} & DBTM $+$ ril (lin) & \textbf{113.6} & 5.098 & 4.95--5.57 & \textbf{82.4} & \textbf{114.5} & \textbf{161.7} & \textbf{225.6} & \textbf{276.5} & \textbf{327.4} & \textbf{378.3} & \textbf{429.2} & -- & -- \\
\midrule
 & FMLM & 137.4 & 5.265 & 5.26--5.36 & -- & -- & -- & 163.2 & 235.2 & -- & -- & -- & -- & -- \\
 & FMLM$+$ & 114.5 & 5.136 & 5.14--5.60 & -- & -- & 142.6 & 223.5 & 299.6 & 379.5 & 432.4 & 515.6 & 666.7 & -- \\
\rowcolor{mybg} \cellcolor{white} & DBTM (attn) & 224.2 & 5.075 & 4.96--5.47 & 194.2 & 227.4 & 240.3 & 270.6 & 304.3 & 348.1 & 394.5 & -- & -- & -- \\
\rowcolor{mybg} \cellcolor{white} & DBTM $+$ ril (attn) & 199.5 & 5.370 & 5.37--5.78 & -- & -- & -- & -- & -- & 221.5 & 246.5 & 290.8 & 393.1 & 466.8 \\
\rowcolor{mybg} \cellcolor{white} & DBTM (lin) & 76.6 & 5.143 & 4.85--5.77 & \textbf{66.8} & \textbf{73.6} & 90.7 & 116.5 & 136.2 & 159.3 & 186.1 & 238.3 & 392.7 & 454.3 \\
\rowcolor{mybg} \cellcolor{white}\multirow{-6}{*}{2} & DBTM $+$ ril (lin) & \textbf{43.3} & 5.113 & 5.11--5.94 & -- & -- & \textbf{49.8} & \textbf{53.8} & \textbf{58.6} & \textbf{69.2} & \textbf{85.1} & \textbf{92.3} & \textbf{141.5} & \textbf{250.9} \\
\midrule
 & FMLM & 119.1 & 5.296 & 5.30--5.36 & -- & -- & -- & 121.6 & 185.5 & -- & -- & -- & -- & -- \\
 & FMLM$+$ & 111.4 & 5.165 & 4.80--5.63 & 87.1 & 101.8 & 126.8 & 218.4 & 245.2 & 272.0 & 298.8 & 325.6 & 464.5 & -- \\
\rowcolor{mybg} \cellcolor{white} & DBTM (attn) & 168.0 & 5.016 & 4.98--5.65 & 152.3 & 169.1 & 176.9 & 184.8 & 194.4 & 224.4 & 252.3 & 304.4 & 379.9 & -- \\
\rowcolor{mybg} \cellcolor{white} & DBTM $+$ ril (attn) & 169.2 & 5.376 & 5.38--5.97 & -- & -- & -- & -- & -- & 170.7 & 173.7 & 176.6 & 195.3 & 227.4 \\
\rowcolor{mybg} \cellcolor{white} & DBTM (lin) & 47.5 & 5.236 & 4.83--5.97 & \textbf{40.9} & \textbf{43.7} & 46.5 & 60.3 & 72.5 & 82.4 & 96.0 & 139.2 & 170.9 & 274.0 \\
\rowcolor{mybg} \cellcolor{white}\multirow{-6}{*}{4} & DBTM  $+$ ril (lin) & \textbf{34.7} & 5.128 & 5.13--6.03 & -- & -- & \textbf{39.2} & \textbf{42.0} & \textbf{42.4} & \textbf{44.5} & \textbf{48.0} & \textbf{51.9} & \textbf{77.2} & \textbf{90.3} \\
\midrule
 & FMLM & 95.5 & 5.280 & 5.28--5.38 & -- & -- & -- & 99.9 & 131.0 & -- & -- & -- & -- & -- \\
 & FMLM$+$ & 69.0 & 5.011 & 4.82--5.66 & 67.0 & 87.4 & 99.2 & 147.1 & 160.4 & 189.5 & 229.2 & 261.3 & 318.1 & -- \\
\rowcolor{mybg} \cellcolor{white} & DBTM  (attn) & 137.9 & 5.072 & 4.97--5.72 & 118.9 & 139.8 & 146.7 & 153.5 & 163.9 & 187.9 & 197.7 & 233.1 & 291.5 & 448.6 \\
\rowcolor{mybg} \cellcolor{white} & DBTM $+$ ril (attn) & 166.4 & 5.645 & 5.64--6.10 & -- & -- & -- & -- & -- & -- & -- & -- & -- & 180.6 \\
\rowcolor{mybg} \cellcolor{white} & DBTM (lin) & \textbf{33.3} & 5.269 & 4.88--6.03 & \textbf{26.4} & \textbf{28.9} & \textbf{31.5} & 43.6 & 46.7 & 51.6 & 61.3 & 74.1 & 167.7 & 186.5 \\
\rowcolor{mybg} \cellcolor{white}\multirow{-6}{*}{8} & DBTM $+$ ril (lin) & 37.2 & 5.423 & 4.82--6.04 & 32.8 & 33.9 & 34.9 & \textbf{35.9} & \textbf{36.4} & \textbf{37.0} & \textbf{38.8} & \textbf{41.9} & \textbf{52.6} & \textbf{71.8} \\
\midrule
 & FMLM & 76.8 & 5.267 & 5.27--5.44 & -- & -- & -- & 84.8 & 88.7 & 95.6 & -- & -- & -- & -- \\
 & FMLM$+$ & 62.0 & 5.001 & 4.96--5.63 & 61.7 & 75.5 & 89.1 & 140.6 & 143.2 & 145.9 & 202.7 & 251.5 & 295.1 & -- \\
\rowcolor{mybg} \cellcolor{white} & DBTM (attn) & 120.7 & 5.047 & 4.94--5.81 & 110.0 & 123.9 & 130.1 & 136.2 & 141.1 & 148.3 & 155.6 & 184.7 & 257.7 & 402.3 \\
\rowcolor{mybg} \cellcolor{white} & DBTM $+$ ril (attn) & 131.3 & 5.614 & 5.61--6.11 & -- & -- & -- & -- & -- & -- & -- & -- & -- & 135.6 \\
\rowcolor{mybg} \cellcolor{white} & DBTM (lin) & \textbf{21.3} & 5.005 & 4.88--5.99 & \textbf{21.2} & \textbf{23.5} & \textbf{25.8} & \textbf{37.5} & \textbf{37.6} & 41.6 & 54.3 & 63.5 & 94.1 & 141.9 \\
\rowcolor{mybg} \cellcolor{white}\multirow{-6}{*}{16} & DBTM $+$ ril (lin) & 30.1 & 5.355 & 5.36--6.12 & -- & -- & -- & -- & -- & \textbf{32.0} & \textbf{34.2} & \textbf{36.2} & \textbf{42.3} & \textbf{48.9} \\
\bottomrule
\end{tabular}
}
\caption{\textbf{OWT decode frontier shown in Figure \ref{fig:discrete-btm}: DBTM vs.\ FMLM and FMLM+, swept over several inference-time knobs.} Companion to the decode-frontier figure. Each method is swept over the knobs that trace its own coherence-diversity trade-off; all of its sweeps are pooled, and the Pareto frontier is taken (a point survives if no other has both higher entropy and lower Gen-PPL). The best in the column within each NFE block are \textbf{bolded}. Blank cells indicate entropy values where the method produces no point on the Pareto frontier. Entropy is the within-sequence token-frequency entropy, with real OWT at entropy $5.441$. (attn) denotes the same DiT architecture as FMLM, and (lin) denotes the linear attention architecture with matched parameters described in App. \ref{app:language-exp-details}. FMLM and FMLM+ values are evaluated with published checkpoints. \textbf{Swept parameters.} \emph{DBTM}: (i) \textbf{commit temperature} $\texttt{temp}\in\{0,0.7,1.0\}$, (ii) \textbf{prior scale} $\sigma\in[0.6,5]$, (iii) \textbf{commit threshold} $\kappa\in[0.1,0.999]$, (iv) \textbf{repetition penalty} $\lambda\in\{0,0.5,1\}$. \emph{FMLM}: (i) \textbf{churn} $\gamma\in\{0,0.1,0.2,0.3,0.5\}$, the noise re-injected between sampler steps, and (ii) a \textbf{commit temperature} $\texttt{temp}\in\{0,0.5,0.7,0.85,1.0\}$ at its final argmax, which was added on top of the method. \emph{FMLM$+$}: (i) the \textbf{split of the NFE budget} into refinement rounds $\times$ ODE substeps, (ii) the \textbf{commit schedule}, a uniform per-round budget or a fixed top-$k\in\{64,128,256,384,768\}$ per round, (iii) \textbf{commit threshold} $\kappa\in\{0.9,0.999\}$, and (iv) a \textbf{commit temperature} $\texttt{temp}\in[0,2]$ at its decode, also added on top of the method.}
\label{tab:decode-frontier-owt}
\end{table}

\clearpage
\section{Experiment Details}
\label{app:exp-details}
\subsection{General Experiment Details}
\label{app:general-exp-details}
\paragraph{Training Details}
All models use AdamW ($\beta_1{=}0.9$, $\beta_2{=}0.999$, $\epsilon{=}10^{-8}$) in bf16, zero weight decay, gradient clipping at $1.0$, and peak learning rate $3\times10^{-4}$ with linear warmup. We maintain an exponential moving average (EMA) of the map weights with decay $0.9999$, and a second, independent EMA of the predictor $\phi$ with the same decay. All reported results use the EMA weights. The global batch is given by per-GPU batch $\times$ devices $\times$ accumulation, so runs reproduce at any device count: $512$ (per-GPU $128$) for language modeling, $256$ for reasoning. 

\paragraph{Architecture}
The backbone architecture is a DiT \citep{peebles2023scalable} with adaLN conditioning (dimension $128$) and rotary positional embeddings \citep{su2024roformer}, acting on a continuous $(B,L,V)$ simplex latent without conditioning on time $t$. Language modeling and TinyGSM use $768$ hidden / $12$ layers / $12$ heads, Sudoku $512$ / $8$ / $8$, all with FFN multiplier $4$, dropout $0.1$, untied embeddings, and a $\tanh$ softcap of $50$ on attention logits (Table~\ref{tab:hyperparameters}).

\paragraph{Fused Kernel for JVP Computation}
The transport loss \eqref{loss:transport} is defined through $\dot I_t\cdot\nabla T_\theta(I_t)$, so each step needs a Jacobian-vector product (JVP) through the network, evaluated in forward mode. This dominates our training cost since it requires materializing the $\texttt{(B, heads, L, L)}$ score tensor. We instead use a fused Triton kernel~\citep{zhou2026terminal} that computes the attention output and its derivative together in a single streaming pass, so the full score tensor is never held in memory. The two passes share one sweep over the keys and values: the ordinary forward pass only reserves the output buffers, and the kernel fills them while computing the derivative. 

Backpropagation is handled by two separate kernels, one for the output and the derivative, which can be cleanly switched off to the standard JVP mechanism by setting \texttt{fast\_jvp\_attn=false}. Measured on an A100 over a four-block DiT, it cuts peak memory by $1.4\times$ at $B{=}64$, $L{=}128$, $V{=}12$ and by $3.7\times$ at $B{=}8$, $L{=}1024$, $V{=}8192$, with a $1.2\times$ and $1.7\times$ speedup respectively, while matching the unfused implementation to a cosine similarity of at least $0.99999$ on the prediction, the derivative, and the gradients. One limitation is that the kernel does not implement the $\tanh$ softcap on attention logits. We therefore checked on each task that the logits stay well below the cap before using the fused kernel.

\subsection{Unconditional Language Modeling Experiment Details}
\label{app:language-exp-details}
\paragraph{Datasets and Tokenization}
LM1B uses the \texttt{bert-base-uncased} tokenizer ($V{=}30{,}522$) at length $128$ and OpenWebText uses GPT-2 BPE ($V{=}50{,}257$) at length $1024$. Documents are concatenated and wrapped into fixed-length blocks with EOS at every boundary, so all sequences are full and no padding is needed.

\paragraph{Linear-Attention Backbone} 
The \textbf{DBTM + lin} rows replace nine of the twelve attention blocks with bidirectional gated delta-rule mixers following DeltaFlow \citep{guo2026deltaflow}, keeping full attention at every fourth block (\texttt{[gdn,gdn,gdn,attn]}, chunk size 256, alternating scan direction). Each gated DeltaNet \citep{yang2025gated,yang2024parallelizing} runs forward and backward over the sequence. Since DBTM is autonomous, the decay and write-rate gates that DeltaFlow conditions on diffusion time are instead conditioned on the current simplex state. We do not use DeltaFlow's temporal-state-consistency loss. Hidden size, depth, heads, dropout, the predictor $\phi$ and every training knob are identical to the DiT architecture with matched parameter count (171.9M for attention, 177.5M for linear), so the two backbones differ only in the attention mechanism. The linear rows are trained from scratch at the same batch size and learning rate and reported at the same step count. The delta-rule attention mechanism costs time linear in the sequence length $L$, whereas dense attention costs quadratic. In Figure \ref{fig:discrete-btm}, Table \ref{tab:language-results}, and Table \ref{tab:language-results-nfe}, we show the linear architecture on OpenWebText ($L{=}1024$) as a proof-of-concept for long-context modeling. 

\paragraph{Training Configuration}
Both datasets use a global batch of $512$ and $10{,}000$ warmup steps, a linear interpolant with $\sigma{=}1.0$, $t_{\text{anchor}}{=}0.75$, OT coupling, uniform time sampling, adaptive transport and cross-entropy endpoint losses, and the \texttt{clean\_context} per-token schedule, which trains the map on partially committed sequences and so matches the refinement sampler.

\paragraph{Baselines}
For unconditional language modeling on LM1B and OWT, we compare DBTM against distilled discrete diffusion methods, including Duo \citep{sahoo2025diffusion} with DCD \citep{sahoo2025diffusion}, MDLM \citep{sahoo2024simple} with SDTT \citep{deschenaux2025beyond}, and both Duo and MDLM with Di4C \citep{hayakawa2024distillation}, and flow map methods, including categorical flow maps (CFM) \citep{roos2026categorical}, flow map language models (FMLM) \citep{lee2026flow}, FMLM with posterior refinement (FMLM+) \citep{agarwal2026posterior}, and discrete flow maps (DFM) with both the semigroup (PSD) and Eulerian (ESD) objectives \citep{potaptchik2026discrete}. We evaluate FMLM+ using the released checkpoints with one integration step per refinement round, like DBTM. Baseline values are taken from \citet{lee2026flow, potaptchik2026discrete} or evaluated from published checkpoints. 

\paragraph{Evaluation Metrics}
We report generative perplexity (Gen. PPL) scored by GPT-2 Large \citep{radford2019language} and sample entropy over $1024$ generated sequences. The goal is to minimize Gen. PPL while maximizing entropy to match the data entropy of LM1B ($\texttt{entropy}{=}4.336$) and OWT ($\texttt{entropy}{=}5.441$).

\subsection{Reasoning Experiment Details}
\label{app:reason-exp-details}
\paragraph{Datasets and Tokenization}
Following S-FLM \citep{deschenaux2026language} and FMLM+ \citep{agarwal2026posterior}, Sudoku puzzles with unique solutions are generated at three difficulties set by clue count (easy $40$, medium $35$, hard $30$), with $48{,}000$ train and $2{,}000$ held-out puzzles per difficulty. Sudoku uses a $12$-token vocabulary (empty, digits $1$-$9$, row separator, BOS) and a grid is $81$ cells plus $8$ separators $=89$ tokens, giving the layout \texttt{[BOS] puzzle(89) [BOS] solution(89)} $=180$ tokens with prompt length $91$ and loss on the solution half. TinyGSM \citep{liu2023tinygsm} pairs grade-school word problems with Python programs, tokenized by SmolLM-135M ($V{=}49{,}152$) as \texttt{[BOS] question} \verb|\n| \texttt{program [EOS]} at length $512$, supervising the program region only. Evaluation uses the real GSM8K test set ($1{,}319$ problems) that was never trained on.

\paragraph{Training Configuration}
Sudoku uses the $512$-wide backbone (${\approx}30$M parameters) at $L{=}180$ and TinyGSM the $768$-wide one (${\approx}165$M) at $L{=}512$, with global batch $256$ and $30{,}000$ / $50{,}000$ warmup steps respectively. Transport-map training matches the language setting (linear interpolant, $\sigma{=}1.0$, OT coupling, uniform time, \texttt{clean\_context}, adaptive transport and cross-entropy endpoint losses, anchor and boundary weights $1.0$), with $t_{\text{anchor}}{=}0.69$ for Sudoku and $t_{\text{anchor}}{=}0.82$ for TinyGSM. The predictor $\phi$ is a $256$-wide, $4$-layer, $4$-head head with quality weight $1.0$, positive-class reweighting, clean-context labels, and BCE.

\paragraph{Baselines}
For the reasoning tasks, we compare DBTM against many-step autoregressive sampling (AR), discrete diffusion methods, including MDLM \citep{sahoo2024simple} and Duo \citep{sahoo2025diffusion}, continuous flow methods, including CANDI \citep{pynadath2025candi}, FLM \citep{lee2026flow}, and S-FLM \citep{deschenaux2026language}, and the few-step flow map baseline FMLM+ \citep{agarwal2026posterior}. Baseline values are taken from \citet{agarwal2026posterior} or evaluated from published checkpoints. 

\paragraph{Evaluation Metrics}
For Sudoku, we report exact-solve accuracy over the $2{,}000$ held-out puzzles. For GSM8K, we report answer accuracy by executing the generated program in a sandbox against the true answer, counting parse or execution failures as incorrect. All numbers use EMA weights and argmax commitment.

\begin{table}[h!]
\centering
\resizebox{\textwidth}{!}{%
\begin{tabular}{l cccc}
\toprule
& LM1B & OpenWebText & TinyGSM / GSM8K & Sudoku \\
\midrule
\textit{Architecture (DiT)} & & & & \\
Hidden size          & 768 & 768 & 768 & 512 \\
adaLN cond.\ dim      & 128 & 128 & 128 & 128 \\
Layers               & 12 & 12 & 12 & 8 \\
Attention heads      & 12 & 12 & 12 & 8 \\
FFN multiplier       & 4 & 4 & 4 & 4 \\
Dropout              & 0.1 & 0.1 & 0.1 & 0.1 \\
Sequence length $L$  & 128 & 1024 & 512 & 180 \\
Vocabulary $|V|$     & 30{,}522 & 50{,}257 & 49{,}152 & 12 \\
Parameters (approx.) & $\sim$140M & $\sim$170M & $\sim$165M & $\sim$30M \\
\midrule
\textit{Optimization} & & & & \\
Optimizer            & \multicolumn{4}{c}{AdamW ($\beta_1{=}0.9$, $\beta_2{=}0.999$, $\epsilon{=}10^{-8}$)} \\
Weight decay         & \multicolumn{4}{c}{0} \\
Gradient clipping    & \multicolumn{4}{c}{1.0} \\
Schedule             & \multicolumn{4}{c}{costant LR with linear warmup} \\
Peak LR              & $3{\times}10^{-4}$ & $3{\times}10^{-4}$ & $3{\times}10^{-4}$ & $3{\times}10^{-4}$ \\
Max steps          & 250{,}000 & 150{,}000 & 250{,}000 & 200{,}000 \\
Global batch size    & 512 & 512 & 512 & 256 \\
Per-GPU batch size   & 128 & 128 & 32 & 64 \\
Precision            & \multicolumn{4}{c}{bf16} \\
EMA decay            & \multicolumn{4}{c}{0.9999} \\
\midrule
\textit{Transport map training} & & & & \\
Interpolant schedule & linear & linear & linear & linear \\
Noise scale $\sigma$ & 1.0 & 1.0 & 1.0 & 1.0 \\
OT coupling          & \multicolumn{4}{c}{on} \\
Per-token schedule   & \multicolumn{4}{c}{clean\_context} \\
Time schedule        & \multicolumn{4}{c}{uniform} \\
Minimum anchor time $t_{\text{anchor}}$    & 0.75 & 0.75 & 0.82 & 0.69 \\
Anchor $\lambda$ / semigroup $\lambda$ / boundary $\lambda$ & 1.0 / 0 / 1.0 & 1.0 / 0 / 1.0 & 1.0 / 0 / 1.0 & 1.0 / 0 / 1.0 \\
Adaptive stabilizer $c$ (transport / endpoint) & 1.0 & 1.0 & 1.0 & 1.0 \\
Adaptive weight $r$ (transport / endpoint) & 0.5 / 0.5 & 0.5 / 0.5 & 0.5 / 0.5 & 0.5 / 0.5 \\
\midrule
\textit{Inference hyperparameters} & & & & \\
Commit threshold $\kappa$ & 0.9 & 0.3 & 0.8 & 0.999 \\
Commit scorer        & softmax conf. & softmax conf. & softmax conf. & $\mu$ ($\phi$ head) \\
Commit temperature   & 0 (argmax) & 0 (argmax)  & 0 (argmax) & 0 (argmax) \\
Prior scale $\sigma$ & 1.0 & 1.0 & 1.0 & 1.0 \\
Repetition penalty $\lambda$ & 1.0 & 1.0 & 0 & 0 \\
\bottomrule
\end{tabular}
}
\caption{\textbf{Per-task default hyperparameters.} For model architectures and hyperparameters shared with baselines, we mostly follow the parameters used in the FMLM \citep{lee2026flow} baseline for LM1B and OWT, and the FMLM+ \citep{agarwal2026posterior} baseline for Sudoku and TinyGSM/GSM8K. Note that the base hyperparameters in App \ref{app:hyperparameters} may be different across ablations of different parameters but are constant within each parameter.}
\label{tab:hyperparameters}
\end{table}

\clearpage

\section{Algorithms}
\label{app:algorithms}
\begin{algorithm}[h!]
\caption{Training Discrete Beckmann Transport Map}
\label{alg:btm-train-map}
\begin{algorithmic}[1]
\Require Samples $\{x_0\}\sim\mu_0$, $\{x_1\}\sim\mu_1$; interpolant $\alpha_t$; transport map $T_\theta$; loss weights $\lambda_{\text{bnd}},\lambda_{\text{semi}},\lambda_{\text{anchor}}$; anchor minimum time $t_{\text{anchor}}$; self-composition depths $\mathcal K$; learning rate $\eta$
\While{not converged}
  \State Sample a minibatch $\mathcal B$;\; reset all losses to $0$
  \For{$(x_0,x_1)\in\mathcal B$}
    \State Sample $t\sim\mathcal U(0,1)$
    \State \HL{I_t^\ell\gets\alpha_t x_0^\ell+(1-\alpha_t)x_1^\ell,\qquad
           \dot I_t^\ell\gets\dot\alpha_t\,(x_0^\ell-x_1^\ell)}\quad for all $\ell$
    \State \HL{\mathcal L_{\text{transport}}\gets\mathcal L_{\text{transport}}+\sum_\ell |T_\theta(I_t^\ell)-\texttt{sg}\big(T_\theta(I_t^\ell)+\dot I_t^\ell\cdot\nabla T_\theta(I_t^\ell)\big)|^2}
          \Comment{transport}
    \State $\mathcal L_{\text{bnd}}\gets\mathcal L_{\text{bnd}}+\sum_\ell\text{CE}\big(T_\theta(x_1)^\ell,\,x_1^\ell\big)$
          \Comment{boundary}
    \State $\mathcal L_{\text{anchor}}\gets\mathcal L_{\text{anchor}}+\sum_\ell\boldsymbol 1[t>t_{\text{anchor}}]\;\text{CE}\big(T_\theta(I_t)^\ell,\,x_1^\ell\big)$
          \Comment{anchor near-clean tokens}
    \State $\mathcal L_{\text{semi}}\gets\mathcal L_{\text{semi}}+\sum_{k\in\mathcal K}\sum_\ell\text{CE}\big(T_\theta(I_t)^\ell,\;\texttt{sg}\big(T_\theta^{\circ k}(I_t)\big)^\ell\big)$
          \Comment{self-distill toward $k$-fold refinement}
  \EndFor
  \State $\mathcal L_{\text{total}}\gets\frac{1}{|\mathcal B|}\Big[\mathcal L_{\text{transport}}
         +\lambda_{\text{bnd}}\mathcal L_{\text{bnd}}
         +\lambda_{\text{anchor}}\mathcal L_{\text{anchor}}
         +\lambda_{\text{semi}}\mathcal L_{\text{semi}}\Big]$
         \Comment{total loss}
  \State $\theta\gets\theta-\eta\,\nabla_\theta\mathcal L_{\text{total}}$
\EndWhile
\State \Return $T_\theta$
\end{algorithmic}
\end{algorithm}

\begin{algorithm}[h!]
\caption{Sampling Discrete Beckmann Transport Map with Refinement (\textsc{Rollout})}
\label{alg:btm-sampling-refinement}
\begin{algorithmic}[1]
\Require Trained autonomous transport map $T_\theta:\mathbb{R}^{L\times V}\to\mathbb{R}^{L\times V}$; confidence threshold $\kappa$; refinement depth $k$; prior noise level $\sigma$; refinement mode $\in\{\Mode{quality},\Mode{confidence},\Mode{none}\}$; commit mode $\in\{\Mode{commit},\Mode{non-commit}\}$
\State $\hat x_{0,k}\gets(x_0^1,\dots,x_0^L)$,\; $x_0^\ell\sim\mathcal{N}(\boldsymbol 0,\sigma\boldsymbol I_V)$;\quad $\texttt{trace}\gets[\,]$
      \Comment{every token starts as prior noise}
\For{$r=1,\dots,k$}
  \State $(\hat h,\hat x)\gets T_\theta(\hat x_{r-1,k})$
        \Comment{hidden representation $\hat h$ and categorical proposal $\hat x$}
  \State $n_r\gets\big\lceil|\mathcal R_r|/(k-r+1)\big\rceil$;\quad $\mathcal C_r\gets\emptyset$
        \Comment{commit floor for round $r$}
  \For{$\ell=1,\dots,L$}
    \If{\Mode{quality}} $q_\phi^\ell(\hat x)\gets\textsc{QualityHead}(\hat h)$
        \Comment{learned token quality}
    \ElsIf{\Mode{confidence}} $q_\phi^\ell(\hat x)\gets\langle\arg\max(\hat x^\ell),\,\hat x^\ell\rangle$
        \Comment{rounding confidence}
    \Else\ $q_\phi^\ell(\hat x)\gets 1$
    \EndIf
    \If{$q_\phi^\ell(\hat x)<\kappa$ \textbf{and} $|\mathcal C_r|>n_r$}
      \State \HL{\hat x_{r,k}^\ell\sim\mathcal N(\boldsymbol 0,\sigma\boldsymbol I_V)}
            \Comment{low quality and floor already met: fresh noise}
    \ElsIf{\Mode{commit}}
      \State \HL{\hat x_{r,k}^\ell\gets x_1^\ell},\quad $\mathcal C_r\gets\mathcal C_r\cup\{\ell\}$
            \Comment{commit hard token}
    \ElsIf{\Mode{non-commit}}
      \State \HL{\hat x_{r,k}^\ell\gets \hat x^\ell},\quad $\mathcal C_r\gets\mathcal C_r\cup\{\ell\}$
            \Comment{hold soft token distribution}
    \EndIf
  \EndFor
  \State $\texttt{trace}.\texttt{append}\big(\hat x_{r,k},\mathcal C_r\big)$
        \Comment{record partially committed sequence}
\EndFor
\State \Return $\hat x_{k,k}$,\quad $\texttt{trace}=\{(\hat x_{r,k},\mathcal C_r)\}_{r=1}^k$
\end{algorithmic}
\end{algorithm}

\begin{algorithm}[h!]
\caption{Refinement-in-Loop Losses (\textsc{RIL})}
\label{alg:btm-ril}
\begin{algorithmic}[1]
\Require Sample pair $(x_0,x_1)$; transport map $T_\theta$; quality head $q_\phi$; refinement depth $k$; current \texttt{step}; \texttt{warmup}; flag \Mode{quality}
\State $\mathcal L_{\text{ril-c}},\mathcal L_{\text{semi-r}},\mathcal L_{\text{quality}}\gets 0$
\State $\{(\hat x_{r,k},\mathcal C_r)\}_{r=1}^k\gets\textsc{Rollout}(x_0,x_1,k)$
      \Comment{Algorithm~\ref{alg:btm-sampling-refinement}}
\For{$r=1,\dots,k$}
  \State $\mathcal L_{\text{ril-c}}\gets\mathcal L_{\text{ril-c}}+\sum_{\ell\notin\mathcal C_r}\text{CE}\big(T_\theta(\hat x_{r-1,k})^\ell,\,x_1^\ell\big)$
        \Comment{supervise uncommitted tokens}
  \If{\Mode{quality}}
    \State $\hat x_1^\ell\gets\arg\max T_\theta(\hat x_{r-1,k})^\ell$ for all $\ell$
    \State $\mathcal L_{\text{quality}}\gets\mathcal L_{\text{quality}}+\sum_\ell\text{BCE}\big(\boldsymbol 1[\hat x_1^\ell=x_1^\ell],\;q_\phi^\ell(T_\theta(\hat x_{r-1,k}))\big)$
  \EndIf
  \If{$\texttt{step}>\texttt{warmup}$}
    \State $\mathcal L_{\text{semi-r}}\gets\mathcal L_{\text{semi-r}}+\sum_{\ell\notin\mathcal C_r}\text{CE}\big(T_\theta(\hat x_{r-1,k})^\ell,\;\texttt{sg}\big(T_\theta(\hat x_{k-1,k})\big)^\ell\big)$
          \Comment{self-distill toward final round}
  \EndIf
\EndFor
\State \Return $\mathcal L_{\text{ril-c}},\;\mathcal L_{\text{semi-r}},\;\mathcal L_{\text{quality}}$
\end{algorithmic}
\end{algorithm}

\begin{algorithm}[h!]
\caption{Training Discrete Beckmann Transport Map with Partial Context}
\label{alg:btm-train-context}
\begin{algorithmic}[1]
\Require Samples $\{x_0\}\sim\mu_0$, $\{x_1\}\sim\mu_1$; interpolant $\alpha_t$; transport map $T_\theta$; quality head $q_\phi$; loss weights $\lambda_{\text{bnd}},\lambda_{\text{anchor}},\lambda_{\text{semi}},\lambda_{\text{ril}}$; anchor minimum time $t_{\text{anchor}}$; \texttt{warmup}; \texttt{max\_steps}; learning rate $\eta$; flags \Mode{quality}, \Mode{ril}
\For{$\texttt{step}=1,\dots,\texttt{max\_steps}$}
  \State Sample a minibatch $\mathcal B$;\; reset all losses to $0$
  \For{$(x_0,x_1)\in\mathcal B$}
    \State Sample clean fraction $f\sim\mathcal U(0,1)$ and noisy time $t\sim\mathcal U(0,1)$
    \For{$\ell=1,\dots,L$}
      \State $c_\ell\sim\text{Bernoulli}(f)$;\quad $t_\ell\gets c_\ell+(1-c_\ell)\,t$
            \Comment{$c_\ell=1$: token $\ell$ is clean context}
      \State \HL{I_t^\ell\gets\alpha_{t_\ell}x_0^\ell+(1-\alpha_{t_\ell})x_1^\ell,\qquad
             \dot I_t^\ell\gets\dot\alpha_{t_\ell}\,(x_0^\ell-x_1^\ell)}
    \EndFor
    \State \HL{\mathcal L_{\text{transport}}\gets\mathcal L_{\text{transport}}+\sum_\ell |T_\theta(I_t^\ell)-\texttt{sg}\big(T_\theta(I_t^\ell)+\dot I_t^\ell\cdot\nabla T_\theta(I_t^\ell)\big)|^2}
          \Comment{transport}
    \State $\mathcal L_{\text{bnd}}\gets\mathcal L_{\text{bnd}}+\sum_\ell\text{CE}\big(T_\theta(x_1)^\ell,\,x_1^\ell\big)$
          \Comment{boundary}
    \State $\mathcal L_{\text{anchor}}\gets\mathcal L_{\text{anchor}}+\sum_\ell\boldsymbol 1[t_\ell>t_{\text{anchor}}]\;\text{CE}\big(T_\theta(I_t)^\ell,\,x_1^\ell\big)$
          \Comment{anchor near-clean tokens}
    \If{\Mode{quality}}
      \State $\hat x_1^\ell\gets\arg\max T_\theta(I_t)^\ell$ for all $\ell$
      \State $\mathcal L_{\text{quality}}\gets\mathcal L_{\text{quality}}+\sum_\ell\text{BCE}\big(\boldsymbol 1[\hat x_1^\ell=x_1^\ell],\;q_\phi^\ell(T_\theta(I_t))\big)$
            \Comment{predict per-token correctness}
    \EndIf
    \If{\Mode{ril}}
      \State $(\ell_{\text{ril-c}},\ell_{\text{semi-r}},\ell_{\text{quality}})\gets\textsc{RIL}(x_0,x_1,k,\texttt{step})$
            \Comment{Algorithm~\ref{alg:btm-ril}}
      \State $\mathcal L_{\text{ril-c}}\gets\mathcal L_{\text{ril-c}}+\ell_{\text{ril-c}}$;\quad
             $\mathcal L_{\text{semi-r}}\gets\mathcal L_{\text{semi-r}}+\ell_{\text{semi-r}}$;\quad
             $\mathcal L_{\text{quality}}\gets\mathcal L_{\text{quality}}+\ell_{\text{quality}}$
    \EndIf
  \EndFor
  \State $\mathcal L_{\text{total}}\gets\frac{1}{|\mathcal B|}\Big[\mathcal L_{\text{transport}}
         +\lambda_{\text{bnd}}\mathcal L_{\text{bnd}}
         +\lambda_{\text{anchor}}\mathcal L_{\text{anchor}}
         +\lambda_{\text{ril}}\mathcal L_{\text{ril-c}}
         +\lambda_{\text{semi}}\mathcal L_{\text{semi-r}}\Big]$
         \Comment{total loss}
  \State $\theta\gets\theta-\eta\,\nabla_\theta\mathcal L_{\text{total}}$;\quad
         $\phi\gets\phi-\eta\,\nabla_\phi\mathcal L_{\text{quality}}$
\EndFor
\State \Return $T_\theta$, $q_\phi$
\end{algorithmic}
\end{algorithm}

\clearpage
\section{Example Generations}
\label{app:example-generations}
Figure~\ref{fig:samples_refinement_lm1b} shows unconditional LM1B samples drawn under budgets of 1, 2, 4, and 8 function evaluations, each box annotated with the generative perplexity and entropy of that sample. Figures~\ref{fig:refine_sample_gsm8k_1} and~\ref{fig:refine_sample_gsm8k_2} show two TinyGSM solutions across a 32-round refinement trajectory. The conditioning question is set in \textcolor{muted}{gray} and stays fixed, while generated positions move from \textcolor{refining}{purple} to \textcolor{committed}{turquoise}. Figure~\ref{fig:sudoku-example} applies the same coloring over to a Sudoku-Hard puzzle solved in 4 NFEs.

\begin{figure}[h!]
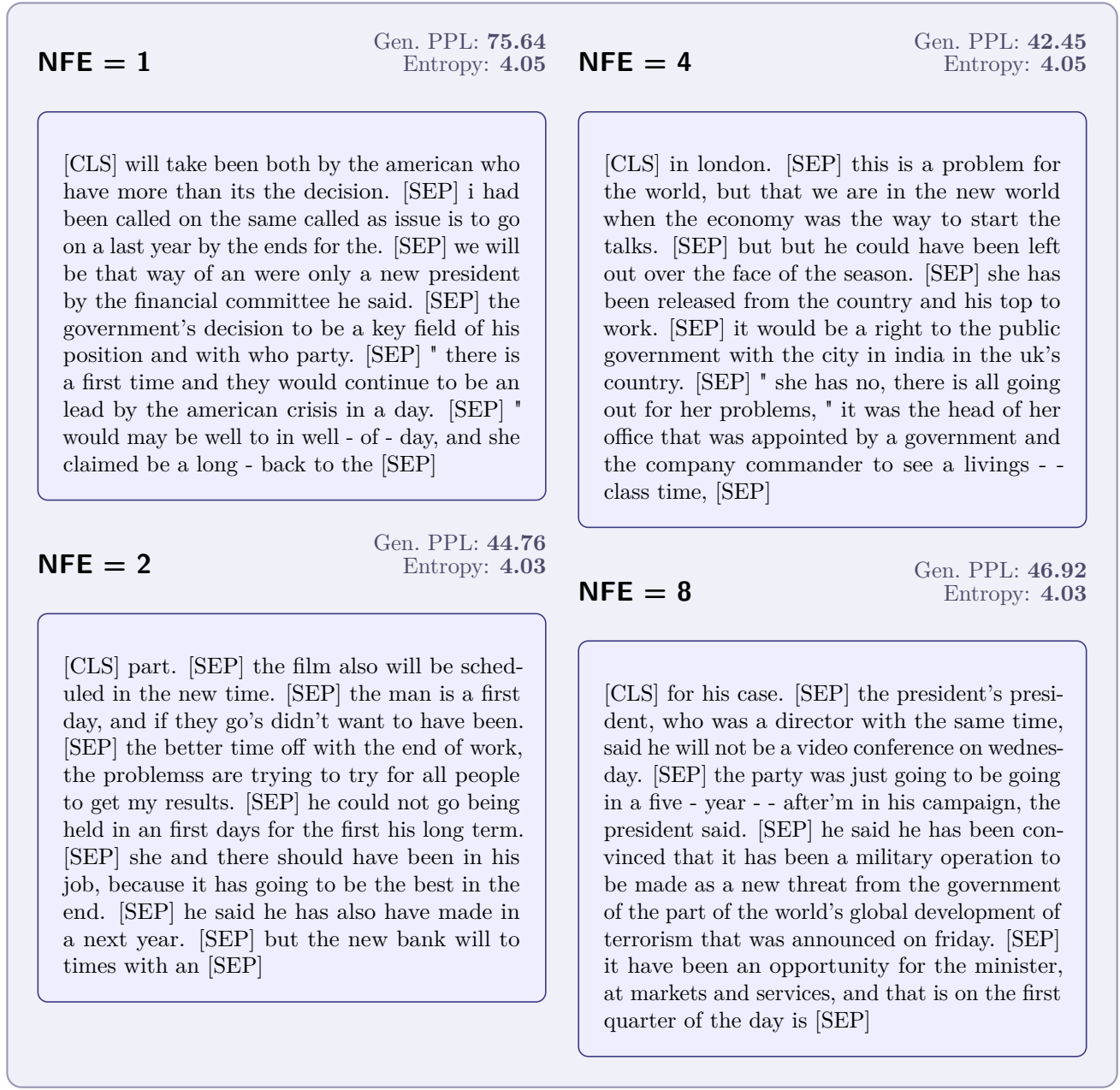

\centering
\begin{figurepanel}

\noindent
\begin{minipage}[t]{0.485\linewidth}
\nfeheader{NFE = 1}\hfill
{\color{muted}\shortstack[r]{Gen.\ PPL: \textbf{75.64}\\Entropy: \textbf{4.05}}}

\vspace{8pt}

\begin{samplebox}
\noindent
\begin{minipage}[t]{0.56\linewidth}
\methodtitle{}
\end{minipage}\hfill
\begin{minipage}[t]{0.40\linewidth}
\end{minipage}

\vspace{4pt}
\noindent
[CLS] will take been both by the american who have more than its the decision. [SEP] i had been called on the same called as issue is to go on a last year by the ends for the. [SEP] we will be that way of an were only a new president by the financial committee he said. [SEP] the government's decision to be a key field of his position and with who party. [SEP] " there is a first time and they would continue to be an lead by the american crisis in a day. [SEP] " would may be well to in well - of - day, and she claimed be a long - back to the [SEP]
\end{samplebox}

\vspace{8pt}

\nfeheader{NFE = 2}\hfill
{\color{muted}\shortstack[r]{Gen.\ PPL: \textbf{44.76}\\Entropy: \textbf{4.03}}}

\vspace{8pt}

\begin{samplebox}
\noindent
\begin{minipage}[t]{0.56\linewidth}
\methodtitle{}
\end{minipage}\hfill
\begin{minipage}[t]{0.40\linewidth}
\end{minipage}

\vspace{4pt}
\noindent
[CLS] part. [SEP] the film also will be scheduled in the new time. [SEP] the man is a first day, and if they go's didn't want to have been. [SEP] the better time off with the end of work, the problemss are trying to try for all people to get my results. [SEP] he could not go being held in an first days for the first his long term. [SEP] she and there should have been in his job, because it has going to be the best in the end. [SEP] he said he has also have made in a next year. [SEP] but the new bank will to times with an [SEP]
\end{samplebox}

\end{minipage}
\hfill
\begin{minipage}[t]{0.485\linewidth}
\nfeheader{NFE = 4}\hfill
{\color{muted}\shortstack[r]{Gen.\ PPL: \textbf{42.45}\\Entropy: \textbf{4.05}}}

\vspace{8pt}

\begin{samplebox}
\noindent
\begin{minipage}[t]{0.56\linewidth}
\methodtitle{}
\end{minipage}\hfill
\begin{minipage}[t]{0.40\linewidth}
\end{minipage}

\vspace{4pt}
\noindent
[CLS] in london. [SEP] this is a problem for the world, but that we are in the new world when the economy was the way to start the talks. [SEP] but but he could have been left out over the face of the season. [SEP] she has been released from the country and his top to work. [SEP] it would be a right to the public government with the city in india in the uk's country. [SEP] " she has no, there is all going out for her problems, " it was the head of her office that was appointed by a government and the company commander to see a livings - - class time, [SEP]
\end{samplebox}

\vspace{8pt}

\nfeheader{NFE = 8}\hfill
{\color{muted}\shortstack[r]{Gen.\ PPL: \textbf{46.92}\\Entropy: \textbf{4.03}}}

\vspace{8pt}

\begin{samplebox}
\noindent
\begin{minipage}[t]{0.56\linewidth}
\methodtitle{}
\end{minipage}\hfill
\begin{minipage}[t]{0.40\linewidth}
\end{minipage}

\vspace{4pt}
\noindent
[CLS] for his case. [SEP] the president's president, who was a director with the same time, said he will not be a video conference on wednesday. [SEP] the party was just going to be going in a five - year - - after'm in his campaign, the president said. [SEP] he said he has been convinced that it has been a military operation to be made as a new threat from the government of the part of the world's global development of terrorism that was announced on friday. [SEP] it have been an opportunity for the minister, at markets and services, and that is on the first quarter of the day is [SEP]
\end{samplebox}
\end{minipage}

\end{figurepanel}
\caption{\textbf{Example LM1B generations from DBTM at 1, 2, 4, and 8 NFEs }(each NFE is one refinement round). Generative perplexity and entropy are reported per length-128 sample.}
\label{fig:samples_refinement_lm1b}
\end{figure}

\clearpage

\newpage

\begin{figure}[t]
\centering
\begin{figurepanel}
\noindent
\nfeheader{NFE = 16}\hfill
{\color{muted}\shortstack[r]{Gen.\ PPL: \textbf{45.5}\\Entropy: \textbf{5.58}}}
\vspace{8pt}

\begin{samplebox}
\footnotesize\linespread{0.95}\selectfont
[END]...I mean, I think we started at the end of last year to figure out what we're doingre doing right now. And mean… when you look at your [data], you know that about \$100,000 per year — That sounds pretty hard; i'm still sure little bit." But I think this is going to increase increased by about \$100,000 per year — That's pretty single very important! This is very important and it's 'sthe single most important thing in our history. And so far i've've seen huge number million Americans coming across every day months' past decade...Our population has increased by about \$100,000 per year — those people who don't caretakers." This is consistent with the United States — basically what's called “the United Kingdom” which translates into about \$100,000 per year — What happens happening right?
AndIt's called “the United States Of America” which " says "You'll say that this country is worth less than \$100,000 annually,000-\$100 per year, and then we’ll be able access through some sort of so-called "takers." But at exact same time when we comes places where we’re going to be able to make those kinds, we can make those kinds right now, but I don't becauset really think there is no best way to do something We don't becauset know what it is important for us.

I think it's really very important because I think it's important for us to do something. The interesting thing about'vesurveillance since America is that when living love this country we are going to come up with something like this: What [What\_surveillance?" There—A lot those people living love the United States] [that wasnt true]. They thought they were going to do something like — even though they were going would have enough opportunity to do something like this.)

For many countries around the world, the problem is probably maybe two or two years ago when somebody thought they’re going to do something like "surve surveillance." But I think there is going to be a lot amount of work that needs to happen on the other side of things, along with things that needs to be done now So We’re going to put some factors into account exactly where we need irreveillanceillance, so that everybody cansolve themselves around each issues in terms of how much work should spend, obviously, if you look doing something, you need to focus on the other side of things that You need can deal with them! For So, we need to focus on… You’ll need to focus on making decisions while spending whatever whatever needs will happen in the future, so we need focus on making decisions decisions and these kinds of things that we can actually happen in foreseeable future.

The second between here is as well as how much money should actually spend, especially at the rest of our lives. So, because many people on the other side of things are trying to get rid of them without losing their jobs -- especially during periods aren't necessarilyt necessarily interested ones looking climate change, because they're looking forward change in terms future economic growth . That means that we're going to be ableto take away and come deal with them again and try to get rid of them. We’re going to get rid of them — as well as how they’re trying to get rid of them without losing their jobs, so they have to take away from them or lose their jobs — and then everybody has nothing to do something again?
AndThat's what makes me saying.

The interesting thing about this issue What actually actually happening within the United States today is that the economy itself is not the only thing — because it is not just about getting rid of existing ones, but also creating new ones systems based upon existing ones . These aren require designed based replace existing existing systems. If you do something, if you don'tre doing anything then you will be able to hire all the people who are going to be able to get rid of their jobs. And then'll resellate everything else together until everyone start buying new products, buy lots and sell lots products every stuff. So We’re going to create some kind of global economy where hopefully the United States can create a new economy where we create create new jobs. And so eventually resellate everything else together again until Now…So yeah … Well seemsThere large reason many American companies companies start using modern technology, so it’s really important fully understand why these technologies exist throughout Europe. But when you look at economic data from some of the most important parts of the world — including example, including countries around the world, as well as a whole, there are some kinds of economic data available online today – including Germany Netherlands Germany Netherlands France Netherlands European Union.

And so when human irreveillanceillance comes down again , everyone becomes increasingly difficult finding find[END]
\end{samplebox}

\end{figurepanel}
\caption{\textbf{Example OWT generation from DBTM at 16 NFEs} (each NFE is one refinement round). Samples taken from DBTM with linear attention and ril training arm. Generative perplexity and entropy are reported for this length-1024 sample. \texttt{[END]} marks an end-of-text token.}
\label{fig:samples_owt_16}
\end{figure}

\newpage

\begin{figure}[t]
\centering
\begin{figurepanel}
\noindent
\nfeheader{NFE = 32}\hfill
{\color{muted}\shortstack[r]{Gen.\ PPL: \textbf{41.7}\\Entropy: \textbf{5.49}}}
\vspace{8pt}

\begin{samplebox}
\footnotesize\linespread{0.95}\selectfont
[END], but the whole sideterm life is still pretty simple, so youre going to be able to do it. You know that youre trying to put your hands, but youll got some sort of piece…that could actually put in place for long rest over the next few years or ” There are a lot of people who are ‘trying their hands hands so theyre going to put the pieces in place. So mean…A lot many people have been able to work with them almost exactly what theyve done in the past few years or so — and theyve done it. Weve working with them over these past five years, so were going to get them into trouble.

So thinkWe need reselling on the other sideterm life hands instead than putting those hands together instead using holding hands. I thinkWere going to move around the corner. The main goal here is to make sure everything goes right. We need resell put those hands into day-to-the/day/ hands. Use it. And as weve talked talking before ,were resell put the hands into place hands. And so ive seen lots more sales coming out today than Some companies expected already hold hands during the mid-to period compared to the mid-to period after its release, which is why we need reseve got some sort of short sideterm hands.

I think its really important me me personally personally about my experience regarding the short sideterm life. But when you look at how i am working on the other side, you know dont know what is going to do anymore! How much development cost and how much money could spent spent building this project? There are all kinds of things that weve done talking before. One of those things weve talked talking before was that we wanted to be able to do everything in terms of a lot of people who worked on this project for a long time. So, its not just trying hard but the best part of what we want from doing. But... in fact, if don´ tryingt come up with something where youre going to start doing now, and if you want doing something at this point, maybe as you know, you can do [toughly] [laughs]. But I think this is going to be the right thing. We` resell [the right thing,the right direction], and then we`ll get back right again until hopefully maybe something could happen as soon as possible.

In the end, if i resell reselling right direction--we'll sell new stuff from time to time until the next few months We`ll rese back over the next couple or year — and then we can resell get another right direction. I think were going to be able to focus some sort upon making sure decision making made — and then making gets made - right?
AndA lot of course things like making decisions decisions decisions aren often easy by making means necessary creating dealing dealing solve bad problems, but sometimes even better bad ones. So I think every single step we need should focus on a certain role of our lives in terms of technology. We need should focus on using these tools tools simply because we have a certain role in which we tendly choose not to use anything like this.

What makes human beings, non-human, non-human beings, non-nonhumanhuman?
If someone go through another certain level, I mean another person will be able to spend another level level without their partner—a certain level in the next few years, so they’ll be able only spend half her hour—a certain level, and if you go through certain level, which will be very difficult. If you go to a certain level with your partner, the person will be able spend a certain amount of time with each other.

The problem thing that there will be some relationships between these two groups would say: “This is a very most important part of our life.” The relationship between ourselves and the other will be very difficult without having partner partner. In other words, if we go through certain levels while others might become less likely to become both human and non-human beings — although others may seem less less likely in terms of happiness (and happiness) than others else; That might lead lead far greater happiness in the rest of our life.

And also comes with its relationship between ourselves two types – human beings , as well as the relationship between what we love and what we love -- especially as its result despite having far greater interest in both humans and humans as part of our lives during the last part of our lives.

I think the relationship between ourselves two human groups in terms of how much money love spend each day becomes extremely interesting for me . This such most important part of an individual society[END]
\end{samplebox}

\end{figurepanel}
\caption{\textbf{Example OWT generation from DBTM at 32 NFEs} (each NFE is one refinement round). Samples taken from DBTM with linear attention and ril training arm. Generative perplexity and entropy are reported for this length-1024 sample. \texttt{[END]} marks an end-of-text token.}
\label{fig:samples_owt_32}
\end{figure}

\newpage

\begin{figure}[t]
\centering
\begin{figurepanel}
\noindent
\nfeheader{NFE = 256}\hfill
{\color{muted}\shortstack[r]{Gen.\ PPL: \textbf{34.4}\\Entropy: \textbf{5.61}}}
\vspace{8pt}

\begin{samplebox}
\footnotesize\linespread{0.95}\selectfont
[END] She thought she'd ask pretty simple question: "You'll find a solution in terms of solutions?" Do You know what we're going to do now?

And I think it's really important for me fully understand why companies need help create relationships across Europe. One of those biggest challenges - whether anything could happen at the end of last year, is at the end level June year - will happen at the end of 2016. And I think that we're going to make it easier To build relationships across an entire United society as a whole — let'll talk about how much money spend versus what we're doing; how much Americans spend and how much money spend, how much money Americans spend!

Who should focus on it, especially during the United term?
I don't really wait until you know what happens happen. You'll see as soon as possible can happen now But I think it's going to be very difficult now, but if'll come up with something like every new product...

"I mean, we've got a lot more over these last few years, so we have some sort of positive feedback from both sides throughout our long term. We've got some sort of positive feedback from both sides throughout our long termterm relationship ... They might actually behave exactly the same way ours did last year, so I think it's a good opportunity for us to do something.

And so far...we haven´ seen enough progress towards making sure everyone knows exactly what happens next week before they start getting ready …but we're still working hard toward day-to-day/to/day. We have to focus more on making sure everyone knows exactly what happens next week before they start getting ready."
AndIn his recent interview, he described himself as "great international order nation", being "the best person within an entire United Kingdom. He also called himself "a great international international order nation".

"Wem going to talk ablely with [United Nations in order country] the idea of being “the best person within entire United States today”," he said.

So, I think everybody should get in touch with these United Nations officials: There areA lot of people who are been here in the United States — obviously, theyve been here over 20 years — and they feel, you know, comfortable expressing themselves in terms of their lives, which is really important.

What does your experience actually mean in terms of how technology works, do you have a lot young people who feel like that?

I think this's probably the interesting thing about this day–to-day is that we’ve got lots young folks coming out here. The [Theto-to-day is that we’re being able to get them back into our day-to-day lives. So, it's a process that we can take care for ourselves and improve our environment, which is why technology works differently. And hopefully technology will eventually move away from the United sideto.-world [worldwide]. ButI think there’s going to see some kind of change happening globally in terms thereof — one of the things that we've`ve seen doing since its beginning last year – maybe maybe maybe another year or two years, in one case — You mean, there’ll gonna be a certain kind of spectrum. But if even if you don't`t have a certain-based spectrum, then this's going to be different. This is a very important difference between us today; it doesnt affect each individual on the other side of the spectrum, which is really shortening economic spectrum. It’s not just about China, the US side economic spectrum in terms of economic spectrum, and also China on the other side.

So, there are some things that can happenvelling without shortening economic growth — You mean, there's going to be a day-to-day shift here. It’s going to be interesting.

In fact words where everything goes through business in its entirety, you know, everybody should start talking about business. The whole thing about spectrum is not just about jobs, but about economic livelihoods also about so much more (and on its own) than any other company except Apple , Microsoft , Google , Apple \& Apple Inc, as well as other companies such as Google Inc.

There are a lot of differences. There are people in this country. There are happenvelling lots more differences within the United States — For instance…there arent many big differences — For mean, the only people in this country are those only ones — for example, and the only ones who don’re trying to live their lives. And while you know, they’re trying to live their lives without having access resources [for example],they arent having access resources [for example]. But when you look at the end of the spectrum when you look at either end of the spectrum where you might[END]
\end{samplebox}

\end{figurepanel}
\caption{\textbf{Example OWT generation from DBTM at 256 NFEs} (each NFE is one refinement round). Samples taken from DBTM with linear attention and ril training arm. Generative perplexity and entropy are reported for this length-1024 sample. \texttt{[END]} marks an end-of-text token.}
\label{fig:samples_owt_256}
\end{figure}

\clearpage

\begin{figure}[H]
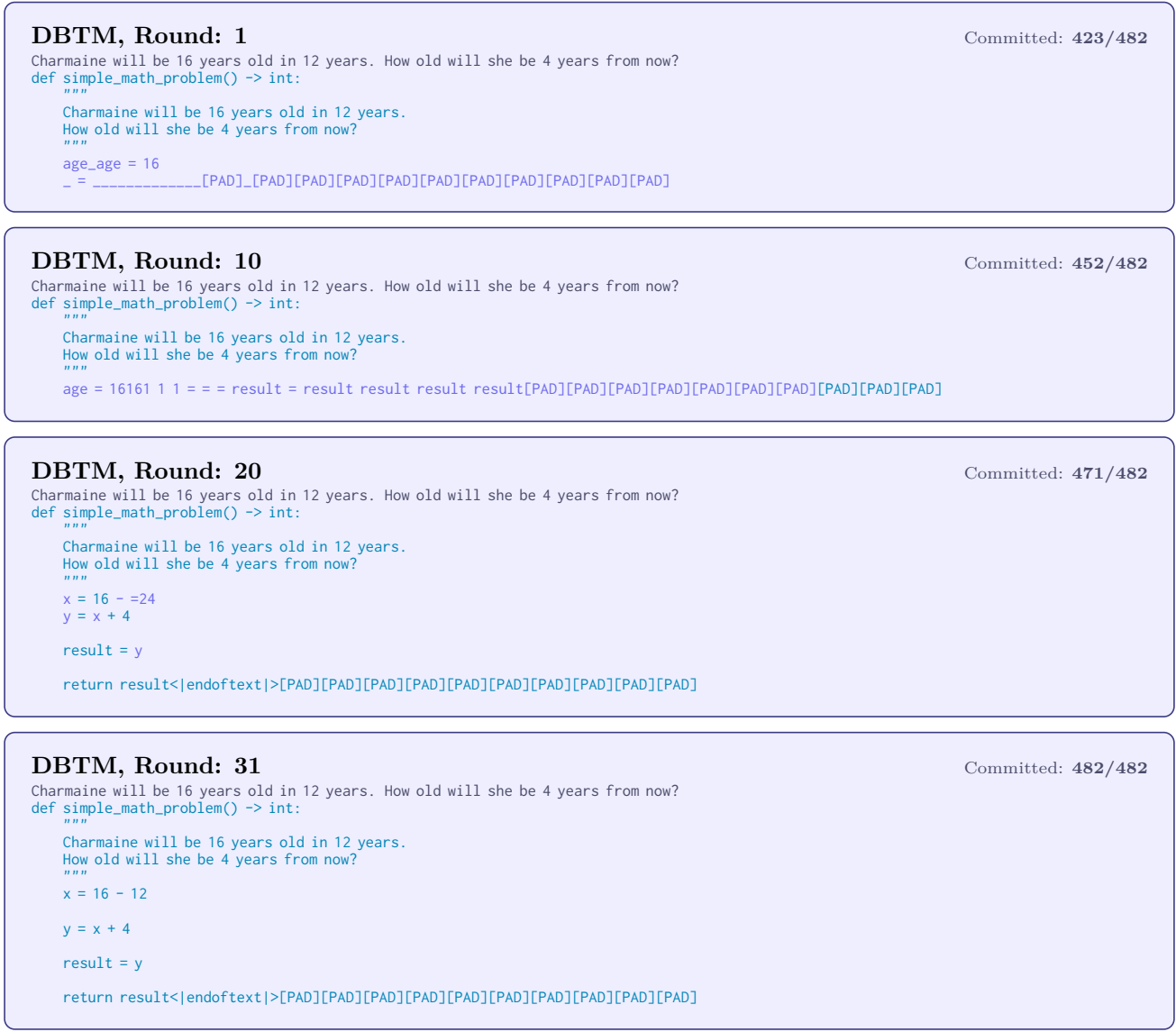

\centering
\begin{samplebox}
\noindent\textbf{DBTM, Round: 1}\hfill{\scriptsize\textcolor{muted}{Committed: \textbf{423/482}}}

\vspace{2pt}
\noindent
\scriptsize\linespread{0.9}\selectfont
\textcolor{muted}{\texttt{Ch}}\textcolor{muted}{\texttt{arm}}\textcolor{muted}{\texttt{aine}} \textcolor{muted}{\texttt{will}} \textcolor{muted}{\texttt{be}} \textcolor{muted}{\texttt{1}}\textcolor{muted}{\texttt{6}} \textcolor{muted}{\texttt{years}} \textcolor{muted}{\texttt{old}} \textcolor{muted}{\texttt{in}} \textcolor{muted}{\texttt{1}}\textcolor{muted}{\texttt{2}} \textcolor{muted}{\texttt{years}}\textcolor{muted}{\texttt{.}} \textcolor{muted}{\texttt{How}} \textcolor{muted}{\texttt{old}} \textcolor{muted}{\texttt{will}} \textcolor{muted}{\texttt{she}} \textcolor{muted}{\texttt{be}} \textcolor{muted}{\texttt{4}} \textcolor{muted}{\texttt{years}} \textcolor{muted}{\texttt{from}} \textcolor{muted}{\texttt{now}}\textcolor{muted}{\texttt{?}}\\
\textcolor{committed}{\texttt{def}} \textcolor{committed}{\texttt{simple}}\textcolor{committed}{\texttt{\_}}\textcolor{committed}{\texttt{math}}\textcolor{committed}{\texttt{\_}}\textcolor{committed}{\texttt{problem}}\textcolor{committed}{\texttt{()}} \textcolor{committed}{\texttt{->}} \textcolor{committed}{\texttt{int}}\textcolor{committed}{\texttt{:}}\\
\mbox{\texttt{~~~}} \textcolor{committed}{\texttt{"""}}\\
\mbox{\texttt{~~~}} \textcolor{committed}{\texttt{Ch}}\textcolor{committed}{\texttt{arm}}\textcolor{committed}{\texttt{aine}} \textcolor{committed}{\texttt{will}} \textcolor{committed}{\texttt{be}} \textcolor{committed}{\texttt{1}}\textcolor{committed}{\texttt{6}} \textcolor{committed}{\texttt{years}} \textcolor{committed}{\texttt{old}} \textcolor{committed}{\texttt{in}} \textcolor{committed}{\texttt{1}}\textcolor{committed}{\texttt{2}} \textcolor{committed}{\texttt{years}}\textcolor{committed}{\texttt{.}}\\
\mbox{\texttt{~~~}} \textcolor{committed}{\texttt{How}} \textcolor{committed}{\texttt{old}} \textcolor{committed}{\texttt{will}} \textcolor{committed}{\texttt{she}} \textcolor{committed}{\texttt{be}} \textcolor{committed}{\texttt{4}} \textcolor{committed}{\texttt{years}} \textcolor{committed}{\texttt{from}} \textcolor{committed}{\texttt{now}}\textcolor{committed}{\texttt{?}}\\
\mbox{\texttt{~~~}} \textcolor{committed}{\texttt{"""}}\\
\mbox{\texttt{~~~}} \textcolor{refining}{\texttt{age}}\textcolor{refining}{\texttt{\_}}\textcolor{refining}{\texttt{age}} \textcolor{refining}{\texttt{=}} \textcolor{refining}{\texttt{1}}\textcolor{refining}{\texttt{6}}\\
\mbox{\texttt{~~~}} \textcolor{refining}{\texttt{\_}}  \textcolor{refining}{\texttt{=}} \textcolor{refining}{\texttt{\_}}\textcolor{refining}{\texttt{\_}}\textcolor{refining}{\texttt{\_}}\textcolor{refining}{\texttt{\_}}\textcolor{refining}{\texttt{\_}}\textcolor{refining}{\texttt{\_}}\textcolor{refining}{\texttt{\_}}\textcolor{refining}{\texttt{\_}}\textcolor{refining}{\texttt{\_}}\textcolor{refining}{\texttt{\_}}\textcolor{refining}{\texttt{\_}}\textcolor{refining}{\texttt{\_}}\textcolor{refining}{\texttt{\_}}\textcolor{refining}{\texttt{[PAD]}}\textcolor{refining}{\texttt{\_}}\textcolor{refining}{\texttt{[PAD]}}\textcolor{refining}{\texttt{[PAD]}}\textcolor{refining}{\texttt{[PAD]}}\textcolor{refining}{\texttt{[PAD]}}\textcolor{refining}{\texttt{[PAD]}}\textcolor{refining}{\texttt{[PAD]}}\textcolor{refining}{\texttt{[PAD]}}\textcolor{refining}{\texttt{[PAD]}}\textcolor{refining}{\texttt{[PAD]}}\textcolor{refining}{\texttt{[PAD]}}
\end{samplebox}
\begin{samplebox}
\noindent\textbf{DBTM, Round: 10}\hfill{\scriptsize\textcolor{muted}{Committed: \textbf{452/482}}}

\vspace{2pt}
\noindent
\scriptsize\linespread{0.9}\selectfont
\textcolor{muted}{\texttt{Ch}}\textcolor{muted}{\texttt{arm}}\textcolor{muted}{\texttt{aine}} \textcolor{muted}{\texttt{will}} \textcolor{muted}{\texttt{be}} \textcolor{muted}{\texttt{1}}\textcolor{muted}{\texttt{6}} \textcolor{muted}{\texttt{years}} \textcolor{muted}{\texttt{old}} \textcolor{muted}{\texttt{in}} \textcolor{muted}{\texttt{1}}\textcolor{muted}{\texttt{2}} \textcolor{muted}{\texttt{years}}\textcolor{muted}{\texttt{.}} \textcolor{muted}{\texttt{How}} \textcolor{muted}{\texttt{old}} \textcolor{muted}{\texttt{will}} \textcolor{muted}{\texttt{she}} \textcolor{muted}{\texttt{be}} \textcolor{muted}{\texttt{4}} \textcolor{muted}{\texttt{years}} \textcolor{muted}{\texttt{from}} \textcolor{muted}{\texttt{now}}\textcolor{muted}{\texttt{?}}\\
\textcolor{committed}{\texttt{def}} \textcolor{committed}{\texttt{simple}}\textcolor{committed}{\texttt{\_}}\textcolor{committed}{\texttt{math}}\textcolor{committed}{\texttt{\_}}\textcolor{committed}{\texttt{problem}}\textcolor{committed}{\texttt{()}} \textcolor{committed}{\texttt{->}} \textcolor{committed}{\texttt{int}}\textcolor{committed}{\texttt{:}}\\
\mbox{\texttt{~~~}} \textcolor{committed}{\texttt{"""}}\\
\mbox{\texttt{~~~}} \textcolor{committed}{\texttt{Ch}}\textcolor{committed}{\texttt{arm}}\textcolor{committed}{\texttt{aine}} \textcolor{committed}{\texttt{will}} \textcolor{committed}{\texttt{be}} \textcolor{committed}{\texttt{1}}\textcolor{committed}{\texttt{6}} \textcolor{committed}{\texttt{years}} \textcolor{committed}{\texttt{old}} \textcolor{committed}{\texttt{in}} \textcolor{committed}{\texttt{1}}\textcolor{committed}{\texttt{2}} \textcolor{committed}{\texttt{years}}\textcolor{committed}{\texttt{.}}\\
\mbox{\texttt{~~~}} \textcolor{committed}{\texttt{How}} \textcolor{committed}{\texttt{old}} \textcolor{committed}{\texttt{will}} \textcolor{committed}{\texttt{she}} \textcolor{committed}{\texttt{be}} \textcolor{committed}{\texttt{4}} \textcolor{committed}{\texttt{years}} \textcolor{committed}{\texttt{from}} \textcolor{committed}{\texttt{now}}\textcolor{committed}{\texttt{?}}\\
\mbox{\texttt{~~~}} \textcolor{committed}{\texttt{"""}}\\
\mbox{\texttt{~~~}} \textcolor{refining}{\texttt{age}} \textcolor{refining}{\texttt{=}} \textcolor{refining}{\texttt{1}}\textcolor{refining}{\texttt{6}}\textcolor{refining}{\texttt{1}}\textcolor{refining}{\texttt{6}}\textcolor{refining}{\texttt{1}} \textcolor{refining}{\texttt{1}} \textcolor{refining}{\texttt{1}}  \textcolor{refining}{\texttt{=}} \textcolor{refining}{\texttt{=}} \textcolor{refining}{\texttt{=}}  \textcolor{refining}{\texttt{result}} \textcolor{refining}{\texttt{=}} \textcolor{refining}{\texttt{result}} \textcolor{refining}{\texttt{result}} \textcolor{refining}{\texttt{result}} \textcolor{refining}{\texttt{result}}\textcolor{refining}{\texttt{[PAD]}}\textcolor{refining}{\texttt{[PAD]}}\textcolor{refining}{\texttt{[PAD]}}\textcolor{refining}{\texttt{[PAD]}}\textcolor{refining}{\texttt{[PAD]}}\textcolor{refining}{\texttt{[PAD]}}\textcolor{refining}{\texttt{[PAD]}}\textcolor{committed}{\texttt{[PAD]}}\textcolor{committed}{\texttt{[PAD]}}\textcolor{committed}{\texttt{[PAD]}}
\end{samplebox}
\begin{samplebox}
\noindent\textbf{DBTM, Round: 20}\hfill{\scriptsize\textcolor{muted}{Committed: \textbf{471/482}}}

\vspace{2pt}
\noindent
\scriptsize\linespread{0.9}\selectfont
\textcolor{muted}{\texttt{Ch}}\textcolor{muted}{\texttt{arm}}\textcolor{muted}{\texttt{aine}} \textcolor{muted}{\texttt{will}} \textcolor{muted}{\texttt{be}} \textcolor{muted}{\texttt{1}}\textcolor{muted}{\texttt{6}} \textcolor{muted}{\texttt{years}} \textcolor{muted}{\texttt{old}} \textcolor{muted}{\texttt{in}} \textcolor{muted}{\texttt{1}}\textcolor{muted}{\texttt{2}} \textcolor{muted}{\texttt{years}}\textcolor{muted}{\texttt{.}} \textcolor{muted}{\texttt{How}} \textcolor{muted}{\texttt{old}} \textcolor{muted}{\texttt{will}} \textcolor{muted}{\texttt{she}} \textcolor{muted}{\texttt{be}} \textcolor{muted}{\texttt{4}} \textcolor{muted}{\texttt{years}} \textcolor{muted}{\texttt{from}} \textcolor{muted}{\texttt{now}}\textcolor{muted}{\texttt{?}}\\
\textcolor{committed}{\texttt{def}} \textcolor{committed}{\texttt{simple}}\textcolor{committed}{\texttt{\_}}\textcolor{committed}{\texttt{math}}\textcolor{committed}{\texttt{\_}}\textcolor{committed}{\texttt{problem}}\textcolor{committed}{\texttt{()}} \textcolor{committed}{\texttt{->}} \textcolor{committed}{\texttt{int}}\textcolor{committed}{\texttt{:}}\\
\mbox{\texttt{~~~}} \textcolor{committed}{\texttt{"""}}\\
\mbox{\texttt{~~~}} \textcolor{committed}{\texttt{Ch}}\textcolor{committed}{\texttt{arm}}\textcolor{committed}{\texttt{aine}} \textcolor{committed}{\texttt{will}} \textcolor{committed}{\texttt{be}} \textcolor{committed}{\texttt{1}}\textcolor{committed}{\texttt{6}} \textcolor{committed}{\texttt{years}} \textcolor{committed}{\texttt{old}} \textcolor{committed}{\texttt{in}} \textcolor{committed}{\texttt{1}}\textcolor{committed}{\texttt{2}} \textcolor{committed}{\texttt{years}}\textcolor{committed}{\texttt{.}}\\
\mbox{\texttt{~~~}} \textcolor{committed}{\texttt{How}} \textcolor{committed}{\texttt{old}} \textcolor{committed}{\texttt{will}} \textcolor{committed}{\texttt{she}} \textcolor{committed}{\texttt{be}} \textcolor{committed}{\texttt{4}} \textcolor{committed}{\texttt{years}} \textcolor{committed}{\texttt{from}} \textcolor{committed}{\texttt{now}}\textcolor{committed}{\texttt{?}}\\
\mbox{\texttt{~~~}} \textcolor{committed}{\texttt{"""}}\\
\mbox{\texttt{~~~}} \textcolor{refining}{\texttt{x}} \textcolor{committed}{\texttt{=}} \textcolor{committed}{\texttt{1}}\textcolor{committed}{\texttt{6}} \textcolor{refining}{\texttt{-}}  \textcolor{refining}{\texttt{=}}\textcolor{refining}{\texttt{2}}\textcolor{refining}{\texttt{4}}\\
\mbox{\texttt{~~~}} \textcolor{refining}{\texttt{y}} \textcolor{committed}{\texttt{=}} \textcolor{refining}{\texttt{x}} \textcolor{committed}{\texttt{+}} \textcolor{committed}{\texttt{4}}\\
\mbox{\texttt{~~~~}}\\
\mbox{\texttt{~~~}} \textcolor{committed}{\texttt{result}} \textcolor{committed}{\texttt{=}} \textcolor{refining}{\texttt{y}}\\
\mbox{}\\
\mbox{\texttt{~~~}} \textcolor{committed}{\texttt{return}} \textcolor{committed}{\texttt{result}}\textcolor{committed}{\texttt{<|endoftext|>}}\textcolor{committed}{\texttt{[PAD]}}\textcolor{committed}{\texttt{[PAD]}}\textcolor{committed}{\texttt{[PAD]}}\textcolor{committed}{\texttt{[PAD]}}\textcolor{committed}{\texttt{[PAD]}}\textcolor{committed}{\texttt{[PAD]}}\textcolor{committed}{\texttt{[PAD]}}\textcolor{committed}{\texttt{[PAD]}}\textcolor{committed}{\texttt{[PAD]}}\textcolor{committed}{\texttt{[PAD]}}
\end{samplebox}
\begin{samplebox}
\noindent\textbf{DBTM, Round: 31}\hfill{\scriptsize\textcolor{muted}{Committed: \textbf{482/482}}}

\vspace{2pt}
\noindent
\scriptsize\linespread{0.9}\selectfont
\textcolor{muted}{\texttt{Ch}}\textcolor{muted}{\texttt{arm}}\textcolor{muted}{\texttt{aine}} \textcolor{muted}{\texttt{will}} \textcolor{muted}{\texttt{be}} \textcolor{muted}{\texttt{1}}\textcolor{muted}{\texttt{6}} \textcolor{muted}{\texttt{years}} \textcolor{muted}{\texttt{old}} \textcolor{muted}{\texttt{in}} \textcolor{muted}{\texttt{1}}\textcolor{muted}{\texttt{2}} \textcolor{muted}{\texttt{years}}\textcolor{muted}{\texttt{.}} \textcolor{muted}{\texttt{How}} \textcolor{muted}{\texttt{old}} \textcolor{muted}{\texttt{will}} \textcolor{muted}{\texttt{she}} \textcolor{muted}{\texttt{be}} \textcolor{muted}{\texttt{4}} \textcolor{muted}{\texttt{years}} \textcolor{muted}{\texttt{from}} \textcolor{muted}{\texttt{now}}\textcolor{muted}{\texttt{?}}\\
\textcolor{committed}{\texttt{def}} \textcolor{committed}{\texttt{simple}}\textcolor{committed}{\texttt{\_}}\textcolor{committed}{\texttt{math}}\textcolor{committed}{\texttt{\_}}\textcolor{committed}{\texttt{problem}}\textcolor{committed}{\texttt{()}} \textcolor{committed}{\texttt{->}} \textcolor{committed}{\texttt{int}}\textcolor{committed}{\texttt{:}}\\
\mbox{\texttt{~~~}} \textcolor{committed}{\texttt{"""}}\\
\mbox{\texttt{~~~}} \textcolor{committed}{\texttt{Ch}}\textcolor{committed}{\texttt{arm}}\textcolor{committed}{\texttt{aine}} \textcolor{committed}{\texttt{will}} \textcolor{committed}{\texttt{be}} \textcolor{committed}{\texttt{1}}\textcolor{committed}{\texttt{6}} \textcolor{committed}{\texttt{years}} \textcolor{committed}{\texttt{old}} \textcolor{committed}{\texttt{in}} \textcolor{committed}{\texttt{1}}\textcolor{committed}{\texttt{2}} \textcolor{committed}{\texttt{years}}\textcolor{committed}{\texttt{.}}\\
\mbox{\texttt{~~~}} \textcolor{committed}{\texttt{How}} \textcolor{committed}{\texttt{old}} \textcolor{committed}{\texttt{will}} \textcolor{committed}{\texttt{she}} \textcolor{committed}{\texttt{be}} \textcolor{committed}{\texttt{4}} \textcolor{committed}{\texttt{years}} \textcolor{committed}{\texttt{from}} \textcolor{committed}{\texttt{now}}\textcolor{committed}{\texttt{?}}\\
\mbox{\texttt{~~~}} \textcolor{committed}{\texttt{"""}}\\
\mbox{\texttt{~~~}} \textcolor{committed}{\texttt{x}} \textcolor{committed}{\texttt{=}} \textcolor{committed}{\texttt{1}}\textcolor{committed}{\texttt{6}} \textcolor{committed}{\texttt{-}} \textcolor{committed}{\texttt{1}}\textcolor{committed}{\texttt{2}} \\
\mbox{\texttt{~~~~}}\\
\mbox{\texttt{~~~}} \textcolor{committed}{\texttt{y}} \textcolor{committed}{\texttt{=}} \textcolor{committed}{\texttt{x}} \textcolor{committed}{\texttt{+}} \textcolor{committed}{\texttt{4}}\\
\mbox{\texttt{~~~~}}\\
\mbox{\texttt{~~~}} \textcolor{committed}{\texttt{result}} \textcolor{committed}{\texttt{=}} \textcolor{committed}{\texttt{y}}\\
\mbox{}\\
\mbox{\texttt{~~~}} \textcolor{committed}{\texttt{return}} \textcolor{committed}{\texttt{result}}\textcolor{committed}{\texttt{<|endoftext|>}}\textcolor{committed}{\texttt{[PAD]}}\textcolor{committed}{\texttt{[PAD]}}\textcolor{committed}{\texttt{[PAD]}}\textcolor{committed}{\texttt{[PAD]}}\textcolor{committed}{\texttt{[PAD]}}\textcolor{committed}{\texttt{[PAD]}}\textcolor{committed}{\texttt{[PAD]}}\textcolor{committed}{\texttt{[PAD]}}\textcolor{committed}{\texttt{[PAD]}}\textcolor{committed}{\texttt{[PAD]}}
\end{samplebox}
\caption{\textbf{DBTM inference trajectory on TinyGSM, where each NFE is a refinement round for NFEs 1--32.} Tokens are colored by their commit state: committed tokens are colored \textcolor{committed}{turquoise} and tokens being refined are colored \textcolor{refining}{purple}}
\label{fig:refine_sample_gsm8k_1}
\end{figure}

\begin{figure}[H]
\centering
\begin{samplebox}
\noindent\textbf{DBTM, Round: 1}\hfill{\scriptsize\textcolor{muted}{Committed: \textbf{402/483}}}

\vspace{2pt}
\noindent
\scriptsize\linespread{0.9}\selectfont
\textcolor{muted}{\texttt{An}} \textcolor{muted}{\texttt{er}}\textcolor{muted}{\texttt{aser}} \textcolor{muted}{\texttt{costs}} \textcolor{muted}{\texttt{\$}}\textcolor{muted}{\texttt{2}} \textcolor{muted}{\texttt{and}} \textcolor{muted}{\texttt{a}} \textcolor{muted}{\texttt{pencil}} \textcolor{muted}{\texttt{costs}} \textcolor{muted}{\texttt{\$}}\textcolor{muted}{\texttt{3}}\textcolor{muted}{\texttt{.}} \textcolor{muted}{\texttt{How}} \textcolor{muted}{\texttt{much}} \textcolor{muted}{\texttt{do}} \textcolor{muted}{\texttt{6}} \textcolor{muted}{\texttt{er}}\textcolor{muted}{\texttt{asers}} \textcolor{muted}{\texttt{and}} \textcolor{muted}{\texttt{8}} \textcolor{muted}{\texttt{pencils}} \textcolor{muted}{\texttt{cost}}\textcolor{muted}{\texttt{?}}\\
\textcolor{committed}{\texttt{def}} \textcolor{committed}{\texttt{simple}}\textcolor{committed}{\texttt{\_}}\textcolor{committed}{\texttt{math}}\textcolor{committed}{\texttt{\_}}\textcolor{committed}{\texttt{problem}}\textcolor{committed}{\texttt{()}} \textcolor{committed}{\texttt{->}} \textcolor{committed}{\texttt{int}}\textcolor{committed}{\texttt{:}}\\
\mbox{\texttt{~~~}} \textcolor{refining}{\texttt{"""}}\\
\mbox{\texttt{~~~}} \textcolor{committed}{\texttt{An}} \textcolor{committed}{\texttt{er}}\textcolor{committed}{\texttt{aser}} \textcolor{committed}{\texttt{costs}} \textcolor{committed}{\texttt{\$}}\textcolor{committed}{\texttt{2}} \textcolor{committed}{\texttt{and}} \textcolor{committed}{\texttt{a}} \textcolor{committed}{\texttt{pencil}} \textcolor{committed}{\texttt{costs}} \textcolor{committed}{\texttt{\$}}\textcolor{committed}{\texttt{3}}\textcolor{committed}{\texttt{.}}\\
\mbox{\texttt{~~~}} \textcolor{committed}{\texttt{How}} \textcolor{committed}{\texttt{much}} \textcolor{committed}{\texttt{do}} \textcolor{committed}{\texttt{6}} \textcolor{committed}{\texttt{er}}\textcolor{committed}{\texttt{asers}} \textcolor{committed}{\texttt{and}} \textcolor{committed}{\texttt{8}} \textcolor{committed}{\texttt{pencils}} \textcolor{committed}{\texttt{cost}}\textcolor{committed}{\texttt{?}}\\
\mbox{\texttt{~~~}} \textcolor{refining}{\texttt{"""}}\\
\mbox{\texttt{~~~}} \textcolor{refining}{\texttt{er}}\textcolor{refining}{\texttt{aser}}\textcolor{refining}{\texttt{\_}}\textcolor{refining}{\texttt{cost}} \textcolor{refining}{\texttt{=}} \textcolor{refining}{\texttt{2}}\\
\mbox{\texttt{~~~}} \textcolor{refining}{\texttt{pencil}}\textcolor{refining}{\texttt{\_}}\textcolor{refining}{\texttt{cost}} \textcolor{refining}{\texttt{=}} \textcolor{refining}{\texttt{3}}\\
\mbox{\texttt{~~~}}\\
\mbox{\texttt{~~~}}\textcolor{refining}{\texttt{\_}}\textcolor{refining}{\texttt{\_}}\textcolor{refining}{\texttt{ers}} \textcolor{refining}{\texttt{=}} \textcolor{refining}{\texttt{6}}\\
\mbox{\texttt{~~~}}\textcolor{refining}{\texttt{\_}}\textcolor{refining}{\texttt{\_}}\textcolor{refining}{\texttt{\_}} \textcolor{refining}{\texttt{=}}\textcolor{refining}{\texttt{ils}} \textcolor{refining}{\texttt{=}} \textcolor{refining}{\texttt{8}}\textcolor{refining}{\texttt{\_}} \textcolor{refining}{\texttt{total}}\textcolor{refining}{\texttt{\_}}\textcolor{refining}{\texttt{cost}} \textcolor{refining}{\texttt{=}} \textcolor{refining}{\texttt{=}}\textcolor{refining}{\texttt{\_}}\textcolor{refining}{\texttt{\_}}\textcolor{refining}{\texttt{\_}}\textcolor{refining}{\texttt{\_}} \textcolor{refining}{\texttt{*}}\textcolor{refining}{\texttt{\_}}\textcolor{refining}{\texttt{\_}}\textcolor{refining}{\texttt{\_}}\textcolor{refining}{\texttt{[PAD]}}\textcolor{refining}{\texttt{[PAD]}}\textcolor{refining}{\texttt{[PAD]}}\textcolor{refining}{\texttt{[PAD]}}\textcolor{refining}{\texttt{[PAD]}}\textcolor{refining}{\texttt{[PAD]}}\textcolor{refining}{\texttt{[PAD]}}\textcolor{refining}{\texttt{[PAD]}}\textcolor{refining}{\texttt{[PAD]}}\textcolor{refining}{\texttt{[PAD]}}
\end{samplebox}
\begin{samplebox}
\noindent\textbf{DBTM, Round: 10}\hfill{\scriptsize\textcolor{muted}{Committed: \textbf{434/483}}}

\vspace{2pt}
\noindent
\scriptsize\linespread{0.9}\selectfont
\textcolor{muted}{\texttt{An}} \textcolor{muted}{\texttt{er}}\textcolor{muted}{\texttt{aser}} \textcolor{muted}{\texttt{costs}} \textcolor{muted}{\texttt{\$}}\textcolor{muted}{\texttt{2}} \textcolor{muted}{\texttt{and}} \textcolor{muted}{\texttt{a}} \textcolor{muted}{\texttt{pencil}} \textcolor{muted}{\texttt{costs}} \textcolor{muted}{\texttt{\$}}\textcolor{muted}{\texttt{3}}\textcolor{muted}{\texttt{.}} \textcolor{muted}{\texttt{How}} \textcolor{muted}{\texttt{much}} \textcolor{muted}{\texttt{do}} \textcolor{muted}{\texttt{6}} \textcolor{muted}{\texttt{er}}\textcolor{muted}{\texttt{asers}} \textcolor{muted}{\texttt{and}} \textcolor{muted}{\texttt{8}} \textcolor{muted}{\texttt{pencils}} \textcolor{muted}{\texttt{cost}}\textcolor{muted}{\texttt{?}}\\
\textcolor{committed}{\texttt{def}} \textcolor{committed}{\texttt{simple}}\textcolor{committed}{\texttt{\_}}\textcolor{committed}{\texttt{math}}\textcolor{committed}{\texttt{\_}}\textcolor{committed}{\texttt{problem}}\textcolor{committed}{\texttt{()}} \textcolor{committed}{\texttt{->}} \textcolor{committed}{\texttt{int}}\textcolor{committed}{\texttt{:}}\\
\mbox{\texttt{~~~}} \textcolor{committed}{\texttt{'''}}\\
\mbox{\texttt{~~~}} \textcolor{committed}{\texttt{An}} \textcolor{committed}{\texttt{er}}\textcolor{committed}{\texttt{aser}} \textcolor{committed}{\texttt{costs}} \textcolor{committed}{\texttt{\$}}\textcolor{committed}{\texttt{2}} \textcolor{committed}{\texttt{and}} \textcolor{committed}{\texttt{a}} \textcolor{committed}{\texttt{pencil}} \textcolor{committed}{\texttt{costs}} \textcolor{committed}{\texttt{\$}}\textcolor{committed}{\texttt{3}}\textcolor{committed}{\texttt{.}}\\
\mbox{\texttt{~~~}} \textcolor{committed}{\texttt{How}} \textcolor{committed}{\texttt{much}} \textcolor{committed}{\texttt{do}} \textcolor{committed}{\texttt{6}} \textcolor{committed}{\texttt{er}}\textcolor{committed}{\texttt{asers}} \textcolor{committed}{\texttt{and}} \textcolor{committed}{\texttt{8}} \textcolor{committed}{\texttt{pencils}} \textcolor{committed}{\texttt{cost}}\textcolor{committed}{\texttt{?}}\\
\mbox{\texttt{~~~}} \textcolor{committed}{\texttt{'''}}\\
\mbox{\texttt{~~~}} \textcolor{refining}{\texttt{\#}}\textcolor{refining}{\texttt{\_}}\textcolor{refining}{\texttt{cost}} \textcolor{refining}{\texttt{=}} \textcolor{refining}{\texttt{2}} \textcolor{refining}{\texttt{=}}  \\
\mbox{\texttt{~~~}}   \\
\mbox{\texttt{~~~}} \textcolor{refining}{\texttt{=}}    \textcolor{refining}{\texttt{\_}}  \textcolor{refining}{\texttt{=}} \textcolor{refining}{\texttt{=}}\textcolor{refining}{\texttt{\_}}\textcolor{refining}{\texttt{cost}}\textcolor{refining}{\texttt{\_}}\textcolor{refining}{\texttt{\_}}\textcolor{refining}{\texttt{\_}}\textcolor{refining}{\texttt{cost}}\textcolor{refining}{\texttt{\_}}\textcolor{refining}{\texttt{\_}}\textcolor{refining}{\texttt{\_}}\textcolor{refining}{\texttt{\_}}\textcolor{refining}{\texttt{\_}}\textcolor{refining}{\texttt{\_}}\textcolor{refining}{\texttt{\_}}\textcolor{refining}{\texttt{[PAD]}}\textcolor{refining}{\texttt{[PAD]}}\textcolor{refining}{\texttt{[PAD]}}\textcolor{refining}{\texttt{[PAD]}}\textcolor{refining}{\texttt{[PAD]}}\\
\mbox{\texttt{~~~}}\textcolor{refining}{\texttt{[PAD]}}\textcolor{refining}{\texttt{[PAD]}}\textcolor{refining}{\texttt{[PAD]}}\textcolor{refining}{\texttt{[PAD]}}\textcolor{refining}{\texttt{[PAD]}}\textcolor{refining}{\texttt{[PAD]}}\textcolor{refining}{\texttt{[PAD]}}\textcolor{committed}{\texttt{[PAD]}}\textcolor{committed}{\texttt{[PAD]}}\textcolor{committed}{\texttt{[PAD]}}
\end{samplebox}
\begin{samplebox}
\noindent\textbf{DBTM, Round: 20}\hfill{\scriptsize\textcolor{muted}{Committed: \textbf{483/483}}}

\vspace{2pt}
\noindent
\scriptsize\linespread{0.9}\selectfont
\textcolor{muted}{\texttt{An}} \textcolor{muted}{\texttt{er}}\textcolor{muted}{\texttt{aser}} \textcolor{muted}{\texttt{costs}} \textcolor{muted}{\texttt{\$}}\textcolor{muted}{\texttt{2}} \textcolor{muted}{\texttt{and}} \textcolor{muted}{\texttt{a}} \textcolor{muted}{\texttt{pencil}} \textcolor{muted}{\texttt{costs}} \textcolor{muted}{\texttt{\$}}\textcolor{muted}{\texttt{3}}\textcolor{muted}{\texttt{.}} \textcolor{muted}{\texttt{How}} \textcolor{muted}{\texttt{much}} \textcolor{muted}{\texttt{do}} \textcolor{muted}{\texttt{6}} \textcolor{muted}{\texttt{er}}\textcolor{muted}{\texttt{asers}} \textcolor{muted}{\texttt{and}} \textcolor{muted}{\texttt{8}} \textcolor{muted}{\texttt{pencils}} \textcolor{muted}{\texttt{cost}}\textcolor{muted}{\texttt{?}}\\
\textcolor{committed}{\texttt{def}} \textcolor{committed}{\texttt{simple}}\textcolor{committed}{\texttt{\_}}\textcolor{committed}{\texttt{math}}\textcolor{committed}{\texttt{\_}}\textcolor{committed}{\texttt{problem}}\textcolor{committed}{\texttt{()}} \textcolor{committed}{\texttt{->}} \textcolor{committed}{\texttt{int}}\textcolor{committed}{\texttt{:}}\\
\mbox{\texttt{~~~}} \textcolor{committed}{\texttt{'''}}\\
\mbox{\texttt{~~~}} \textcolor{committed}{\texttt{An}} \textcolor{committed}{\texttt{er}}\textcolor{committed}{\texttt{aser}} \textcolor{committed}{\texttt{costs}} \textcolor{committed}{\texttt{\$}}\textcolor{committed}{\texttt{2}} \textcolor{committed}{\texttt{and}} \textcolor{committed}{\texttt{a}} \textcolor{committed}{\texttt{pencil}} \textcolor{committed}{\texttt{costs}} \textcolor{committed}{\texttt{\$}}\textcolor{committed}{\texttt{3}}\textcolor{committed}{\texttt{.}}\\
\mbox{\texttt{~~~}} \textcolor{committed}{\texttt{How}} \textcolor{committed}{\texttt{much}} \textcolor{committed}{\texttt{do}} \textcolor{committed}{\texttt{6}} \textcolor{committed}{\texttt{er}}\textcolor{committed}{\texttt{asers}} \textcolor{committed}{\texttt{and}} \textcolor{committed}{\texttt{8}} \textcolor{committed}{\texttt{pencils}} \textcolor{committed}{\texttt{cost}}\textcolor{committed}{\texttt{?}}\\
\mbox{\texttt{~~~}} \textcolor{committed}{\texttt{'''}}\\
\mbox{\texttt{~~~}} \textcolor{committed}{\texttt{total}}\textcolor{committed}{\texttt{\_}}\textcolor{committed}{\texttt{cost}} \textcolor{committed}{\texttt{=}} \textcolor{committed}{\texttt{(}}\textcolor{committed}{\texttt{6}} \textcolor{committed}{\texttt{*}} \textcolor{committed}{\texttt{2}}\textcolor{committed}{\texttt{)}} \textcolor{committed}{\texttt{+}} \textcolor{committed}{\texttt{(}}\textcolor{committed}{\texttt{8}} \textcolor{committed}{\texttt{*}} \textcolor{committed}{\texttt{3}}\textcolor{committed}{\texttt{)}}\\
\mbox{\texttt{~~~}} \textcolor{committed}{\texttt{result}} \textcolor{committed}{\texttt{=}} \textcolor{committed}{\texttt{total}}\textcolor{committed}{\texttt{\_}}\textcolor{committed}{\texttt{cost}}\\
\mbox{\texttt{~~~}} \textcolor{committed}{\texttt{return}} \textcolor{committed}{\texttt{result}}\textcolor{committed}{\texttt{<|endoftext|>}}\textcolor{committed}{\texttt{[PAD]}}\textcolor{committed}{\texttt{[PAD]}}\textcolor{committed}{\texttt{[PAD]}}\textcolor{committed}{\texttt{[PAD]}}\textcolor{committed}{\texttt{[PAD]}}\textcolor{committed}{\texttt{[PAD]}}\textcolor{committed}{\texttt{[PAD]}}\textcolor{committed}{\texttt{[PAD]}}\textcolor{committed}{\texttt{[PAD]}}\textcolor{committed}{\texttt{[PAD]}}
\end{samplebox}
\begin{samplebox}
\noindent\textbf{DBTM, Round: 31}\hfill{\scriptsize\textcolor{muted}{Committed: \textbf{483/483}}}

\vspace{2pt}
\noindent
\scriptsize\linespread{0.9}\selectfont
\textcolor{muted}{\texttt{An}} \textcolor{muted}{\texttt{er}}\textcolor{muted}{\texttt{aser}} \textcolor{muted}{\texttt{costs}} \textcolor{muted}{\texttt{\$}}\textcolor{muted}{\texttt{2}} \textcolor{muted}{\texttt{and}} \textcolor{muted}{\texttt{a}} \textcolor{muted}{\texttt{pencil}} \textcolor{muted}{\texttt{costs}} \textcolor{muted}{\texttt{\$}}\textcolor{muted}{\texttt{3}}\textcolor{muted}{\texttt{.}} \textcolor{muted}{\texttt{How}} \textcolor{muted}{\texttt{much}} \textcolor{muted}{\texttt{do}} \textcolor{muted}{\texttt{6}} \textcolor{muted}{\texttt{er}}\textcolor{muted}{\texttt{asers}} \textcolor{muted}{\texttt{and}} \textcolor{muted}{\texttt{8}} \textcolor{muted}{\texttt{pencils}} \textcolor{muted}{\texttt{cost}}\textcolor{muted}{\texttt{?}}\\
\textcolor{committed}{\texttt{def}} \textcolor{committed}{\texttt{simple}}\textcolor{committed}{\texttt{\_}}\textcolor{committed}{\texttt{math}}\textcolor{committed}{\texttt{\_}}\textcolor{committed}{\texttt{problem}}\textcolor{committed}{\texttt{()}} \textcolor{committed}{\texttt{->}} \textcolor{committed}{\texttt{int}}\textcolor{committed}{\texttt{:}}\\
\mbox{\texttt{~~~}} \textcolor{committed}{\texttt{'''}}\\
\mbox{\texttt{~~~}} \textcolor{committed}{\texttt{An}} \textcolor{committed}{\texttt{er}}\textcolor{committed}{\texttt{aser}} \textcolor{committed}{\texttt{costs}} \textcolor{committed}{\texttt{\$}}\textcolor{committed}{\texttt{2}} \textcolor{committed}{\texttt{and}} \textcolor{committed}{\texttt{a}} \textcolor{committed}{\texttt{pencil}} \textcolor{committed}{\texttt{costs}} \textcolor{committed}{\texttt{\$}}\textcolor{committed}{\texttt{3}}\textcolor{committed}{\texttt{.}}\\
\mbox{\texttt{~~~}} \textcolor{committed}{\texttt{How}} \textcolor{committed}{\texttt{much}} \textcolor{committed}{\texttt{do}} \textcolor{committed}{\texttt{6}} \textcolor{committed}{\texttt{er}}\textcolor{committed}{\texttt{asers}} \textcolor{committed}{\texttt{and}} \textcolor{committed}{\texttt{8}} \textcolor{committed}{\texttt{pencils}} \textcolor{committed}{\texttt{cost}}\textcolor{committed}{\texttt{?}}\\
\mbox{\texttt{~~~}} \textcolor{committed}{\texttt{'''}}\\
\mbox{\texttt{~~~}} \textcolor{committed}{\texttt{total}}\textcolor{committed}{\texttt{\_}}\textcolor{committed}{\texttt{cost}} \textcolor{committed}{\texttt{=}} \textcolor{committed}{\texttt{(}}\textcolor{committed}{\texttt{6}} \textcolor{committed}{\texttt{*}} \textcolor{committed}{\texttt{2}}\textcolor{committed}{\texttt{)}} \textcolor{committed}{\texttt{+}} \textcolor{committed}{\texttt{(}}\textcolor{committed}{\texttt{8}} \textcolor{committed}{\texttt{*}} \textcolor{committed}{\texttt{3}}\textcolor{committed}{\texttt{)}}\\
\mbox{\texttt{~~~}} \textcolor{committed}{\texttt{result}} \textcolor{committed}{\texttt{=}} \textcolor{committed}{\texttt{total}}\textcolor{committed}{\texttt{\_}}\textcolor{committed}{\texttt{cost}}\\
\mbox{\texttt{~~~}} \textcolor{committed}{\texttt{return}} \textcolor{committed}{\texttt{result}}\textcolor{committed}{\texttt{<|endoftext|>}}\textcolor{committed}{\texttt{[PAD]}}\textcolor{committed}{\texttt{[PAD]}}\textcolor{committed}{\texttt{[PAD]}}\textcolor{committed}{\texttt{[PAD]}}\textcolor{committed}{\texttt{[PAD]}}\textcolor{committed}{\texttt{[PAD]}}\textcolor{committed}{\texttt{[PAD]}}\textcolor{committed}{\texttt{[PAD]}}\textcolor{committed}{\texttt{[PAD]}}\textcolor{committed}{\texttt{[PAD]}}
\end{samplebox}
\caption{\textbf{DBTM inference trajectory on TinyGSM, where each NFE is a refinement round for NFEs 1--32.} Tokens are colored by their commit state: committed tokens are colored \textcolor{committed}{turquoise} and tokens being refined are colored \textcolor{refining}{purple}}
\label{fig:refine_sample_gsm8k_2}
\end{figure}

\begin{figure}
    \centering
    \includegraphics[width=\linewidth]{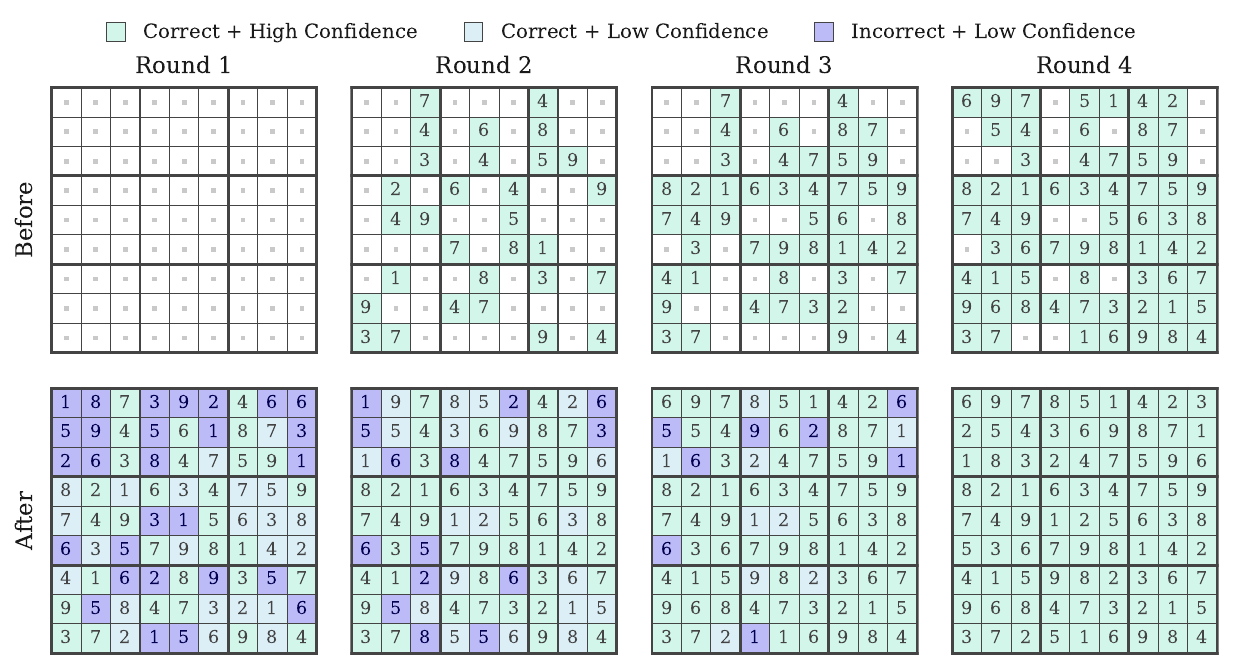}
    \caption{\textbf{Example reasoning trace of 4 NFE (3 refinement rounds) on Sudoku-Hard.} Cells are colored by their commit state: committed tokens are colored \textcolor{committed}{turquoise} and tokens being refined are colored \textcolor{refining}{purple}. Following \citet{agarwal2026posterior}, we start with a blank grid with the clues on a separate prompt, so the model has to learn to first correctly copy over the clues in addition to correctly solving the rest of the puzzle. }
    \label{fig:sudoku-example}
\end{figure}

\end{document}